\documentclass{article}

\usepackage{graphicx}
\usepackage{subcaption}
\usepackage{booktabs} 
\usepackage{wrapfig}
\usepackage{hyperref}
\usepackage{amssymb}
\usepackage{comment}
\usepackage{pifont}
\usepackage[utf8]{inputenc}
\newcommand{\tick}{\ding{51}}
\newcommand{\cross}{\ding{55}}
\usepackage{pdfpages}

\usepackage{soul}
\usepackage{url}
\usepackage{multirow}
\usepackage{adjustbox}
\usepackage{xcolor}
\usepackage{amsthm}
\usepackage{comment}
\usepackage{thmtools}
\usepackage{todonotes}
\usepackage{color}
\usepackage{wrapfig}
\usepackage{amsmath}

\def\cast{{
   \mathord{
      \hbox to 0em{
         \ooalign{
	   \smash{\hbox{$\ast$}}\crcr
	   \smash{\hskip-1pt\Large\hbox{$\circ$}} }
	 \hidewidth}
      \phantom{\bigcirc}
} }}

\def\bm#1{\mbox{\boldmath $#1$}}

\newcommand{\bds}{\begin {itemize}}
\newcommand{\eds}{\end {itemize}}
\newcommand{\bdf}{\begin{definition}}
\newcommand{\blm}{\begin{lemma}}
\newcommand{\edf}{\end{definition}}
\newcommand{\elm}{\end{lemma}}
\newcommand{\bthm}{\begin{theorem}}
\newcommand{\ethm}{\end{theorem}}
\newcommand{\bprp}{\begin{prop}}
\newcommand{\eprp}{\end{prop}}
\newcommand{\bcl}{\begin{claim}}
\newcommand{\ecl}{\end{claim}}
\newcommand{\bcr}{\begin{coro}}
\newcommand{\ecr}{\end{coro}}
\newcommand{\bquest}{\begin{question}}
\newcommand{\equest}{\end{question}}

\newcommand{\larrow}{{\larrow}}

\newcommand{\argmin}{\ensuremath{\mathrm{arg}\min}}
\newcommand{\argmax}{\ensuremath{\mathrm{arg}\max}}

\newcommand{\cG}{{\ensuremath{\mathcal{G}}}}

\newcommand{\cO}{{\ensuremath{\mathcal{O}}}}

\newcommand{\vg}{{\ensuremath{{\mathbf{g}}}}}

\newcommand{\vh}{{\ensuremath{{\mathbf{h}}}}}

\newcommand{\vp}{{\ensuremath{{\mathbf{p}}}}}
\newcommand{\vq}{{\ensuremath{{\mathbf{q}}}}}
\newcommand{\vr}{{\ensuremath{{\mathbf{r}}}}}

\newcommand{\vw}{{\ensuremath{{\mathbf{w}}}}}

\newcommand{\vx}{{\ensuremath{{\mathbf{x}}}}}

\newcommand{\vy}{{\ensuremath{{\mathbf{y}}}}}

\newcommand{\mA}{{\ensuremath{\mathbf{A}}}}

\newcommand{\mB}{{\ensuremath{\mathbf{B}}}}

\newcommand{\mC}{{\ensuremath{\mathbf{C}}}}

\newcommand{\mD}{{\ensuremath{\mathbf{D}}}}

\newcommand{\mH}{{\ensuremath{\mathbf{H}}}}

\newcommand{\mI}{{\ensuremath{\mathbf{I}}}}

\newcommand{\mL}{{\ensuremath{\mathbf{L}}}}

\newcommand{\mN}{{\ensuremath{\mathbf{N}}}}

\newcommand{\mV}{{\ensuremath{\mathbf{V}}}}

\newcommand{\mW}{{\ensuremath{\mathbf{W}}}}

\newcommand{\mX}{{\ensuremath{\mathbf{X}}}}

\usepackage{latexsym}

\def\IC{\mathbb C}
\def\IN{\mathbb N}
\def\IZ{\mathbb Z}
\def\IR{\mathbb R}

\def\shat{^{\mathchoice{}{}%
 {\,\,\smash{\hbox{\lower4pt\hbox{$\widehat{\null}$}}}}%
 {\,\smash{\hbox{\lower3pt\hbox{$\hat{\null}$}}}}}}

\def\bSigma{{
      \ooalign{
      \smash{\hskip.4pt\raise.4pt\hbox{$\Sigma$}}\vphantom{}\crcr
      \smash{\hskip.7pt\raise.6pt\hbox{$\Sigma$}}\vphantom{}\crcr
      \smash{\hbox{$\Sigma$}}\vphantom{$\Sigma$}}
      \vphantom{\hbox{$\Sigma$}}
      }}
\def\bTheta{{
      \ooalign{
      \smash{\hskip.5pt\raise.5pt\hbox{$\Theta$}}\vphantom{}\crcr
      \smash{\hskip.0pt\raise.1pt\hbox{$\Theta$}}\vphantom{}\crcr
      \smash{\hbox{$\Theta$}}\vphantom{$\Theta$}}
      \vphantom{\hbox{$\Theta$}}
      }}
\def\bDelta{{
      \ooalign{
      \smash{\hskip.4pt\raise.4pt\hbox{$\Delta$}}\vphantom{}\crcr
      \smash{\hskip.7pt\raise.6pt\hbox{$\Delta$}}\vphantom{}\crcr
      \smash{\hbox{$\Delta$}}\vphantom{$\Delta$}}
      \vphantom{\hbox{$\Delta$}}
      }}
\def\bLambda{{
      \ooalign{
      \smash{\hskip.5pt\raise.5pt\hbox{$\Lambda$}}\vphantom{}\crcr
      \smash{\hskip.0pt\raise.1pt\hbox{$\Lambda$}}\vphantom{}\crcr
      \smash{\hbox{$\Lambda$}}\vphantom{$\Lambda$}}
      \vphantom{\hbox{$\Lambda$}}
      }}

\makeatletter

\def\bordermatrix#1{\begingroup \m@th
  \@tempdima 8.75\p@
  \setbox\z@\vbox{%
    \def\cr{\crcr\noalign{\kern2\p@\global\let\cr\endline}}%
    \ialign{$##$\hfil\kern2\p@\kern\@tempdima&\thinspace\hfil$##$\hfil
      &&\quad\hfil$##$\hfil\crcr
      \omit\strut\hfil\crcr\noalign{\kern-\baselineskip}%
      #1\crcr\omit\strut\cr}}%
  \setbox\tw@\vbox{\unvcopy\z@\global\setbox\@ne\lastbox}%
  \setbox\tw@\hbox{\unhbox\@ne\unskip\global\setbox\@ne\lastbox}%
  \setbox\tw@\hbox{$\kern\wd\@ne\kern-\@tempdima\left[\kern-\wd\@ne
    \global\setbox\@ne\vbox{\box\@ne\kern2\p@}%
    \vcenter{\kern-\ht\@ne\unvbox\z@\kern-\baselineskip}\,\right]$}%
  \null\;\vbox{\kern\ht\@ne\box\tw@}\endgroup}
\makeatother

\makeatletter
\def\argmin{\mathop{\operator@font arg\,min}}
\def\argmax{\mathop{\operator@font arg\,max}}
\makeatother

\def\bm#1{\mbox{\boldmath $#1$}}

\newcommand{\bea}{\begin{array}}
\newcommand{\ena}{\end{array}}
\newcommand{\beq}{\begin{equation}}
\newcommand{\enq}{\end{equation}}

\newcommand{\beqa}{\begin{eqnarray}}
\newcommand{\enqa}{\end{eqnarray}}

\newcommand{\beqan}{\begin{eqnarray*}}
\newcommand{\enqan}{\end{eqnarray*}}

\newcommand{\AL}{\begin{enumerate}}
\newcommand{\ALE}{\end{enumerate}}

\def\addots{\mathinner{
    \mkern1mu\raise0pt\vbox{\kern7pt\hbox{.}}
    \mkern2mu\raise4pt\hbox{.}
    \mkern2mu\raise7pt\hbox{.}
    \mkern1mu}}

\def\sddots{\mathinner{
    \mkern.8mu\raise7pt\hbox{.}
    \mkern.8mu\raise4pt\hbox{.}
    \mkern.8mu\raise0pt\vbox{\kern7pt\hbox{.}}
    \mkern1mu}}

\def\saddots{\mathinner{
    \mkern.2mu\raise0pt\vbox{\kern7pt\hbox{.}}
    \mkern.2mu\raise4pt\hbox{.}
    \mkern.2mu\raise7pt\hbox{.}
    \mkern1mu}}

\def\sqplus{\mathbin{
	{\ooalign{\hfil\raise.3ex\hbox{\scriptsize
	+}\hfil\crcr\mathhexbox274\crcr\mathhexbox275}}
	}} 
\def\sqminus{\mathbin{
	{\ooalign{\hfil\raise.3ex\hbox{\scriptsize
	--}\hfil\crcr\mathhexbox274\crcr\mathhexbox275}}
	}}

\def\IC{{
   \mathord{
      \hbox to 0em{
	 \hskip-4pt
         \ooalign{
	   \smash{\hskip1.9pt\raise2.6pt\hbox{$\scriptscriptstyle |$}}\crcr
	   \smash{\hbox{\rm\sf C}} }
	 \hidewidth}
      \phantom{\hbox{\rm\sf C}}
} }}
\def\IN{
    {\ooalign{
   \smash{\hskip2.2pt\raise1.5pt\hbox{$\scriptscriptstyle |$}}\vphantom{}\crcr
   \hbox{\sf N}
	}}
	} 
\def\IZ{
    {\ooalign{
   \smash{\hskip1.9pt\raise0pt\hbox{$\sf Z$}}\vphantom{}\crcr
   \hbox{\sf Z}
	}}
	} 
\def\IR{
    {\ooalign{
   \smash{\hskip2.2pt\raise1.5pt\hbox{$\scriptscriptstyle |$}}\vphantom{}\crcr
   \smash{\hskip2.2pt\raise3.3pt\hbox{$\scriptscriptstyle |$}}\vphantom{}\crcr
   \hbox{\sf R}
	}}
	} 

\DeclareMathAlphabet{\mathcmb}{OT1}{cmr}{b}{n}

\def\bSigma{\ensuremath{\mathcmb{\Sigma}}}
\def\bLambda{\ensuremath{\mathcmb{\Lambda}}}

\def\bTheta{\ensuremath{\mathcmb{\Theta}}}

\newcommand{\SI}{\begin{indlist}}
\newcommand{\EI}{\end{indlist}}

\newcommand{\DL}{\begin{dashlist}}
\newcommand{\DLE}{\end{dashlist}}

\makeatletter
\def\setboxz@h{\setbox\z@\hbox}
\def\wdz@{\wd\z@}
\def\boxz@{\box\z@}
\def\underset#1#2{\binrel@{#2}%
  \binrel@@{\mathop{\kern\z@#2}\limits_{#1}}}
\def\binrel@#1{\begingroup
  \setboxz@h{\thinmuskip0mu
    \medmuskip\m@ne mu\thickmuskip\@ne mu
    \setbox\tw@\hbox{$#1\m@th$}\kern-\wd\tw@
    ${}#1{}\m@th$}%
  \edef\@tempa{\endgroup\let\noexpand\binrel@@
    \ifdim\wdz@<\z@ \mathbin
    \else\ifdim\wdz@>\z@ \mathrel
    \else \relax\fi\fi}%
  \@tempa
}
\let\binrel@@\relax%
\makeatother

\declaretheorem[style=plain,numberwithin=section]{theorem}
\declaretheorem[style=plain,sibling=theorem]{lemma}

\newtheorem{definition}{Definition}
\declaretheorem[style=plain,unnumbered]{Remark}
\declaretheorem[style=plain,unnumbered]{corollary}

\usepackage[accepted]{icml2026}

\icmltitlerunning{Learning Long Range Spatio-Temporal Representations over Continuous Time Dynamic Graphs with State Space Models}

\begin{document}

\twocolumn[
\icmltitle{Learning Long Range Spatio-Temporal Representations over Continuous Time Dynamic Graphs with State Space Models}



\icmlsetsymbol{equal}{*}
\icmlsetsymbol{dagger}{$\dagger$}
\begin{icmlauthorlist}
 \icmlauthor{Ayushman Raghuvanshi}{equal,yyy}
\icmlauthor{Thummaluru Siddartha Reddy}{equal,comp}
\icmlauthor{Sundeep Prabhakar Chepuri}{dagger,yyy}
\icmlauthor{Mahesh Chandran}{comp}
\end{icmlauthorlist}

\icmlaffiliation{yyy}{ Department of Electrical Communication Engineering, Indian Institute of Science, Bangalore}
\icmlaffiliation{comp}{ Fujitsu Research India, Bangalore}
\vskip 0.3in
 \icmlcorrespondingauthor{Ayushman} {ayushmanr@iisc.ac.in}
\icmlcorrespondingauthor{Siddartha}{Thummaluru.Siddarthareddy@fujitsu.com}


\vskip 0.3in
]



\printAffiliationsAndNotice{\icmlEqualContribution$^{\dagger}$Part of this work is done during visit to Fujitsu Research India.}
\begin{abstract}
   Continuous-time dynamic graphs (CTDGs) provide a richer framework to capture fine-grained temporal patterns in evolving relational data. Long-range information propagation is a key challenge while learning representations, wherein it is important to retain and update information over long temporal horizons. Existing approaches restrict models to capture one-hop or local temporal neighborhoods and fail to capture multi-hop or global structural patterns. To mitigate this, we derive a parameter-efficient state-space modeling framework for continuous-time dynamic graphs \texttt{(CTDG-SSM)} from first principles. 
   We first introduce continuous-time Topology-Aware higher order polynomial projection operator (\texttt{CTT-HiPPO}), a novel memory-based reformulation of  \texttt{HiPPO} to jointly encode temporal dynamics and graph structure. The solution from \texttt{CTT-HiPPO} are obtained by projecting the classical HiPPO solution through a polynomial of the Laplacian matrix, yielding topology-aware memory updates that admit an equivalent state-space formulation for CTDGs (\texttt{CTDG-SSM}). Then a computationally efficient discrete formulation is obtained using the zero-order hold approach for model implementation.
   Across benchmarks on dynamic link prediction, dynamic node classification, and sequence classification, \texttt{CTDG-SSM} achieves state-of-the-art performance. Notably, it achieves large performance gains on datasets that require long range temporal (LRT) and spatial reasoning\footnote{Code to reproduce the results is available at: \href{https://github.com/adhocmp122/CTDG-SSM}{https://github.com/adhocmp122/CTDG-SSM}  }.     
\end{abstract}

\section{Introduction}
\label{submission}
Continuous-time dynamic graphs (CTDGs) provide a principled framework for modeling evolving relational data as a continuous stream of timestamped events, with each event capturing interactions between entities at a specific time instance~\citep{TGN}. Unlike discrete-time dynamic graphs (DTDGs), which rely on coarser snapshot intervals~\citep{kazemi2020representation}, CTDGs preserve fine-grained temporal information, making them especially well-suited for tasks such as dynamic link prediction and dynamic node classification~\citep{DYGMAMBA,TGN}. These capabilities have made CTDGs increasingly important in domains including finance, healthcare, e-commerce, and social network analysis, to name a few.  Despite initial efforts in representation learning for CTDGs, existing approaches still face two primary challenges:  
(1) \emph{long-range temporal dependencies (LRT)}: the ability to preserve and use node states and interactions over extended time horizons; and  
(2) \emph{long-range spatial dependencies (LRS)}: the ability to capture multi-hop structural interactions beyond immediate neighborhoods in dynamic graphs. 

Based on these challenges, existing models for CTDGs can be broadly categorized into two types: \textit{event-driven models} and \textit{sequence-based models}. {\textit{Event-driven models} update node states at the arrival of each interaction and capture structural context through mechanisms such as temporal random walks and graph neural networks (GNNs)-based message passing~\citep{CAWN,TGN,TGAT}}. While computationally efficient, such models mainly capture short-term temporal patterns and are weak at preserving LRT~\citep{DYGFORMER}. The second category includes \textit{sequence-based models}, which explicitly target LRT using sequence models such as Transformer or Mamba. These methods construct temporal sequences of node features and their 1-hop temporal neighbors, patch them, and process them with either Transformer or Mamba layers~\citep{DYGFORMER,DYGMAMBA}. Although effective for LRT, these models inherently restrict structural context to the local neighborhood, limiting their capacity to capture LRS~\citep{CTAN} and global spatial patterns in dynamic graphs. Modeling LRS is particularly important in domains such as financial fraud detection, where money laundering typically spans long transaction chains rather than isolated local interactions{~\citep{IBM_fin}}.      
\begin{figure}[t]
    \centering
    \includegraphics[width=0.8\linewidth]{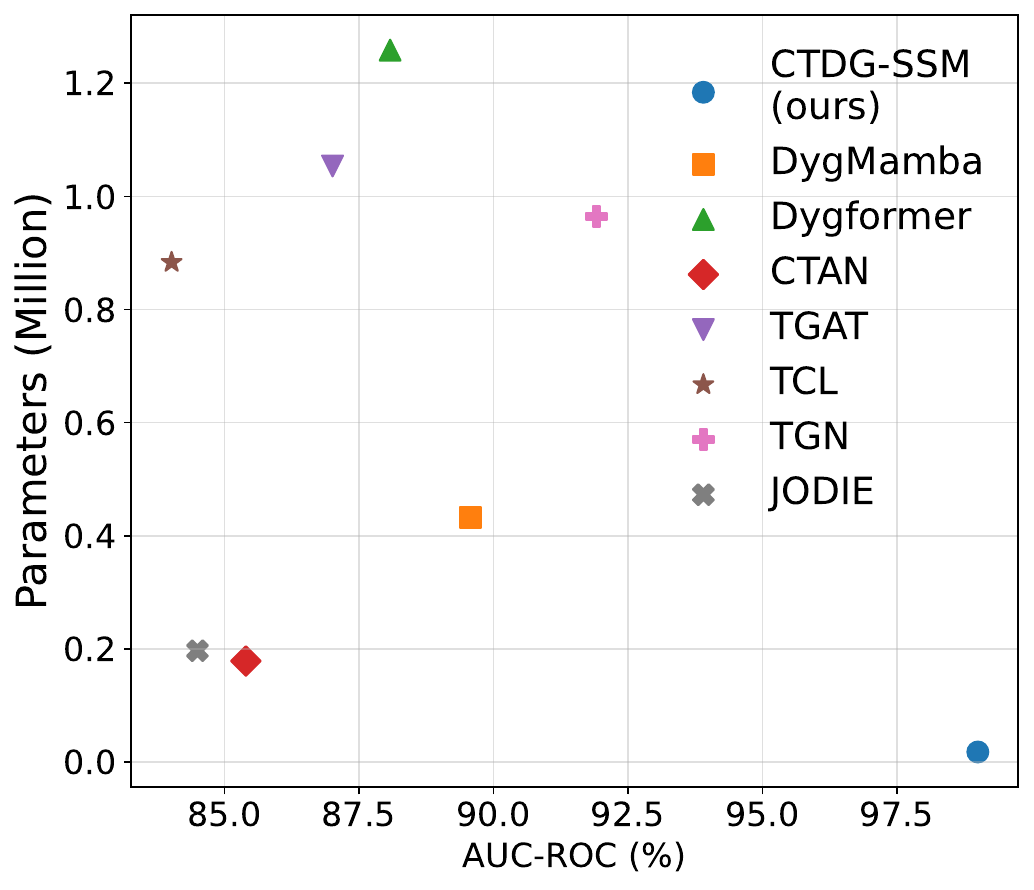}
    \vspace{-3mm}
    \caption{\footnotesize Efficiency of \texttt{CTDG-SSM} in terms of predictive performance and number of learnable parameters.}
    \label{fig:mooc_auc_params}
    \vspace{-4mm}
\end{figure}
In essence, existing methods do not simultaneously ensure LRS and LRT. 

To address this important gap and the limitations of existing methods in maintaining both LRS and LRT, we introduce a continuous-time dynamic graph state-space model (\texttt{CTDG-SSM}), which is a unified spatiotemporal state-space framework that integrates temporal memory compression in an online manner through a temporal polynomial basis and incorporates graph structure through graph filters that are polynomials of the graph Laplacian. 
To begin with, we derive a continuous-time, topology-aware higher-order polynomial projection operator (\texttt{CTT-HiPPO}), in which time-varying node signals are jointly represented using temporal and spatial polynomial bases. The coefficients of \texttt{CTT-HiPPO} are obtained by minimizing the discrepancy between observed node features and their graph-filtered polynomial approximations. This joint time-space formulation is unique in that it directly addresses the challenges of learning long-range spatial (LRS) and long-range temporal (LRT) dependencies, while capturing fine-grained temporal evolution in CTDGs. To implement \texttt{CTDG-SSM} efficiently, we discretize the continuous-time formulation using zero-order hold (ZOH), yielding the discrete counterpart of the model. The resulting \texttt{CTDG-SSM} remains lightweight, with only a small set of learnable parameters, primarily the coefficients of the graph polynomial filter and the system matrices governing state evolution. For the LRT task on the MOOC dataset, Fig.~\ref{fig:mooc_auc_params} shows that the model captures temporal patterns effectively despite its small parameter count. Using AUC-ROC and number of parameters as metrics, \texttt{CTDG-SSM} achieves top performance while using about one-tenth the parameters of competing methods.

\noindent\textbf{Contributions and main results.} We summarize the main contributions of the paper as follows: 

\begin{itemize}

\item We develop \texttt{CTT-HiPPO}, a HiPPO-based memory mechanism for CTDGs that efficiently compresses historical information from all events while maintaining LRT and LRS dependencies. One of the main results in the paper is that, by leveraging the relationship between the classical HiPPO coefficients and the coefficients of the developed \texttt{CTT-HiPPO}, we derive a novel SSM, \texttt{CTDG-SSM}, that governs the evolution of \texttt{CTT-HiPPO} via a state-space representation that captures both temporal and structural changes in CTDGs.


\item We derive a discrete form of \texttt{CTDG-SSM} \comment{using ZOH discretization} that enables efficient implementation with diagonal parameterization for scalable and stable computation.

\item We provide theoretical guarantees characterizing the robustness of \texttt{CTT-HiPPO} coefficients to graph perturbations and establish the permutation equivariance \comment{property} of  \texttt{CTDG-SSM}. These properties are crucial for real-world scenarios where \comment{continuous} data stream collection and processing are susceptible to errors and failures.
\end{itemize}

We conduct extensive experiments to assess the efficacy of our model to preserve both LRT and LRS contexts. For temporal long-range dependency, we benchmark \texttt{CTDG-SSM} on dynamic graph learning tasks such as link prediction and node classification, where it outperforms state-of-the-art methods on LRT benchmarks, including LastFM, Enron, and MOOC. The LRS dependency is evaluated by sequence classification experiment~\citep{CTAN} where \texttt{CTDG-SSM} is shown to significantly outperform current methods with information propagating over long-distances through node states with minimal decay.

\section{Related Works} 

\begin{table}[t]
\centering
\begin{tabular}{lccc}
\hline
Model & Graphs & LRS & LRT \\
\hline
Graph Mamba & Static & \tick & -- \\
GraphSSM & DTDG & \tick & -- \\
DyGMamba & CTDG & \cross & \tick \\
\textbf{CTDG-SSM(Ours)} & CTDG & \tick & \tick \\
\hline
\end{tabular}
\caption{Comparison of SSM across for graphs.}
\label{tab:model_comparison}
\end{table}

\textbf{Learning with DTDGs.} Learning on dynamic graphs can be broadly categorized into two subareas: learning for DTDGs and CTDGs. DTDGs represent data as a sequence of graph snapshots observed at discrete time intervals. Most learning algorithms for DTDGs extend static graph learning methods, such as graph convolutional networks (\texttt{GCNs}), to each snapshot and employ recurrent neural networks (\texttt{RNNs}) to capture temporal dependencies~\citep{pareja2020evolvegcn, chen2022evolving}. \texttt{GraphSSM} extends SSMs to the DTDGs to capture LRT dependencies~\citep{First_principles}. However, it assumes a fixed graph structure within each interval and then combines node embeddings from these snapshots using GNNs. 
Since CTDGs are fundamentally different from DTDGs, involving continuous graph evolution in which the set of nodes evolves over time and edges occur at irregular intervals, these methods (e.g.,~\citep{First_principles}) cannot be naively extended to CTDGs. Moreover, representing event streams using DTDGs rather than CTDGs inevitably leads to a loss of fine-grained temporal information~\citep{TGN, JODIE, Dyrep}.    

\textbf{Learning with CTDGs.} CTDGs represent dynamic graphs as streams of time-stamped events. Existing learning methods typically model either short-range or LRT dependencies, relying on random walks, message passing, or sequence modeling with Transformer or Mamba layers. Representative approaches include temporal random walks~\citep{nguyen2018dynamic,starnini2012random}, message passing architectures such as \texttt{TGAT}~\citep{TGAT}, and memory-based methods such as \texttt{TGN} and \texttt{JODIE}~\citep{TGN,JODIE}. Memory-based models that rely on \texttt{RNNs} often suffer from gradient instability (vanishing or exploding), which limits their ability to capture long-range dependencies~\citep{TGN}. To address this, recent architectures such as \texttt{DyGFormer} and \texttt{DyGmamba} employ Transformers and Mamba, respectively~\citep{DYGFORMER,DYGMAMBA}. However, these methods pre-process temporal data by restricting attention to one-hop temporal neighborhoods, thereby significantly limiting their ability to maintain LRS. In contrast, our proposed method learns node representations without imposing such structural constraints.
Specifically, our method is a graph-time state space model that directly processes interaction streams via an online state update mechanism, enabling continuous-time state evolution of structural and temporal dependencies. 

In sum, to the best of our knowledge, this is the first SSM framework for CTDGs that ensures both LRS and LRT. Notably, the proposed method can also be applied to discrete-time dynamic graphs (DTDGs), since DTDGs can be obtained by uniformly discretizing CTDGs at regular time intervals. In contrast, methods developed specifically for DTDGs \citep{First_principles} cannot be naively  extended to CTDGs, as such discretization inevitably discards fine-grained temporal patterns~\citep{souza2022provably}. A comparison with the most closely related methods is presented in Table~\ref{tab:model_comparison}.

\section{Continuous-Time Dynamic Graphs}

\paragraph{Notations.}{
We use boldface capital letters (e.g., \(\mX\)) to denote matrices and boldface lowercase letters (e.g., \(\vx\)) to denote vectors. The entry of \(\mX\) at index \((i,j)\) is written as \([\mX]_{i,j}\), and the \(i\)-th row of \(\mX\) is denoted by \(\mX_{i,:}\).
}

Consider a \emph{continuous-time} observation $\mathcal{G}(t) = (u, v, t),$ which represents a temporal edge between node $u$ and $v$ at time $t$. A CTDG~\citep{TGN}, denoted by $\cG$, is an ordered sequence of temporal interactions $\mathcal{G} = 
\{\mathcal{G}(t_1),\mathcal{G}(t_2),\ldots\}$ appearing at time instances $t_{1} < t_{2} <~\cdots$.  It should be noted that the same subset of nodes may appear in $\mathcal{G}(t_i)$ and $\mathcal{G}(t_j)$ for $i \neq j$. We capture the unique subsets of nodes that appear within a temporal window and define the underlying graph operator.

\noindent\textbf{Observed Graph.}  
For a given time $\tau \in \mathbb{R}_+$, we define the observed subgraph $\mathcal{G}_{\tau}$ of $\mathcal{G}$ as the collection of temporal interactions that occur up to time $\tau$. Formally,  
$
\mathcal{G}_{\tau} = \{ \mathcal{G}(t_i) \in \cG\;|\; t_i \le \tau \}.
$ The set of \emph{observed nodes} up to time $\tau$ consists of all the nodes that participate in any interaction in $\mathcal{G}_{\tau}$, and is denoted by 
$\mathcal{V}_{\tau} = \{ u \;|\; u \in \mathcal{G}(t_i), t_i \le \tau \}.$ Let us denote the number of nodes in $\mathcal{V}_{\tau}$ by $N_\tau = |\mathcal{V}_{\tau}|.$

\noindent\textbf{Subgraph operator and filters.}  
The temporal interactions of the observed nodes in $\mathcal{G}_\tau$ is captured by the subgraph adjacency matrix $\mA_\tau \in \mathbb{R}^{N_\tau \times N_\tau}$ with entries  
$
[\mA_{\tau}]_{u,v}= \sum\limits_{\cG(t_i) \in \cG_\tau} \mathbb{I}\big( \{u,v\} \in \mathcal{G}(t_i) \big),
$
where $\mathbb{I}(\cdot)$ is the indicator function defined as $\mathbb{I}(\{u,v\}\in \mathcal{G}(t_i)) = 1$ if $\{u,v\} \in \mathcal{G}(t_i)$, and $0$ otherwise.  We use the degree normalized Laplacian matrix defined as
$
\mL_\tau = \mI - \mD_\tau^{-1/2} \mA_\tau\mD_\tau^{-1/2},
$
where $\mD_\tau$ is the corresponding degree matrix $\mD_\tau = \operatorname{diag}(\mA_\tau \bm 1)$.  

Graph filters are expressed as matrix polynomials of the normalized Laplacian matrix. We define a $K$th-order filter as  
$
p(\mL_\tau) = \sum_{k=0}^{K-1} \alpha_k \, \mL_\tau^{k},
$
where $\{\alpha_k\}_{k=1}^{K-1}$ are learnable filter coefficients. Applying a $K$th-order filter aggregates information from up to $K$-hop neighborhoods in the subgraph $\cG_\tau$. Specifically, as $\tau$ evolves continuously with time in CTDGs, both $\mA_\tau$ and $\mL_\tau$ evolve sequentially, and thus the corresponding filters $p(\mL_\tau)$ adapt to the temporal evolution of the graph structure.  

Each node $u$ in the subgraph $\cG_\tau$ is associated with a feature vector $\vx_u(t) \in \mathbb{R}^{D_n}$.  Collecting the node features over the subgraph yields $\mX \in \mathbb{R}^{N_\tau \times D_n}$.

\section{The Proposed State-space Models for CTDGs}
In this section, we develop SSMs for CTDGs, with the objective of compressing historical event information into compact latent memory representations. We first present a HiPPO matrix~\citep{hippo} computation that incorporates graph structure as an inductive bias within latent memory representations. Specifically, we decompose node signals as graph-aware transformations of signals represented in an orthogonal polynomial space and then we develop a novel SSM model for CTDG and derive its discrete counterpart, which is useful for practical implementation.




%
To begin with, we describe the HiPPO projection for graph data.
Let $\vg(t) \in \mathbb{R}^{d\times 1}$ be a vector of orthogonal polynomials. We then model the $i$-th feature $\mX_{:,i}(t) \in \mathbb{R}^{N_{\tau}\times 1}$ on $\mathcal{V}_\tau$ as
\begin{align}\label{eqn:reformulation}
    \mX_{:,i}(t) = p(\mL_\tau)\mH_{\tau}^{(i)} \vg(t) + \vr_i(t) , ~~ \forall t< \tau.
\end{align}
Here, \(\mH_{\tau}^{(i)} \in \mathbb{R}^{N_{\tau} \times d}\) is the coefficient matrix corresponding to the \(i\)-th feature, for \(i = 1, \dots, D_n\). The polynomial \(p(\mL_\tau)\) of the normalized Laplacian acts as a graph filter that incorporates the topology of \(\cG_\tau\) by aggregating the HiPPO coefficients in \(\mH_{\tau}^{(i)}\). The error term \(\vr_i(t) \in \mathbb{R}^{N_{\tau}}\) accounts for model mismatch.

The coefficients \(\mH_{\tau}^{(i)}\) are then obtained by minimizing the residual over the temporal window \([0,\tau]\), as follows
\begin{align}
  \min_{\mH_{\tau}^{(i)}} ~~ \int_0^\tau \left\| \mX_{:,i}(t) - p(\mL_\tau) \mH_{\tau}^{(i)} \vg(t) \right\|^2_2 d\mu(t), 
  \label{eq:residual}
\end{align}
where $\mu(t)$ is the measure under which the orthogonality of $\vg(t)$ is defined. Although the above formulation provides a general framework for modeling the HiPPO coefficients for graph data with a learnable $ K$th-order graph filter, it is related under the setting $p(\mL_\tau) = \mI$ to the approaches in ~\citep{First_principles} that instead uses a quadratic Laplacian regularizer in~\eqref{eq:residual} and to the classical HiPPO formulation (without any graph structure)~\cite{hippo}. 


Now, to find the optimal set of coefficients $\mH_{\tau}^{(i)}$, we use the first-order optimality condition (detailed derivation can be found in Appendix~\ref{sec:CTT-Coefficients}) to obtain 
\begin{align}
\label{eqn:graph_coff_hippo_coff_relation}
&p(\mL_\tau) \mH_{\tau}^{(i)} = \int_0^\tau \mX_{:,i}(t)\vg(t)^\top d\mu(t)   = \mH_{\tau}^{(i),\texttt{HiPPO}}\\
&\mH_{\tau}^{(i)} = p(\mL_\tau)^{-1}  \mH_{\tau}^{(i),\texttt{HiPPO}} 
\end{align}
where $\mH_{\tau}^{(i),\texttt{HiPPO}}$ denotes the solution to the classical HiPPO formulation without any graph structure~\citep{hippo}, and by the choice of $\mL_\tau$,  $p(\mL_\tau)^{-1}$ is well-defined. From \eqref{eqn:graph_coff_hippo_coff_relation}, it can be seen that the \texttt{CTT-HiPPO} coefficients $\mH_{\tau}^{(i)}$ are essentially the graph-aware extension of the classical HiPPO coefficients, obtained by projecting $\mH_{\tau}^{(i),\texttt{HiPPO}}$ through the inverse polynomial graph filter. Although we provide the solution $\mH_{\tau}^{(i)}$ for a single feature $i$, it can be easily extended to multiple features along the lines as above. Henceforth, for brevity, we drop the index $i$ in  $\mH_{\tau}^{(i)}$ and $\mX_{:,i}(t)$ and simply use $\mH_\tau$ and $\mX(t)$.


\subsection{The CTDG State-Space Model}
We now present the main result of the paper, i.e., the state-space formulation that governs the evolution of the memory coefficients $\mH_\tau$. In SSMs, the temporal dynamics of an input signal are modeled through the progression of latent memory representations (state vectors). Accordingly, we model the evolution of CTDGs over time through the evolution of the memory coefficient matrix $\mH_\tau$, which jointly captures both temporal and topological structures. We refer to the proposed SSM for CTDG as \texttt{CTDG-SSM}, whose model is described in the next theorem.

\begin{theorem} [\emph{\texttt{CTDG-SSM}}] \label{theorem:4.1}
Consider a interval $s\in[\tau,\tau_+)$ with CTDGs $\cG_{\tau}$ and $\cG_{\tau_+}$ Let $\cG_\tau$ denote a CTDG at time $\tau$, and a new observation $\cG(\tau_+)$ with corresponding CTDG $\cG_{\tau_+}$. The evolution of the memory coefficients $\mH_{s}$ for $s \in [\tau,\tau_+)$ admits the following state-space representation:
\begin{align}\label{eq:structural-ssm}
\frac{d\mH_{s}}{ds}
\hspace{-0.1cm}= \hspace{-0.1cm} -  \frac{\mH_{s}\mA^\top}{M(s)}-
p(\mL_{s})^{-1} \frac{dp(\mL_{s})}{ds} \mH_{s}+\hspace{-0.1cm}
\frac{p(\mL_{s})^{-1} \mX({s})\mB^\top}{M(s)}
\end{align}
where $\mA \in \mathbb{R}^{d\times d}$ is the state-transition matrix that depends on the choice of the orthogonal polynomial $\vg(\cdot)$, $\mB \in \mathbb{R}^{d\times 1}$ is the input matrix, and $M(s): \mathbb{R}_+ \to \mathbb{R}_+$ is a normalization term that depends on the choice of the measure $\mu(t)$. Here, $\mL_{s} \in \mathbb{R}^{N_{\tau_+}\times N_{\tau_+}}$, $\mX(s) \in \mathbb{R}^{N_{\tau_+} \times 1}$, $p(\mL_{s}) = \frac{\tau-s}{\tau-\tau_+} p(\mL_{\tau_+}) + \frac{\tau_+-s}{\tau_+-\tau}p(\mL_{\tau})$, and $\mH_s \in \mathbb{R}^{N_{\tau_+}\times d}$ for $s\in[\tau,\tau_+)$  
\footnote{To match the dimension of $\mL_\tau$ and $\mL_{\tau_+}$ in \eqref{eq:structural-ssm}, we construct $\mL_\tau$ by removing the edges observed in $\cG(\tau+)$ from $\mL_{\tau_+}$.}.


\end{theorem}
The proof of this theorem is relegated to Appendix~\ref{sec:proof_4.1}. The result directly follows from the equivalence between the classical~\texttt{HiPPO} coefficients and a linear ODE (Theorem 1 in~\citep{hippo}) characterized by the state matrix $\mA$ and input matrix $\mB$, and more importantly, incorporating the fact that $\mL_\tau$ depends on $\tau$ in CTDGs. We end this subsection with the following remark that explicitly connects \texttt{CTDG-SSM} to~\citep{hippo} and~\citep{First_principles}.

\begin{Remark}
\eqref{eq:structural-ssm} shows that the graph filter $p(\mL_\tau)$ modifies the classical $\operatorname{HiPPO}$ dynamics by introducing time-dependent graph-aware terms that account for the change in temporal evolution of the graph. When the polynomial of Laplacian is static or fixed, we have $\tfrac{d p(\mL_s)}{ds} = 0$. Thus, \texttt{CTDG-SSM} reduces to the piecewise constant \comment{constant} SSM variant in~\citep{First_principles}.

When there is no graph, i.e., $p(\mL_\tau)=\mI$, \texttt{CTDG-SSM} reduces exactly to classical SSM~\citep{hippo}. 
\end{Remark}

\subsection{ Discretized CTDG-SSM}
 We now describe the discrete-time version of \texttt{CTDG-SSM} in this section. 

 In practice, the continuous-time SSMs are discretized using ZOH, which assumes piecewise constant inputs, i.e., $\mX(t) = \mX[k]$ for $t \in [t_{k-1}, t_{k})$ with step size $\Delta[k] = t_{k}-t_{k-1}$. Here $t_{k-1}$ is $\tau$ and $t_{k}$ is $\tau_+$. Where $\mL[k] \in \mathbb{R}^{N_{\tau_+}\times N_{\tau_+}}$ is obtained using a subgraph $\mathcal{G}_{\tau_+}$ and $\mL[k-1]  \in \mathbb{R}^{N_{\tau_+}\times N_{\tau_+}}$ is obtained by removing the newly observed edges in $\Delta[k]$ time interval from $\mL[k]$.
 
\begin{theorem}[\emph{Discrete \texttt{CTDG-SSM}}] \label{them:4.3}
Let $\mX[k]$ denote the input at time $t_{k}$, and let the temporal graph structures at times $t_{k}$ and $t_{k-1}$ be represented by the Laplacians $\mL[k]$ and $\mL[k-1]$, respectively. Then for $\Delta[k] = t_{k}-t_{k-1}$, the memory update of the proposed $\texttt{CTDG-SSM}$ model is governed by the following discrete-time recursion:
\begin{align} &\mH[k+1] = \Bar{\mA}_{\mL[k]} \, \mH[k] \, \Bar{\mA} + \bar\mB(\mL[k], \mX[k]) \label{eq:discrete-update} 
\end{align}
Here, $\Bar{\mA}_{\mL[k]} = \exp \left(- p(\mL[k])^{-1} \left( p(\mL[k]) - p(\mL[k-1]) \right) \right)$, $\Bar{\mA} = \exp(-\Delta[k] \mA^\top )$, and $\bar\mB(\mL[k], \mX[k]) = \int_0^1 (\Bar{\mA}_{\mL[k]})^s \, p(\mL[k])^{-1} \mX[k] \mB^\top (\Bar{\mA})^s \, \Delta[k] \, ds.$
\end{theorem}

We present the detailed proof in Appendix~\ref{appendix:them_4.2}. The proof proceeds by first simplifying Equation~\eqref{eq:structural-ssm} to standard state-space form with system and input matrices, leveraging the properties of the Kronecker structure. We then apply the ZOH discretization to this form and subsequently factorize the discretized equations to obtain the final expression.
The discrete memory update in \eqref{eq:discrete-update} is structurally analogous to the vanilla Mamba update~\citep{gu_mamba_2024}, with two key distinctions: there are two state-transition matrices that jointly operate on the state variable, and the input-dependent component $\bar\mB(\mL[k], \mX[k])$ does not admit a closed-form solution.   
{
\begin{Remark}
The invertibility of the matrix \( p(\mL[k]) \) involved in \eqref{eq:discrete-update} is ensured by restricting the set of feasible polynomials filters to those satisfying \( p(y) \neq 0 \) for all \( y \in [0,2] \). Since \( \mL[k] \) is a normalized symmetric graph Laplacian, its spectrum satisfies \( \lambda[k] \in [0,2] \). Therefore, for any feasible polynomial we have \( p(\lambda[k]) \neq 0 \), which implies that \( p(\mL[k]) \) is invertible. It is also important to note that the computational complexity of this operation is influenced by the number of active batch nodes $N_B$ and the batch size, which are controllable. More detailed discussion on runtime complexity analysis is relegated to  Appendix~\ref{sec:time_complexity} and ~\ref{Sec:runtimeanalysis}.  

\end{Remark}
}
\begin{figure*}[t]
    \centering
\includegraphics[width=0.9\linewidth]{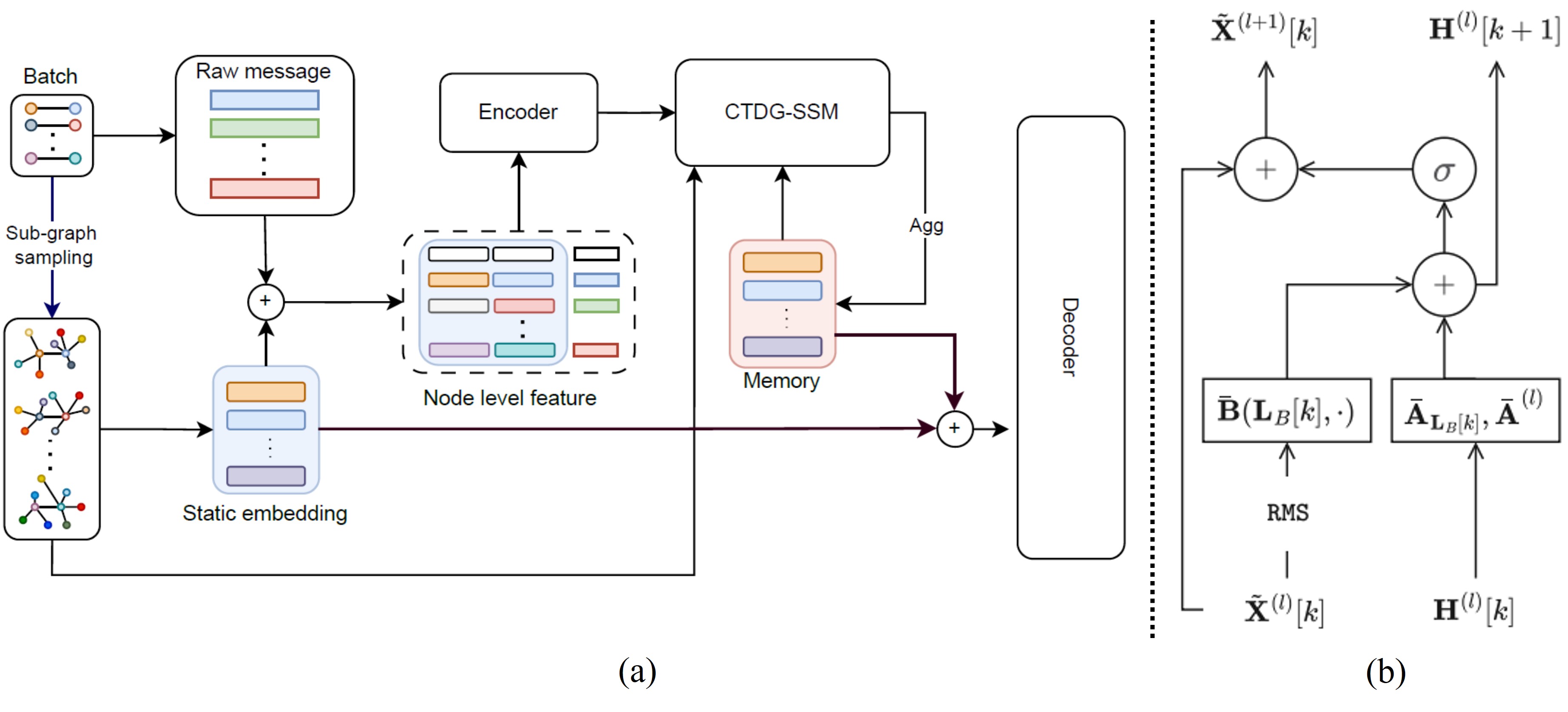}
    \caption{(a). Architecture of the \texttt{CTDG-SSM} framework. Events are batched from the input event stream, and a batch-level subgraph is constructed via subgraph sampling. Raw messages are combined with static embeddings to form node-level features, which are encoded and processed by the CTDG-SSM module to update dynamic memory. The updated memory is then aggregated with static embeddings to produce the final node representations. (b) Illustration of a single SSM layer. At time step \(t_k\), the input \(\tilde{\mathbf{X}}^{(l)}[k]\) is RMS-normalized and projected using \(\tilde{\mathbf{B}}*{L_B[k], \cdot}\), while the hidden state \(\mathbf{H}^{(l)}[k]\) is transformed via the state transition matrix \(\bar{\mathbf{A}}*{L_B[k]} \text{ and } \bar{\mathbf{A}}^{(l)}\). The resulting terms are combined through additive updates and a nonlinear activation \(\sigma(\cdot)\) to produce the updated hidden state \(\mathbf{H}^{(l)}[k+1]\) and output \(\tilde{\mathbf{X}}^{(l+1)}[k]\).}
    \label{fig:simple-flow}
\end{figure*}
\section{Architecture} \label{sec:Architecture}
In this section, we introduce the proposed architecture that implements discrete \texttt{CTDG-SSM}. The overall modular design is illustrated in Fig.~\ref{fig:simple-flow}(a). It mainly consists of three blocks: 
(a) Subgraph sampler: constructs $N_u$-temporal neighborhoods for each node.
(b) Node feature encoder: integrates node, edge, and temporal information into node feature representations.
(c) \texttt{CTDG-SSM} module: generates memory representations that capture LRT dependencies and structural context. 

\noindent\textbf{Subgraph sampling.} At each training step, we construct a mini-batch of temporal interactions by grouping together $B$ chronologically consecutive events. From this batch, we develop a batch level Laplacian $\mL_B[k] \in \mathbb{R}^{N_B \times N_B}$ by generating subgraphs via a neighborhood-based sampling strategy: for every node participating in an event, we sample up to $N_u$ of its most recent neighbors, where $N_u$ defines the spatial context size. To estimate $\mL_B[k-1]$, we remove the current batch interaction edges from $\mL_B[k]$ while preserving the neighborhood edges of the $N_u$ neighbors. This subgraph-based approach is motivated by two factors: (i) it captures information from the multi-hop temporal neighborhood, and (ii) it enables the model to update states for nodes beyond those directly involved in the observed interactions, thereby incorporating both local structural dependencies and broader temporal context.

\noindent\textbf{Node feature encoder.} 
We construct the batch input features $\mX^{(B)}\in \mathbb{R}^{N_B \times D_B}$ \comment{for the current batch} by concatenating node-specific features, temporal neighbor features, edge attributes, and the corresponding timestamp information of events in the batch. For an interaction event $\cG(t_i) = (u,v,t_i)$ with edge feature $\vx_{uv}$, the feature vectors for the participating nodes are defined as
$\mX^{(B)}_{u,:} = [\vx_u(t_i)|| \vx_v(t_i) || \vx_{u,v}|| \phi(\Delta t_i)]$, and $\mX^{(B)}_{v,:} = [\vx_u(t_i)|| \vx_v(t_i) || \vx_{u,v} || \phi(\Delta t_i)]$. Here, $\vx_u$ and $\vx_v$ denote the static embeddings of nodes $u$ and $v$ concatenated with their raw features, and $\phi(\cdot)$ denotes a fixed (non-trainable) time-encoding function. The term $\Delta t_i$ corresponds to the inter-event time since the last occurrence of $(u,v)$; for first-time interactions, $\Delta t_i$ is assigned a large constant following prior works~\citep{DYGMAMBA}.

\textbf{Encoder.} The encoder $h_\theta$ takes the input feature matrix $\mX^{(B)}$ and projects it into a latent space of $d$-dimension. These projected features are then used to update the memory representation through the \texttt{CTDG-SSM} recurrence. In experimentation, we implement the encoder as a 2-layer neural network and represent augmented and projected node features as $h_\theta(\mX^{(B)}) = \Tilde{\mX} [k] \in \mathbb{R}^{N_{B} \times d}$.

\textbf{Learnable CTDG-SSMs.} The \texttt{CTDG-SSM} block computes node memory representations according to \eqref{eq:discrete-update}. 
While a single-layer \texttt{CTDG-SSM} is sufficient to capture linear state-space dynamics, stacking multiple layers enables the model to learn richer temporal feature transformations. To enhance representational capacity, we incorporate residual connections, \texttt{RMSNorm} normalization, and the \texttt{GeLU} activation following design principles from Mamba~\citep{gu_mamba_2024} (see Fig.~\ref{fig:simple-flow}(b)).
To elaborate, given the output of $(l-1)$-th layer denoted as $\Tilde{\mX}^{(l)}[k]$, the $l$-th layer performs the following sequence of operations:
$\mH^{(l)}[k+1] = \texttt{CTDG-SSM}(\operatorname{RMS}(\Tilde{\mX}^{(l)}[k]), \mL_B[k])$, and $\Tilde{\mX}^{(l+1)}[k] = \Tilde{\mX}^{(l)}[k] + \sigma\big(\mH^{(l)}[k+1]\big)$,
where $\operatorname{RMS}(\cdot)$ denotes RMS normalization and $\sigma$ is a nonlinear activation function. We use the \texttt{GeLU} activation, which promotes stable training and ensures well-scaled feature transformations. The input for the first layer, i.e,  $\Tilde{\mX}^{(0)}[k] = \Tilde{\mX}[k]$, is the projected node features.   
For nodes participating in multiple events within the same batch, we apply a \textit{mean} aggregator to obtain a single consolidated representation.

\textbf{Memory.} The memory module maintains the latent representations of all nodes.  These are initialized as zero vectors of dimension 
$d$. After each batch, the memory is updated with the newly computed representations of the nodes involved in the current interactions and their sampled neighbors from the observed graph.

\textbf{Decoder.} For downstream tasks such as link prediction and node classification, the decoder operates on the memory representations of the target nodes.  The complete implementation algorithm of the proposed architecture is detailed in Appendix~\ref{sec:thm_pert}.

\textit{Link Prediction.}  
Given a query of the form $(u,v,T)$, we first retrieve the static embeddings and dynamic memory states of nodes $u$ and $v$, denoted $\vh_u$ and $\vh_v$. This representation is augmented with a learnable temporal embedding $\psi(\Delta t)$, where $\Delta t = T - t_{\text{last}}$ and $t_{\text{last}}$ denotes the most recent interaction time between $u$ and $v$. The concatenated vector$\big[\, \vh_u \,\|\, \vh_v \,\|\, \psi(\Delta t) \,\big]$ is then passed through a linear layer to produce an edge score.

\textit{Node Classification.}  
For a query of the form $(u, v, T)$ or $(u,T)$, only the representation of node $u$ is used. The decoder applies a linear mapping to $\vh_u$, optionally concatenated with available temporal information, to produce a multi-class probability vector corresponding to the predicted node label.


\begin{table*}
\centering 
\caption{AUC-ROC of dynamic link prediction with random negative sampling under T: Transductive, and I: Inductive setup. Best-performing model per dataset is shown in bold.}
\resizebox{\textwidth}{!}{
\begin{tabular}{llccccccccccc}
\toprule \label{tab:link_pred_main}
\textbf{Setup} & \textbf{Datasets} & JODIE & DyRep & TGAT & TGN & CAWN & TCL & GraphMixer & DyGFormer & CTAN & DyGmamba & \texttt{CTDG-SSM}  \\
\midrule
\multirow{7}{*}{T} 
& LastFM & 70.89  $\pm$  1.97 & 71.40  $\pm$  2.12 & 71.47  $\pm$  0.14 & 76.64  $\pm$  4.66 & 85.92  $\pm$  0.16 & 71.09  $\pm$  1.48 & 73.51  $\pm$  0.14 & 93.03  $\pm$  0.11 & 85.12  $\pm$  0.77 & 93.31  $\pm$  0.18 & \textbf{93.79 $\pm$  0.22} \\
& Enron &87.77 $\pm$ 2.43 & 83.09 $\pm$ 2.20 & 68.57 $\pm$ 1.46 & 88.72 $\pm$ 0.95 & 90.34 $\pm$ 0.23  & 83.33 $\pm$ 0.93 & 84.16 $\pm$ 0.34 & 93.20 $\pm$ 0.12 & 87.09 $\pm$ 1.51 & 93.34 $\pm$ 0.23 & \textbf{94.98  $\pm$  2.92} \\
& MOOC & 84.50 $\pm$ 0.60 & 84.50 $\pm$ 0.87 & 87.01 $\pm$ 0.16 & 91.91 $\pm$ 0.82 & 80.48 $\pm$ 0.41  & 84.02 $\pm$ 0.59 & 84.04 $\pm$ 0.12 & 88.08 $\pm$ 0.50 & 85.40 $\pm$ 2.67 & 89.58 $\pm$ 0.12  & \textbf{99.00  $\pm$  0.33} \\
& Reddit & 98.29 $\pm$ 0.05 & 98.13 $\pm$ 0.04 & 98.50 $\pm$ 0.01 & 98.61 $\pm$ 0.05 & 99.02 $\pm$ 0.00  & 97.67 $\pm$ 0.01 & 97.17 $\pm$ 0.02 & 99.15 $\pm$ 0.01 & 97.24 $\pm$ 0.75 & 99.27 $\pm$ 0.01 & \textbf{99.48  $\pm$  0.02}\\
& Wikipedia & 96.36 $\pm$ 0.14 & 94.43 $\pm$ 0.32 & 96.60 $\pm$ 0.07 & 98.37 $\pm$ 0.10 & 98.54 $\pm$ 0.01  & 97.27 $\pm$ 0.06 & 96.89 $\pm$ 0.04 & 98.92 $\pm$ 0.03 & 97.00 $\pm$ 0.21 & 99.08 $\pm$ 0.02 & \textbf{99.33 $\pm$ 0.08} \\
& UCI & 90.35 $\pm$ 0.51 & 69.46 $\pm$ 2.66 & 78.76 $\pm$ 1.10 & 92.03 $\pm$ 0.69 & 93.81 $\pm$ 0.23  & 85.49 $\pm$ 0.82 & 91.62 $\pm$ 0.52 & 94.45 $\pm$ 0.22 & 76.25 $\pm$ 2.83 & \textbf{94.77 $\pm$ 0.18}  & 89.24 $\pm$ 0.43\\
& Social Evo. &92.13 $\pm$ 0.20 & 90.37 $\pm$ 0.52 & 94.93 $\pm$ 0.06 & 95.31 $\pm$ 0.27 & 87.34 $\pm$ 0.10 & 95.45 $\pm$ 0.21 & 95.21 $\pm$ 0.07 & 96.25 $\pm$ 0.04 & Timeout & 96.38 $\pm$ 0.02 &  \textbf{99.10 $\pm$ 0.49} \\
\midrule
& \textbf{Avg. Rank} & 7.93 & 9.36 & 7.86 & 4.57 & 5.71 & 8.00 & 7.71 & 3.00 & 7.50 & 2.00 & \textbf{1.86} \\
\midrule
\multirow{7}{*}{I} 
& LastFM & 83.13 $\pm$ 1.19 & 83.47 $\pm$ 1.06 & 78.40 $\pm$ 0.30 & 81.18 $\pm$ 3.27 & 89.33 $\pm$ 0.06 & 81.38 $\pm$ 1.53 & 82.07 $\pm$ 0.31 & 94.17 $\pm$ 0.10 & 60.40 $\pm$ 3.01 & 94.42 $\pm$ 0.21 & \textbf{94.49 $\pm$  0.27} \\
& Enron &  78.97 $\pm$ 1.59 & 73.97 $\pm$ 3.00 & 66.67 $\pm$ 1.07 & 78.76 $\pm$ 1.69 & 86.30 $\pm$ 0.56 & 82.61 $\pm$ 0.61 & 75.55 $\pm$ 0.81 & 89.62 $\pm$ 0.27 & 74.61 $\pm$ 1.64 & 89.67 $\pm$ 0.27& \textbf{93.66  $\pm$  4.67}\\
& MOOC& 80.57 $\pm$ 0.52 & 80.50 $\pm$ 0.68 & 85.28 $\pm$ 0.30 & 88.01 $\pm$ 1.48 & 81.32 $\pm$ 0.42 & 82.28 $\pm$ 0.99 & 81.38 $\pm$ 0.17 & 87.05 $\pm$ 0.51 & 64.99 $\pm$ 2.24 & 88.64 $\pm$ 0.08 & \textbf{98.67  $\pm$  0.46} \\
& Reddit  & 96.43 $\pm$ 0.16 & 95.89 $\pm$ 0.26 & 97.13 $\pm$ 0.04 & 97.41 $\pm$ 0.12 & 98.62 $\pm$ 0.01 & 95.01 $\pm$ 0.10 & 95.24 $\pm$ 0.08 & 98.83 $\pm$ 0.02 & 80.07 $\pm$ 2.53 & 98.97 $\pm$ 0.01 & \textbf{99.13  $\pm$  0.03} \\
& Wikipedia & 94.91 $\pm$ 0.32 & 92.21 $\pm$ 0.29 & 96.26 $\pm$ 0.12 & 97.81 $\pm$ 0.18 & 98.27 $\pm$ 0.02 & 97.48 $\pm$ 0.06 & 96.61 $\pm$ 0.04 & 98.58 $\pm$ 0.01 & 93.58 $\pm$ 0.65 & 98.77 $\pm$ 0.03 & \textbf{99.06  $\pm$  0.10} \\
& UCI & 79.73 $\pm$ 1.48 & 58.39 $\pm$ 2.38 & 79.10 $\pm$ 0.49 & 87.81 $\pm$ 1.32 & 92.61 $\pm$ 0.35 & 84.19 $\pm$ 1.37 & 91.17 $\pm$ 0.29 & 94.45 $\pm$ 0.13 & 49.78 $\pm$ 5.02 & \textbf{94.76 $\pm$ 0.19} & 87.43  $\pm$  0.79 \\
& Social Evo. & 91.72 $\pm$ 0.66 & 89.10 $\pm$ 1.90 & 91.47 $\pm$ 0.10 & 90.74 $\pm$ 1.40 & 79.83 $\pm$ 0.14 & 92.51 $\pm$ 0.11 & 91.89 $\pm$ 0.05 & 93.05 $\pm$ 0.10 & Timeout & 93.13 $\pm$ 0.05 & \textbf{98.60 $\pm$ 0.14} \\
\midrule
\textbf{} & \textbf{Avg. Rank} 
& 7.29 & 9.00 & 8.00 & 6.00 & 5.29 & 6.57 & 6.71 & 3.00 & 10.57 & 1.86 & \textbf{1.71} \\ 
\bottomrule
\end{tabular}}
\end{table*}

\section{Theoretical Characterization}

In this section, we derive the robustness and permutation equivariance \comment{properties} of \texttt{CTDG-SSM}. In particular, the robustness property characterizes the stability of memory representations under structural perturbations and is crucial given that real-world temporal graphs may include spurious edges.

\begin{theorem}[\emph{Robustness property}] \label{thm:pertrurbation}
Let $\bar{\mL} = \mL + \Delta\mL$ be the perturbed graph Laplacian with $\|\Delta\mL\|_{2} \leq \epsilon$. Then the error between the perturbed and true coefficients is bounded linearly in terms of the energy of the perturbed graph Laplacian as  $\frac{\|\hat{\mH}_{\tau}-{\mH}_{\tau}\|_{2}}{\|\mH_{\tau}\|_{2}}\leq \epsilon \Gamma$, where $\Gamma = \frac{\lambda_{2}\lambda_{c}}{\lambda^2_{1}}$ with $\lambda_1:= \min_{y\in [0,2] } |p(y)| >0$, $\lambda_2:= \max_{y\in [0,2] }|p(y)|$, and $\lambda_c:= \max_{y\in[0,2]} |\frac{dp(y)}{dy}|$. 
\end{theorem}

    We relegate the proof to  Appendix \ref{sec:thm_pert}. The derivations follow by using the triangle inequality and exploiting spectral bounds of the normalized graph Laplacian. The derived error bound shows that the deviation between the perturbed and true coefficients scales linearly with the energy of the perturbed Laplacian $\Delta\mL$ i.e., small structural perturbations in the underlying graph induce proportionally small deviations in the coefficients. Hence, the representations produced by \texttt{CTT-HiPPO} are stable and robust with respect to perturbations.

\begin{theorem}[\emph{Permutation Equivariance}] \label{thm:permutation_equivariance}
Let \(\mathcal{P} = \{\, \bm \Pi \in\{0,1\}^{N_{\tau}\times N_{\tau}} : \bm \Pi^\top\bm \Pi = \bm \Pi \bm \Pi^\top = \mI_{N_{\tau}}\}\) be the set of all \(N_{\tau}\times N_{\tau}\) permutation matrices.  Then under the permutation of the graph Laplacian \(\mL[k]\) and node-features \(\mX\) by any \(\bm \Pi\in\mathcal{P}\), the representations from \texttt{CTDG-SSM} also modifies as $\bar{\mH}[k+1] = \bm \Pi\mH[k+1]$.
\end{theorem}
    We relegate the proof to the Appendix \ref{sec:proof_thm4.2}. The permutation equivariance property guarantees that, when the nodes in the observed CTDGs and their associated signals are permuted, the representations by \texttt{CTDG-SSM} permute in exactly the same way, thereby preserving equivariance.
 
\section{Numerical Experiments}
We evaluate the proposed algorithm on two downstream temporal graph learning tasks, namely dynamic link prediction and node classification.  Further, to assess the model's ability to preserve long-range information, we test it on a sequence classification task.

\noindent\textbf{Baseline models.} For all the three tasks, we compare the performance of our model against the following state of the art algorithms, namely, \texttt{JODIE}~\citep{JODIE},  \texttt{DyRep}~\citep{Dyrep}, \texttt{TGN}~\citep{TGN},   \texttt{TGAT}~\citep{TGAT},   \texttt{GraphMixer}~\citep{GraphMIXER}, \texttt{DyGFormer}~\citep{DYGFORMER},  \texttt{CTAN}~\citep{CTAN}, \texttt{DyGmamba}~\citep{DYGMAMBA}. 
For dynamic link prediction and node classification tasks, we also consider models \texttt{Edgebank}~\citep{poursafaei2022towards}, \texttt{CAWN}~\citep{CAWN}, and \texttt{TCL}\citep{tcl} for comparison.

\subsection{Implementation Details }
For link prediction and node classification, we follow the experimental protocol of \citet{DYGFORMER} and compare \texttt{CTDG-SSM} with established baselines. For sequence classification, we adopt the setup from \citet{CTAN}. The model is trained with binary cross-entropy using the Adam optimizer; additional hyperparameter details are provided in Appendix~\ref{sec:Additional results}. We train for up to 200 epochs with early stopping and select the best validation model for testing. Experiments are conducted on two machines equipped with NVIDIA A6000 and RTX 8000 GPUs (48 GB).

\subsection{Dynamic Link Prediction}
In this section, we present results on dynamic link prediction where the task is to predict the existence of an edge between two nodes at a given time. We evaluate the proposed algorithm in both transductive (test nodes are observed during training) and inductive (test nodes are unseen during training) settings, under different sampling strategies (random, historical, and inductive) for generating negative samples.
Experiments are performed on benchmark temporal link prediction datasets~\citep{poursafaei2022towards} the details of which are provided in Appendix~\ref{sec:data_Sec}.

\begin{table*}[t]
\centering

\begin{minipage}[t]{0.40\textwidth}
\centering
\caption{Performance comparison on dynamic node classification with AUC-ROC over 5 runs.}
\label{tab:node_classification}
\footnotesize
\begin{tabular}{lccc}
\hline
Method & Wikipedia & Reddit & Avg. Rank \\
\hline
\midrule
\texttt{JODIE}        & 88.10 $\pm$ 1.57 & 59.53 $\pm$ 3.18 & 7.14 \\
\texttt{DyRep}        & 87.41 $\pm$ 1.94 & 63.12 $\pm$ 0.51 & 8.86 \\
\texttt{TGAT}         & 83.42 $\pm$ 2.92 & 69.31 $\pm$ 2.18 & 7.14 \\
\texttt{TGN}          & 85.51 $\pm$ 3.28 & 63.21 $\pm$ 3.00 & 3.86 \\
\texttt{CAWN}         & 84.59 $\pm$ 1.16 & 65.22 $\pm$ 0.79 & 4.86 \\
\texttt{TCL}          & 79.03 $\pm$ 1.18 & 68.04 $\pm$ 2.00 & 7.29 \\
\texttt{GraphMixer}   & 85.60 $\pm$ 1.73 & 64.42 $\pm$ 1.15 & 7.14 \\
\texttt{DyGFormer}    & 86.35 $\pm$ 2.19 & 67.67 $\pm$ 1.39 & 2.14 \\
\texttt{CTAN}         & 87.38 $\pm$ 0.14 & 67.29 $\pm$ 0.15 & 7.29 \\
\texttt{DyGmamba}     & 87.44 $\pm$ 0.82 & 67.70 $\pm$ 1.32 & 1.14 \\
\texttt{CTDG-SSM}     & \textbf{88.61 $\pm$ 0.64} & \textbf{69.50 $\pm$ 0.82} & \textbf{1.00} \\
\bottomrule \\
\end{tabular}
\end{minipage}
\hfill
\begin{minipage}[t]{0.5\textwidth}
\centering
\caption{Sequence classification performance. Values represent the mean accuracy, with standard deviation reported in parentheses.}
\label{tab:mamba_seq_classification}
\footnotesize
\setlength{\tabcolsep}{2pt}
\renewcommand{\arraystretch}{0.70}
\begin{tabular}{lccccc}
\hline
 Method & $n=3$ & $n=9$ & $n=15$ & $n=20$  \\
\hline
\midrule
\texttt{DyRep} &
$100.0_{(0.0)}$ & $47.93_{(2.73)}$ & $48.60_{(2.48)}$ & $50.47_{(2.88)}$  \\

\texttt{GraphMixer} &
$100.0_{(0.0)}$ & $52.80_{(5.56)}$ & $52.49_{(15.36)}$ & $52.04_{(8.20)}$ \\

\texttt{JODIE} &
$100.0_{(0.0)}$ & \textbf{$100.0_{(0.0)}$} & $60.0_{(14.91)}$ & $50.87_{(2.46)}$  \\

\texttt{TGAT} &
$100.0_{(0.0)}$ & $47.87_{(2.72)}$ & $50.53_{(2.15)}$ & $49.07_{(1.55)}$ \\

\texttt{TGN} &
$100.0_{(0.0)}$ & $48.13_{(1.63)}$ & $48.67_{(2.76)}$ & $50.13_{(2.17)}$  \\

\texttt{CTAN} &
$100.0_{(0.0)}$ & $\textbf{99.93}_\textbf{{(0.21)}}$ & $93.47_{(8.78)}$ & $88.93_{(12.06)}$  \\

\texttt{TU-SSM} &
$47.0_{(1.12)}$ & $50.73_{(1.74)}$ & $52.26_{(2.44)}$ & $54.46_{(0.73)}$  \\

\texttt{DyGFormer} &
$100.0_{(0.0)}$ & $53.02_{(6.06)}$ & $42.80_{(16.25)}$ & $42.79_{(19.62)}$  \\

\texttt{DyGmamba} &
$100.0_{(0.0)}$ & $54.01_{(6.06)}$ & $45.60_{(12.25)}$ & $45.29_{(17.62)}$ \\

\texttt{CTDG-SSM} (FO) &
$100.0_{(0.0)}$ & $97.06_{(0.44)}$ & $97.40_{(0.20)}$ & $97.13_{(0.89)}$  \\

\texttt{CTDG-SSM} (SO) &
$\textbf{100.0}_\textbf{{(0.0)}}$ & $98.13_{(0.58)}$ & $\textbf{97.80}_\textbf{{(0.58)}}$ & $\textbf{98.60}_\textbf{{(0.29)}}$ &  \\  
\hline
\end{tabular}
\end{minipage}
\end{table*}

\noindent\textbf{Results.}  In Table~\ref{tab:link_pred_main}, we present the results with AUC-ROC as a metric calculated for $5$ independent trials on transductive and inductive settings with random negative sampling (more experiments with different metrics and different sampling criteria are relegated to Appendix~\ref{sec:Additional results}). It can be seen that on LRT benchmarks such as LastFM, MOOC, and Enron, our method consistently outperforms state-of-the-art baselines due to the model’s ability in jointly encoding structural information via graph polynomials that capture multi-hop neighborhood
interactions and temporal evolution through a state-space formulation. Furthermore, \texttt{CTDG-SSM} exhibits only a minor performance drop in inductive setting, highlighting its ability to effectively capture global structural and temporal patterns instead of learning local structural patterns.



\subsection{Dynamic Node Classification}
For dynamic node classification, the goal is to predict the class label of nodes participating in an interaction $\cG(T)$ at time $T$. We evaluate our model on the Wikipedia and Reddit datasets with $2$ classes. We follow the dataset splits and preprocessing strategy outlined in \citet{DYGFORMER}. The model is trained for 200 epochs with early stopping, and memory representations are updated as described in Section~\ref{sec:Architecture}. During testing, we combine the memory states with static embeddings and temporal encodings, which are then passed through an MLP decoder for classification.

In Table~\ref{tab:node_classification}, we report the mean AUC-ROC over 5 runs. The results demonstrate that \texttt{CTDG-SSM} consistently outperforms state-of-the-art approaches, highlighting the effectiveness of jointly capturing LRS and LRT dependencies.


\subsection{Sequence Classification}
In this section, we present results on the sequence classification task, primarily designed to test the model's ability to capture LRS and LRT \citep{CTAN}. The task involves predicting the label of the initial node after traversing a long path, where each new node is connected to the node from the previous event, as illustrated in Fig.~\ref{fig:SeqClass}.
\begin{table*}[t]
\centering
\caption{Per-epoch time (minutes) and GPU memory usage (GB) across multiple datasets.}
\small
\begin{tabular}{lcccccccccccc}
\toprule
\multirow{2}{*}{\textbf{Models}} & \multicolumn{2}{c}{\textbf{LastFM}} & \multicolumn{2}{c}{\textbf{Enron}} & \multicolumn{2}{c}{\textbf{MOOC}} & \multicolumn{2}{c}{\textbf{UCI}} & \multicolumn{2}{c}{\textbf{Reddit}} & \multicolumn{2}{c}{\textbf{Social Evo.}} \\
\cmidrule(lr){2-3} \cmidrule(lr){4-5} \cmidrule(lr){6-7} \cmidrule(lr){8-9} \cmidrule(lr){10-11} \cmidrule(lr){12-13}
 & Time & Mem & Time & Mem & Time & Mem & Time & Mem & Time & Mem & Time & Mem \\
\midrule
JODIE      & 4.4  & 2.28 & 0.07 & 1.30 & 0.78 & 2.36 & 0.03 & 1.44 & 3.95 & 1.10 & 4.70 & 1.71 \\
DyRep      & 6.6  & 2.29 & 0.10 & 1.34 & 0.88 & 2.38 & 0.05 & 1.51 & 5.75 & 1.21 & 7.55 & 1.76 \\
TGAT       & 22.75 & 4.15 & 1.28 & 3.46 & 4.08 & 3.64 & 0.60 & 3.42 & 16.33 & 2.98 & 25.50 & 3.89 \\
TGN        & 12.14 & 2.21 & 0.15 & 1.45 & 1.03 & 2.54 & 0.08 & 1.51 & 2.05 & 1.67 & 3.83 & 1.78 \\
CAWN       & 99.00 & 14.92 & 2.62 & 4.03 & 13.45 & 8.02 & 1.95 & 9.40 & 20.16 & 5.89 & 85.66 & 8.14 \\
TCL        & 6.23 & 3.04 & 0.30 & 2.51 & 1.00 & 2.49 & 0.13 & 2.00 & 2.25 & 1.82 & 5.05 & 2.48 \\
GraphMixer & 16.35 & 2.78 & 1.20 & 2.23 & 4.02 & 2.40 & 0.73 & 2.19 & 4.92 & 1.57 & 15.50 & 2.71 \\
DyGFormer  & 47.00 & 7.57 & 2.73 & 3.23 & 8.32 & 3.35 & 0.62 & 2.30 & 7.00 & 2.42 & 20.00 & 2.77 \\
CTAN       & 3.33 & 1.44 & 0.50 & 1.33 & 3.22 & 2.30 & 0.38 & 1.30 & 0.86 & 1.54 & 2.41 & 0.63 \\
DyGMamba   & 28.45 & 4.17 & 2.05 & 2.74 & 4.88 & 2.48 & 0.60 & 1.93 & 6.30 & 2.07 & 17.80 & 2.59 \\
CTDG-SSM & 4.45 & 1.15 & 0.55 & 0.86 & 1.25 & 0.43 & 0.17 & 0.31 & 1.95 & 1.18 & 9.57 & 5.22 \\
\bottomrule
\end{tabular}
\label{tab:model-time-mem-ctdgssm}
\end{table*}
We generate the data using the procedure in~\citep{CTAN}. 
\begin{figure}
  \centering
  \includegraphics[width=0.3\textwidth]{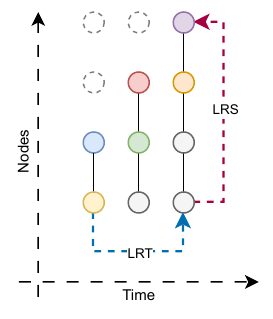}
  \caption{ Illustration of LRT and LRS dependencies with a sequence classification task.}
  \label{fig:SeqClass}
\end{figure}
In this experiment, we evaluate the impact of aggregating one-hop and multi-hop structural information, as well as the significance of our structural change term, by introducing three variants of our method. First, \texttt{CTDG-SSM} (FO) employs a learnable first-order polynomial filter of the form $\mI + \alpha_{1} \mL_\tau$. Second, \texttt{CTDG-SSM} (SO)  utilizes a learnable second-order filter defined as $p(\mL_\tau) = \mI + \alpha_{1} \mL_\tau + \alpha_{2} \mL_\tau^2$. Finally, the topology-unaware SSM (\texttt{TU-SSM}) sets the spatial system matrix to $\bar{\mA}_{\mL[k]} = \mI$, effectively isolating the role of the architecture in learning purely structural patterns.

\begin{figure*}
    \centering
\includegraphics[width=1\linewidth]{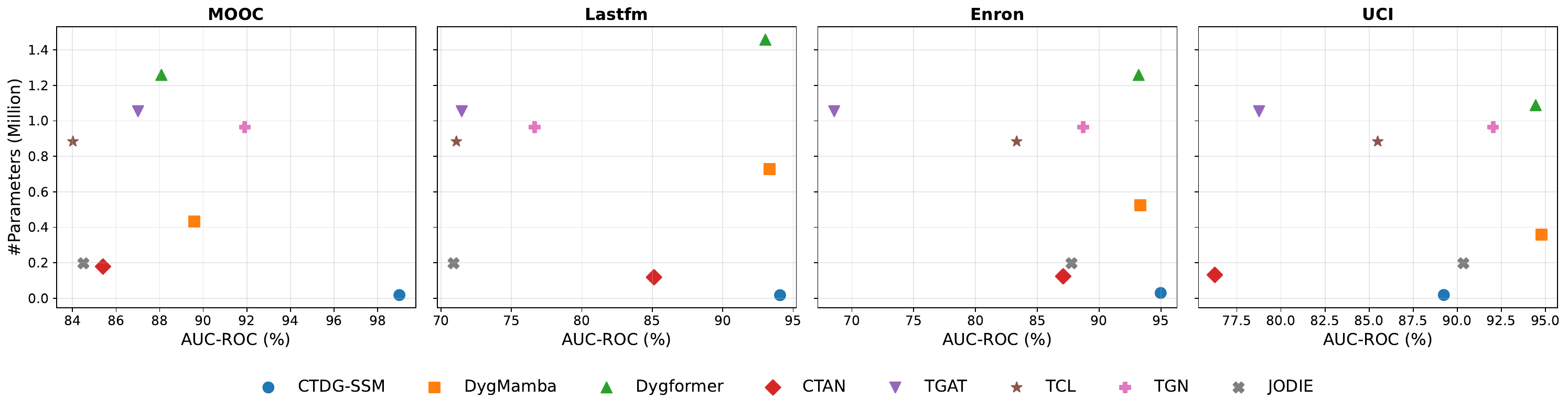}
    \caption{Model size vs. AUC-ROC  under a transductive setting with random negative sampling.}
    \vspace{-5mm}
    \label{fig:param_plot}
\end{figure*}
\noindent\textbf{Results. } 
Table~\ref{tab:mamba_seq_classification} reports results for the sequence classification, where prediction accuracy is defined as the ratio of correctly classified sequences to the total number of sequences. We observe that removing the structural update term in the memory update (\texttt{TU-SSM}) leads to a substantial drop in performance, underscoring the importance of modeling the time-varying graph structure in CTDGs. Further, incorporating higher-order polynomials for multi-hop aggregation in \texttt{CTDG-SSM (SO)} yields clear gains over the single-order variant, which primarily captures local patterns. Finally, the proposed method achieves significant improvements over state-of-the-art baselines, particularly on longer sequences, highlighting its effectiveness in capturing LRS.

\subsection{Parameter Complexity}
We present the comparison among the models based on number of learnable parameters. Recall that the \texttt{CTDG-SSM} layer introduces learnable matrices only through $\bar{\mA}_{\mL_B[k]}$, $\bar{\mA}$ and $\bar{\mB}(\mL[k], \mX[k])$. Figure~\ref{fig:param_plot} presents comparison between among different baseline models using parameter count and AUC-ROC. It can be observed that on long-range datasets such as MOOC and Enron, the proposed model achieves superior performance while being highly parameter-efficient, requiring about one-tenth fewer parameters compared to existing approaches.

\subsection{Runtime Analysis}\label{Sec:runtimeanalysis}
In Table~\ref{tab:model-time-mem-ctdgssm} we report the per-epoch training time (in minutes) and GPU memory consumption (in GB) across multiple datasets. Notably, it can be observed that \texttt{CTDG-SSM} achieves significantly lower per-epoch training time and memory usage compared to \texttt{DyGMamba} and \texttt{DyGFormer}, both of which are specifically designed for long-range propagation tasks. Additionally we give detailed analysis on total runtime and loss convergence behavior in the supplementary material (see \ref{sec:run_timeanalysis1}).   

\section{Conclusions}
We proposed \texttt{CTDG-SSM} a principled approach that formulates a SSM for CTDGs. In particular, we introduce 
\texttt{CTT-HiPPO} that yields topology-aware memory representations by projecting HiPPO coefficients through a polynomial of graph Laplacian. We leverage \texttt{CTT-HiPPO} to derive SSM for CTDGs where the memory representations are governed by the evolving topology.
We further established theoretical guarantees on the robustness and permutation equivariance of \texttt{CTDG-SSM}. Extensive experiments on diverse temporal graph learning tasks that includes link prediction, node classification, and sequence classification tasks shows significant gain in performance by \texttt{CTDG-SSM} vis-a-vis current models, making it a new state-of-the art method for continuous-time dynamic graphs.

\section*{Impact Statement}
Continuous-time dynamic graphs provide a framework for modeling evolving relational data and are broadly relevant across scientific domains with potential industrial applications. The techniques developed in this work contribute to advances in graph machine learning. While such methods may have societal implications depending on their application, we do not anticipate any direct negative impacts arising from this work that require specific discussion.

\bibliography{refs.bib}
\bibliographystyle{icml2026}

\newpage
\appendix
\onecolumn
\icmltitle{Supplementary Material: Learning Long Range Spatio-Temporal Representations over Continuous Time Dynamic Graphs with State Space Models}

{\section{State-Space Models}

State-space models (SSMs) are widely used for sequence modeling due to their ability to capture long-range dependencies through latent state evolution while remaining computationally efficient compared to Transformers~\citep{gu_efficiently_2022}. 
For an input signal $\vx(t)$, an SSM evolves latent states $\vh(t) \in \mathbb{R}^d$ according to a linear ordinary differential equation (ODE), producing output $\vy(t)$ as~\citep{hippo,gu_efficiently_2022,smith_simplified_2023}:
\begin{equation}
    \frac{d\vh(t)}{dt} = \mA(t)\vh(t) + \mB(t)\vx(t) \quad \text{and} \quad 
 \vy(t) = \mC(t)\vh(t) + \mD(t)\vx(t),
\label{eq:SSM}
\end{equation} 
where $\mA(t)$, $\mB(t)$, $\mC(t)$, and $\mD(t)$ are the state transition, input, output, and feedforward matrices, respectively.
In practice, the continuous-time model is discretized using zero-order hold (ZOH), which assumes piecewise constant inputs: $\vx(t) = \vx[k]$ for $t \in [t_k, t_{k+1})$ with step size $\Delta[k] = t_{k+1}-t_{k}$. This yields the discrete-time system:
\begin{equation}
\vh[k+1] = \bar{\mA}\vh[k] + \bar{\mB}\vx[k]
\quad \text{and} \quad 
\vy[k] = \bar{\mC}\vh[k] + \bar{\mD}\vx[k],
\label{eq:zoh}
\end{equation} 
where $\bar{\mA} = e^{\Delta[k] \mA}$, $\bar{\mB} = \int_{0}^{\Delta[k]} e^{\mA\tau} \mB \, d\tau$, $\bar{\mC} = \mC$, and $\bar{\mD} = \mD$.

Research in SSM design has significantly enhanced its effectiveness for sequence modeling. HiPPO~\citep{hippo} introduced principled initialization strategies for capturing long-range dependencies. S4~\citep{gu_efficiently_2022} extended these with structured parameterizations of the system matrix $\mA$ (e.g., diagonal plus low-rank decompositions) to enable efficient computation. More recent work like Mamba~\citep{gu_mamba_2024} has further improved scalability and selectivity by making the input and output matrices $\mB$ and $\mC$ input-dependent and the system matrix diagonal, enhancing model expressivity while maintaining computational efficiency.

\section{Derivation for \texttt{CTT-HiPPO} Coefficients} \label{sec:CTT-Coefficients}

We derive the solution for the representations/coefficients  ($\mH_{\tau}^{(i)}$) for \texttt{CTT-HiPPO}. Recall, to obtain the coefficients we minimize the residual in the observed time interval. To begin with the $\vr_{i}(t)$ can be equivalently expressed as 

\begin{equation}
\begin{aligned}
   \int_0^\tau \|\vr_{i}(t)\|_2^2 \, d\mu(t) &= \int_0^\tau \left\| \mX_{:,i}(t) - p(\mL_\tau) \mH^{(i)}_{\tau} \vg(t) \right\|_2^2 d\mu(t), \\
    &= \int_0^\tau \operatorname{Tr} \left[ \left( \mX_{:,i}(t) - p(\mL_\tau) \mH_{\tau}^{(i)} \vg(t) \right)
    \left( \mX_{:,i}(t) - p(\mL_\tau) \mH_{\tau}^{(i)} \vg(t) \right)^\top \right] d\mu(t).
    \end{aligned}
\end{equation}
    
Using the first optimality condition i.e., $\frac{\partial  \|\vr_i(\tau)\|_2^2}{\partial \mH^{(i)}_{\tau} }  = 0$ and on simplifying we have  
\begin{align}
    \frac{\partial}{\partial \mH_{\tau}^{(i)}} 
    \int_0^\tau \operatorname{Tr} \left[ \left( \mX_{:,i}(t) - p(\mL_\tau) \mH_{\tau}^{(i)} \vg(t) \right)
    \left( \mX_{:,i}(t) - p(\mL_\tau) \mH_{\tau}^{(i)} \vg(t) \right)^\top \right] d\mu(t) &= 0, \nonumber \\
    \int_0^\tau \frac{\partial}{\partial \mH_{\tau}^{(i)}} \left[ 
        \operatorname{Tr} \left( p(\mL_\tau) \mH_{\tau}^{(i)} \vg(t) \vg(t)^\top {\mH_{\tau}^{(i)}}^\top p(\mL_\tau)^\top \right)
        - 2 \operatorname{Tr} \left( p(\mL_\tau) \mH_{\tau}^{(i)} \vg(t) \mX_{:,i}(t)^\top \right)
    \right] d\mu(t) &= 0, \nonumber \\
    \int_0^\tau 2p(\mL_\tau)^\top p(\mL_\tau) \mH_{\tau}^{(i)} \vg(t)\vg(t)^\top 
    - 2p(\mL_\tau)^\top \mX_{:,i}(t) \vg(t)^\top \; d\mu(t) &= 0, \nonumber \\
    \int_0^\tau p(\mL_\tau)^\top p(\mL_\tau) \mH_{\tau}^{(i)} \vg(t)\vg(t)^\top d\mu(t) 
    = \int_0^\tau p(\mL_\tau)^\top \mX_{:,i}(t) \vg(t)^\top d\mu(t). 
    \label{eq:ctt-hippo1}
\end{align}

For fixed set of orthogonal polynomials we have  $\int_0^\tau \vg(t) \vg(t)^\top d\mu(t) = \mI$, then \eqref{eq:ctt-hippo1} can be simplified as 
\begin{align}
    p(\mL_\tau) \mH^{(i)}_{\tau} &= \int_0^\tau \mX_{:,i}(t) \vg(t)^\top d\mu(t), \nonumber\\
      \mH_{\tau}^{(i)} &= p(\mL_\tau)^{-1} \int_0^\tau \mX_{:,i}(t) \vg(t)^\top d\mu(t), \nonumber \\
    \mH_{\tau}^{(i)} &= p(\mL_\tau)^{-1}  \mH^{i,\rm{HiPPO}}_{\tau}. \label{eq:connection}
\end{align}
where $\mH_{\tau}^{(i)}$ corresponds to the solution from \texttt{CTT-HiPPO}, where it is obtained by projecting the HiPPO solution through graph-aware polynomial.

\section{Proof of Theorems}
This section presents detailed proofs of the theorems from the main text.

\subsection{Proof for Theorem \ref{theorem:4.1}} \label{sec:proof_4.1}
\begin{proof}
We derive the SSM for CTDGs (\texttt{CTDG-SSM}) that governs the evolution of memory representations. To begin with,  we consider the relation between 
 structural HiPPO coefficients $\mH_{\tau}$ and  HiPPO coefficients given in  \eqref{eqn:graph_coff_hippo_coff_relation} and obtain the evolution of memory states as follows:  
 For an event observed at $\tau_+$ and corresponding CTDG $\cG_{\tau_+}$, we define the polynomial as in the interval $[\tau,\tau_+)$ as  $p(L_{s}) = \frac{\tau-s}{\tau-\tau_+} p(\mL_{\tau_{+}}) + \frac{\tau_+-s}{\tau_+-\tau}p(\mL_{\tau})$ for $s\in[\tau,\tau_+)$. Here $\mL_{\tau_+} \in \mathbb{R}^{N_{\tau_+} \times N_{\tau_+}}$ and $\mL_\tau$ is calculated by removing the newly observed edges from $\mL_{\tau_+}$. 
 Then the derivative of the coefficients for $s\in[\tau,\tau_+)$ is given as :
\begin{align}
 \frac{d}{ds} \left(p(\mL_{s})\mH_{s}\right)  &= \frac{d\mH_{s}^{(HiPPO)}} {ds} , \nonumber \\
    \frac{dp(\mL_s)}{ds} \mH_{s} + p(\mL_{s}) \frac{d\mH_{s}}{ds} & = \frac{d\mH_{s}^{(HiPPO)}}{ds}.
    \label{eq:SSM_CTDG1}
\end{align}
This can be equivalently expressed on  multiplying with $p(\mL_s)^{-1}$ as 
\begin{align}
    \nonumber&\frac{d\mH_{s}}{ds} = p(\mL_{s})^{-1}\frac{d\mH_{s}^{(HiPPO)}}{ds} - p(\mL_{s})^{-1}\frac{dp(\mL_s)}{d\tau} \mH_{s}
\end{align}
To obtain an equivalent SSM for CTDGs, we leverage the established result  from~\citep{hippo} that relates the evolution of HiPPO coefficients to a linear ordinary equation  as  
\begin{align}
    &\frac{d\mH_{s}^{(HiPPO)}}{ds} =  -\mH_{s}^{(HiPPO)}\frac{\mA^\top}{M(s)} + \vx(s)\frac{\mB^\top}{M(s)},
    \nonumber
\end{align}
where  $\mA \in \mathbb{R}^{d\times d}$ is a state transition matrix, $\mB \in \mathbb{R}^{d\times 1 }$ input matrix and $M(\tau):\mathbb{R}^+\to\mathbb{R}^+$ is a scalar that depends on the choice of bases polynomial and weigh function $\mu(t)$. The continuous SSM for CTDGs for $s\in [\tau,\tau_+)$ is given by
\begin{align}
\label{eqn:ctt-update-apndx}
      \frac{d\mH_{s}}{ds} &= 
      - p(\mL_{s})^{-1}\mH_{s}^{(HiPPO)}\frac{\mA^\top}{M(s)}  - p(\mL_{s})^{-1}\frac{dp(\mL_s)}{ds} \mH_{s} + p(\mL_{s})^{-1}\mX(s)\frac{\mB^\top}{M(s)},  \nonumber \\
      \frac{d\mH_{s}}{ds} &= 
      - \mH_{s}\frac{\mA^\top}{M(s)}  - p(\mL_{s})^{-1}\frac{dp(\mL_s)}{ds} \mH_{s} + p(\mL_{s})^{-1}\vx(s)\frac{\mB^\top}{M(s)}.
\end{align} 
\end{proof}

We can further simplify \eqref{eqn:ctt-update-apndx} to express it in a standard first-order state-space model. To do so, we apply vectorization operation on  \eqref{eqn:ctt-update-apndx} and use the identity $\text{vec}(\mA \mB \mC) = (\mC^\top \otimes \mA) \text{vec}(\mB)$, where $\otimes$ is a Kronecker product. Then we obtain
\begin{align}
    \nonumber&\frac{d\vh_{s}}{ds} = -\left(\frac{\mA}{M(s)} \oplus \left(p(\mL_s)^{-1}\frac{dp(\mL_s)}{ds}\right)\right)\vh_{s} + \frac{\mB}{M(s)}\otimes  p(\mL_s)^{-1} (\vx(s)) \\
   &\frac{d\vh_{s}}{ds} = \mA_{g}(s) \vh_{s} + \mB_{g}(s) \mX(s), \quad s\in[\tau,\tau_+)
   \label{eq:vec_ssm}
\end{align}
where $\vh_{\tau}= \text{vec}(\mH_{\tau}) \in \mathbb{R}^{N_{T}d \times 1}$,  $\mA_g(\tau)$ and $\mB_g(\tau)$ denote the time-dependent system and input matrices, respectively. Here $\oplus$ denotes the Kronecker sum. The evolution of memory coefficients of nodes so far characterizes the continuous time-variant SSM that jointly encodes dynamic graphs' structural and temporal information.
}

\begin{corollary}[Reduction to Classical HiPPO]
Let the graph Laplacian be static ($\mL_\tau = \mL$) and let the filter satisfy $p(\mL_\tau) = \mI$. Then, the \texttt{CTDG-SSM} dynamics~\eqref{eq:structural-ssm} reduce to the classical HiPPO state-space dynamics:
\[
\frac{d\mH_{\tau}}{d\tau} 
= - \mH_{\tau} \frac{\mA^\top}{M(\tau)} + \mX(\tau) \frac{\mB^\top}{M(\tau)} .
\]
This shows that \texttt{CTDG-SSM} is a strict generalization of classical HiPPO: it recovers standard memory evolution when the graph is static or the filter is the identity, while naturally incorporating dynamic graph information when $p(\mL_\tau)$ varies over time.
\end{corollary}

\begin{Remark}
    
The expression,

$$
\min _{\mH _{\tau}^{(i)}} \int _0 ^\tau \left\| \mX_{:,i} (t) - p(\mL _\tau) \mH _{\tau}^{(i)} \vg(t) \right\| _2 ^2\, d \mu (t),
$$

is the CTDG-SSM formulation, where the polynomial operator $p(\mL_\tau)$ specifies how the graph structure influences the reconstruction.

When $p(\mL_\tau)=\mathbf{I}$, the Laplacian dependence vanishes, yielding the classical HiPPO ~\citep{hippo} objective:

$$
\min_ {\mH _{\tau}^{(i)}} \int _0 ^\tau  \left\| \mX_{:,i} (t) - \mH _{\tau}^{(i)} \vg (t) \right\| _2 ^2 d \mu (t),
$$

which matches the standard HiPPO setting.

However, when a quadratic Laplacian regularizer is introduced to enforce smoothness over the reconstructed signal, we obtain the GraphSSM~\citep{First_principles} objective:

$$
\min_ {\mH _{\tau}^{(i)}} \int _0 ^\tau  \left\| \mX_{:,i} (t) - \mH _{\tau}^{(i)} \vg (t) \right\| _2 ^2\, d \mu (t) + \int _0 ^\tau \big(\mH _{\tau}^{(i)}\vg (t) \big) ^\top  \mL _t \big( \mH _{\tau}^{(i)} \vg (t) \big) d \mu (t).
$$

Note: Using \(\mL_t\) instead of $\mL_\tau$ leads to an intractable quadratic loss, motivating an approximate piecewise-constant graph state update.
\end{Remark}

\subsection{Proof for Theorem \ref{them:4.3}}
\label{appendix:them_4.2}
\begin{proof}
We present the equivalent discrete-time SSM for CTDG using ZOH technique. Recall from \eqref{eq:vec_ssm} the continuous-time evolution for CTDGs is given as
\begin{align}
    \frac{d\vh[k]}{dt} &= \mA_g(k)\vh[k] + \mB_{g}(k)\mX[k]. \label{eq:ssm-cont}
\end{align}
Following \eqref{eq:zoh}, the equivalent discrete update is given as 
\begin{align}
    \vh[k+1] 
    &= \exp\!\big(\mA_g(t_k)\Delta[k]\big)\vh[k] 
    + \int_0^{\Delta[k]} \exp\!\big(\mA_g(k)s\big)\mB_g(t_k)\mX[k] \, ds,
    \label{eq:zoh-general}
\end{align}
where $\Delta[k]$ is time interval. Recall
    $\mA_g(t_k) = -\mA \;\oplus\; \Big(-\,p(\mL[k])^{-1}\!\frac{p(\mL[k])-p(\mL[k-1])}{\Delta[k]}\Big), 
    \mB_g(t_k) = \mB \otimes p(\mL[k])^{-1}$  
Although one can directly apply~\eqref{eq:zoh} as discussed in the preliminaries to obtain a discrete equivalent for \eqref{eq:vec_ssm}, this approach incurs significant computational overhead in implementation,  since the Kronecker-structured matrices \( \mA_g \) and \( \mB_g \) involved in \eqref{eq:vec_ssm}  are of large dimensions \( (N_{\tau} d \times N_{\tau} d) \) and \( (N_{\tau} d \times N_{\tau}) \). To alleviate this complexity, we exploit algebraic properties of the Kronecker product to derive an equivalent update rule as  
\begin{align}
    \vh[k+1] 
    &= \Big(e^{-\mA\Delta[k]} \otimes e^{-\,p(\mL[k])^{-1}\tfrac{p(\mL[k])-p(\mL[k-1])}{\Delta[k]}\,\Delta[k]}\Big)\vh[k],  \nonumber\\
    &\quad + \int_0^1 
       \Big(e^{-\mA\Delta[k] s} \otimes e^{-\,p(\mL[k])^{-1}\tfrac{p(\mL[k])-p(\mL[k-1])}{\Delta[k]}\,\Delta[k] s}\Big)(\mB \otimes p(\mL[k])^{-1})\mX[k] \,\Delta[k] ds, 
    \label{eq:vector-update}
\end{align}
where \eqref{eq:vector-update} follows by using the following identities $
    e^{\mA \oplus \mB} = e^{\mA}\otimes e^{\mB}, 
    \mathrm{vec}(\mA \mB \mH) = (\mH^{\top}\otimes \mA)\,\mathrm{vec}(\mB),
    (\mA\otimes \mB)(\mH\otimes \mD) = (\mA\mH)\otimes (\mB\mD)$.
This can be equivalently expressed as 
\begin{align}
    \mH[k+1] 
    &= e^{-\,p(\mL[k])^{-1}\tfrac{p(\mL[k])-p(\mL[k-1])}{\Delta[k]}\,\Delta[k]}\;\mH[k]\; e^{-\mA^\top \Delta[k]}, \nonumber \\
    &\quad + \int_0^{1} 
       e^{-\,p(\mL[k])^{-1}\tfrac{p(\mL[k])-p(\mL[k-1])}{\Delta[k]}\,\Delta[k] s}\;
       p(\mL[k])^{-1}\mX_{:,i}[k]\mB^\top \; e^{-\mA^\top \Delta[k] s}\;\Delta[k]\,ds. \\
       \mH[k+1] &= \Bar{\mA}_{\mL[k]} \, \mH[k] \, \Bar{\mA} + \bar\mB(\mL[k], \mX[k]), \label{eq:discrete-update_1}
\end{align}


where $\Bar{\mA}_{\mL[k]} = \exp (- p(\mL[k])^{-1} \frac{( p(\mL[k]) - p(\mL[k-1])}{\Delta[k]} \Delta[k])$, $\Bar{\mA} = \exp( -\Delta[k]\mA^\top) $, and $\bar\mB(\mL[k], \mX[k]) = \int_0^1 (\Bar{\mA}_{\mL[k]})^s \, p(\mL[k])^{-1} \mX_{:,i}[k] \mB^\top (\Bar{\mA})^s \, \Delta[k] \, ds $.
\end{proof}

\subsection{Proof of Theorem \ref{thm:pertrurbation}} \label{sec:thm_pert}

 Consider the $\bar \mH_{\tau}$ and  $\mH_{\tau}$ as the memory representations obtained with the perturbed graph Laplacian and true Laplacian. For brevity, we call the solution from HiPPO as $\mH_{H}.$ Then the error between the representations is given as  
\begin{equation}
\begin{aligned}
    \|\bar \mH_{\tau} - \mH_{\tau}\|_{2} =& \, \| p(\bar \mL_\tau)^{-1}\mH_{H} - p(\mL_\tau)\mH_{H}\|_{2} \nonumber \\
     \overset{(a)}\le & \| p(\bar \mL_\tau)^{-1} - p(\mL_\tau)^{-1}\|_{2}\|\mH_{H}\|_{2}\nonumber \\
    \overset{(b)}\le & \|p(\bar \mL_\tau)^{-1}( p( \mL_\tau) - p(\bar \mL_\tau))p(\mL_\tau)^{-1}\|_2\|\mH_{H}\|_{2} \nonumber\\
     \overset{(c)}\le& \|p(\bar \mL_\tau)^{-1}\|_{2}\| p( \mL_\tau) - p(\bar \mL_\tau)\|_{2}\|p(\mL_\tau)^{-1}\|_2\|\mH_{H}\|_{2},
    \label{eq:robust1}
    \end{aligned}
    \end{equation}
   where \eqref{eq:robust1}(a), (b), (c) follow from the norm inequalities. Recall $\mL$ is a normalized Laplacian, therefore the spectrum is bounded in the range $\lambda \in [0,2]$.  Let us call   $\lambda_1:= \min_{\lambda\in [0,2] } |p(\lambda)| >0$, $\lambda_2:= \max_{\lambda \in [0,2] }|p(\lambda)|$, and $\lambda_c:= \max_{\lambda\in[0,2]} |p(y)'|$. Then we have
    \begin{align}
    \|\bar \mH_{\tau} - \mH_{\tau}\|_2 &\le \tfrac{1}{\lambda_1^2}\| p( \mL_\tau) - p(\bar \mL_\tau)\|_{2}\|\mH_{i,H}\|_{2}, \nonumber \\
     &\le \tfrac{\lambda_c}{\lambda_1^2}\|  \mL_\tau - \bar \mL_\tau\|_{2}\|\mH_{i,H}\|_{2}, \nonumber\\
     &\le \tfrac{\lambda_c}{\lambda_1^2}\|_{2}  \mL_\tau - \bar \mL_\tau\|_2\| p(\mL_\tau)\mH_{\tau}\|_{2},\\&\le \tfrac{\lambda_c}{\lambda_1^2}\|  \mL_\tau - \bar \mL_\tau\|_{2}\| p(\mL_\tau)\|_{2}\|\mH_{\tau}\|_{2}. \nonumber
     \end{align}
 The normalized error given by     
\begin{align}
     \frac{\|\bar \mH_{\tau} - \mH_{\tau}\|_{2}}{\|\mH_{\tau}\|_{2}}\le& \tfrac{\lambda_c}{\lambda_1^2}\|  \mL_\tau - \bar \mL_\tau\|\|_{2} p(\mL_\tau)\|_{2}, \nonumber \\\le& \tfrac{\lambda_2\lambda_c}{\lambda_1^2}\|  \mL_\tau - \bar \mL_\tau\|_{2}, \nonumber \\ \overset{(a)}\le& \epsilon \Gamma,
     \label{eq:robust2}
\end{align}
where \eqref{eq:robust2}(a) since energy of perturbation is bounded i.e.,  $\|\Delta\mL\|_2 \leq \epsilon$  and $\Gamma=\tfrac{\lambda_2\lambda_c}{\lambda_1^2}$.

\subsection{Proof of Theorem \ref{thm:permutation_equivariance}} \label{sec:proof_thm4.2}
To prove that the representations from  \texttt{CTDG-SSM} as permutation equivariant we first show that representations from \texttt{CTT-HiPPO} are equivariant to permutation. 
Under the permutation the features signal and Laplacian modifies as  $\hat\mX = \bm \Pi \mX$, $\hat \mL = \bm \Pi \mL \bm \Pi^\top$. Let   $\hat{\mH}_{\tau}$  be representations obtained under permutation, then we have  
\begin{align}
\hat{\mH}_{\tau} &= p(\hat{\mL}_\tau)^{-1}  \int_0^\tau \hat\mX(t)\vg(t)^\top dw(t),\nonumber\\
 &= p(\bm \Pi \mL_\tau \bm \Pi^\top)^{-1}  \int_0^\tau \bm \Pi\mX(t)\vg(t)^\top dw(t), \nonumber\\ 
 &=\bm{\Pi} p({\mL}_\tau)^{-1} \bm \Pi^\top  \int_0^\tau \bm \Pi \mX(t)\vg(t)^\top dw(t),\nonumber\\
 &=\bm \Pi p({\mL}_\tau)^{-1} \int_0^\tau\mX(t)\vg(t)^\top dw(t)\nonumber\\
 &=\bm \Pi \mH_\tau, 
 \label{eq:perm_1}
\end{align}
\eqref{eq:perm_1} implies the representations obtained from \texttt{CTT-HiPPO} are permutation equivariant. Now, to prove the equivariance for the representations from \texttt{CTDG-SSM} layer we first evaluate state matrix $\bar{\mA}$ and system matrix $\bar{\mB}$ under permutation as
\begin{align}
    \bar \mA_{\hat\mL[k]} ^s &= \exp(-p(\hat\mL[k])^{-1}(p(\hat\mL[k-1])-p(\hat\mL[k])s), \nonumber \\
    &= \exp(-\bm \Pi p(\mL[k])^{-1}\bm \Pi ^\top \bm \Pi (p(\mL[k])-p(\mL[k-1])\bm \Pi ^\top s ), \nonumber \\
    &= \exp(-\bm \Pi p(\mL[k])^{-1}(p(\mL[k])-p(\mL[k-1])\bm \Pi ^\top s), \nonumber \\
    &= \bm \Pi\exp(- p(\mL[k])^{-1}(p(\mL[k])-p(\mL[k-1]) s)\bm \Pi ^\top, \nonumber\\
 &= \bm \Pi\bar \mA_{\mL[k]}^s\bm \Pi ^\top,
 \label{eq:per_2}
\end{align}
where $\bar{\mB}$ modifies as 
\begin{align}
    \bar{\mB}(\hat\mL[k],\hat\mX[k])&=\int_0^1 \Bar{\mA}_{\hat\mL[k]}^s \, p(\hat\mL[k])^{-1} \hat\mX[k] \mB^\top \Bar{\mA}^s \, \Delta[k] \, ds, \nonumber \\ &=  \int_0^1 \bm \Pi \Bar{\mA}_{\mL[k]}^s \bm \Pi^\top\, \bm \Pi \nonumber p(\mL[k])^{-1} \bm \Pi^\top \bm \Pi \mX[k] \mB^\top \Bar{\mA}^s \, \Delta[k] \, ds, \nonumber\\
    & =  \bm \Pi \int_0^1  \Bar{\mA}_{\mL[k]}^s \,p(\mL[k])^{-1} \mX[k] \mB^\top \Bar{\mA}^s \, \Delta[k] \, ds. \nonumber\\
    &= \bm \Pi \bar{\mB} (\mL[k],{\mX[k]})
    \label{eq:per_3}
\end{align}
Now we show that the updates from \texttt{CTDG-SSM} are permutation equivariant.  Consider 
\begin{align} \hat \mH[k+1] =& \bar\mA_{\hat \mL[k]} \hat\mH[k]\bar\mA + \bar\mB(\hat\mL[k],\hat\mX[k]), \nonumber \\
\overset{(a)} =& \bar\mA_{\hat \mL[k]} (\bm \Pi \mH[k]) \bar\mA + \bm \Pi \bar\mB(\mL[k],\mX[k]), \nonumber \\
\overset{(b)} =& \bm \Pi \bar\mA_{\mL[k]} \mH[k]\bar\mA + \bm \Pi \bar\mB(\mL[k],\mX[k]), \nonumber \\ 
=& \bm \Pi \left(\bar\mA_{\mL[k]} \mH[k]\bar\mA + \bar\mB(\mL[k],\mX[k])\right),\nonumber \\
=& \bm \Pi \mH[k+1], 
\label{eq:per_4}
\end{align}
where \eqref{eq:per_4}(a) follows by recursion. Recall $k=0$  we have $\hat{\mH}[1] = \bar\mB(\hat\mL[0],\hat\mX[0])$ as $\mH[0] = \mathbf{0}$, hence $\hat{\mH}[1] = \bm \Pi \bar\mB(\mL[0],\hat\mX[0]) = \bm \Pi \mH[1]$ from \eqref{eq:per_3} which is propogated through $k$ layers. Then  \eqref{eq:per_4}(b) follows from \eqref{eq:per_2} and \eqref{eq:per_3}.

\section{Numerical Experiments and Additional Results}
In this section, we discuss the dataset details, hyperparameters, and the additional results on the dynamic link prediction task.
\subsection{Dataset Details} \label{sec:data_Sec}
We provide a detailed description of the datasets considered for experimentation in Table~\ref{tab:dataset_stats}. In all the datasets,  LastFM, Enron and MOOC are mainly considered for evaluating the LRT task. In particular, 
 The LastFM dataset corresponds to data from a music streaming platform that records user listening behaviors, where users and songs are nodes and links denote listening events~\citep{lastFM}. 
The Enron dataset is an email communication dataset among employees of the Enron Corporation, recorded over a three-year period~\citep {Enron}. Whereas the MOOC dataset captures student interactions on an online course platform, where links represent students accessing course content such as videos or problem sets~\citep{MOOC}, other DTDG datasets used for evaluation include Flights, Can. Parl, US Legis., UN Trade, UN Vote, and Contact include. For all datasets used in data processing, we employ the same pipeline described in~\citep{DYGFORMER}. Additionally, datasets including \texttt{tgbl-wiki} and \texttt{tgbl-coin} from~\citep{tgbl} were also utilized. 

\begin{table}[t]
\centering
\caption{Statistics of the datasets used in our experiments. \#N \& L feat corresponds to the dimension of node and link features, where -  represents the unavailability of node features.}
\label{tab:dataset_stats}
\resizebox{\textwidth}{!}{%
\begin{tabular}{lcccccccc}
\toprule
Dataset & Domain & \#Nodes & \#Links & \#N\&L Feat & Bipartite & Duration & Unique Steps & Time Granularity \\
\midrule
Wikipedia   & Social      & 9,227  & 157,474   & – \& 172 & True  & 1 month    & 152,757   & Unix timestamps \\
Reddit      & Social      & 10,984 & 672,447   & – \& 172 & True  & 1 month    & 669,065   & Unix timestamps \\
MOOC        & Interaction & 7,144  & 411,749   & – \& 4   & True  & 17 months  & 345,600   & Unix timestamps \\
LastFM      & Interaction & 1,980  & 1,293,103 &– \& –   & True  & 1 month    & 1,283,614 & Unix timestamps \\
Enron       & Social      & 184    & 125,235   & – \& –   & False & 3 years    & 22,632    & Unix timestamps \\
UCI         & Social      & 1,899  & 59,835    & – \& –   & False & 196 days   & 58,911    & Unix timestamps \\
Social Evo. & Proximity   & 74     & 2,099,519 & – \& 2   & False & 8 months   & 565,932   & Unix timestamps \\
Flights     & Transport   & 13{,}169 & 1{,}927{,}145 & - \& 1 & False & 4 months    & 122      & days \\
Can. Parl.  & Politics    & 734      & 74{,}478    & - \& 1   & False & 14 years    & 14       & years \\
US Legis.   & Politics    & 225      & 60{,}396    & - \& 1   & False & 12 congresses & 12     & congresses \\
UN Trade    & Economics   & 255      & 507{,}497   &  - \& 1   & False & 32 years    & 32       & years \\
UN Vote     & Politics    & 201      & 1{,}035{,}742 &  - \& 1 & False & 72 years    & 72       & years \\
Contact     & Proximity   & 692      & 2{,}426{,}279 & - \& 1 & False & 1 month     & 8{,}064  & 5 minutes \\
tgbl-wiki     & Interaction   & 9,227     & 157,474 & - \& 1 & True & 1 month     & 152,757  & Unix timestamps \\
tgbl-coin     & Economics   & 638,486    & 22,809,486 & - \& 1 & False & 7 month     & 1,295,720  & Unix timestamps \\

\bottomrule
\end{tabular}%
}
\end{table}
\begin{table*}[ht]
\centering
\caption{AUC-ROC for transductive dynamic link prediction under. RNS: Random Negative Sampling, HNS: Historical Negative Sampling, INS : Inductive Negative Sampling.}
\label{tab:trans_auc_roc}
\resizebox{\textwidth}{!}{
\begin{tabular}{llccccccccccc}
\toprule
\textbf{NSS} & \textbf{Datasets} & JODIE & DyRep & TGAT & TGN & CAWN & TCL & GraphMixer & DyGFormer & CTAN & DyGmamba & \texttt{CTDG-SSM} \\
\midrule
\multirow{7}{*}{RNS} 
& LastFM & 70.89  $\pm$  1.97 & 71.40  $\pm$  2.12 & 71.47  $\pm$  0.14 & 76.64  $\pm$  4.66 & 85.92  $\pm$  0.16 & 71.09  $\pm$  1.48 & 73.51  $\pm$  0.14 & 93.03  $\pm$  0.11 & 85.12  $\pm$  0.77 & 93.31  $\pm$  0.18 & \textbf{93.79 $\pm$  0.22} \\
& Enron &87.77 $\pm$ 2.43 & 83.09 $\pm$ 2.20 & 68.57 $\pm$ 1.46 & 88.72 $\pm$ 0.95 & 90.34 $\pm$ 0.23  & 83.33 $\pm$ 0.93 & 84.16 $\pm$ 0.34 & 93.20 $\pm$ 0.12 & 87.09 $\pm$ 1.51 & 93.34 $\pm$ 0.23 & \textbf{94.98  $\pm$  2.92} \\
& MOOC & 84.50 $\pm$ 0.60 & 84.50 $\pm$ 0.87 & 87.01 $\pm$ 0.16 & 91.91 $\pm$ 0.82 & 80.48 $\pm$ 0.41  & 84.02 $\pm$ 0.59 & 84.04 $\pm$ 0.12 & 88.08 $\pm$ 0.50 & 85.40 $\pm$ 2.67 & 89.58 $\pm$ 0.12  & \textbf{99.00  $\pm$  0.33} \\
& Reddit & 98.29 $\pm$ 0.05 & 98.13 $\pm$ 0.04 & 98.50 $\pm$ 0.01 & 98.61 $\pm$ 0.05 & 99.02 $\pm$ 0.00  & 97.67 $\pm$ 0.01 & 97.17 $\pm$ 0.02 & 99.15 $\pm$ 0.01 & 97.24 $\pm$ 0.75 & 99.27 $\pm$ 0.01 & \textbf{99.48  $\pm$  0.02}\\
& Wikipedia & 96.36 $\pm$ 0.14 & 94.43 $\pm$ 0.32 & 96.60 $\pm$ 0.07 & 98.37 $\pm$ 0.10 & 98.54 $\pm$ 0.01  & 97.27 $\pm$ 0.06 & 96.89 $\pm$ 0.04 & 98.92 $\pm$ 0.03 & 97.00 $\pm$ 0.21 & 99.08 $\pm$ 0.02 & \textbf{99.33 $\pm$ 0.08} \\
& UCI & 90.35 $\pm$ 0.51 & 69.46 $\pm$ 2.66 & 78.76 $\pm$ 1.10 & 92.03 $\pm$ 0.69 & 93.81 $\pm$ 0.23  & 85.49 $\pm$ 0.82 & 91.62 $\pm$ 0.52 & 94.45 $\pm$ 0.22 & 76.25 $\pm$ 2.83 & \textbf{94.77 $\pm$ 0.18}  & 89.24 $\pm$ 0.43\\
& Social Evo. &92.13 $\pm$ 0.20 & 90.37 $\pm$ 0.52 & 94.93 $\pm$ 0.06 & 95.31 $\pm$ 0.27 & 87.34 $\pm$ 0.10 & 95.45 $\pm$ 0.21 & 95.21 $\pm$ 0.07 & 96.25 $\pm$ 0.04 & Timeout & 96.38 $\pm$ 0.02 &  \textbf{99.10 $\pm$ 0.49} \\
\midrule
& \textbf{Avg. Rank} & 7.93 & 9.36 & 7.86 & 4.57 & 5.71 & 8.00 & 7.71 & 3.00 & 7.50 & 2.00 & \textbf{1.86} \\
\midrule
\multirow{7}{*}{HNS} 
& LastFM  &75.65 $\pm$ 4.43 & 70.63 $\pm$ 2.56 & 64.23 $\pm$ 0.45 & 78.00 $\pm$ 2.97 & 67.92 $\pm$ 0.32 & 60.53 $\pm$ 2.54 & 64.06 $\pm$ 0.34 & 78.80 $\pm$ 0.02 & 79.50 $\pm$ 0.82 & 79.82 $\pm$ 0.27 & \textbf{89.55 $\pm$  0.57} \\
& Enron & 75.21 $\pm$ 1.27 & 76.36 $\pm$ 1.42 & 62.36 $\pm$ 1.07 & 76.75 $\pm$ 1.40 & 65.62 $\pm$ 0.49  & 71.72 $\pm$ 1.24 & 74.82 $\pm$ 2.04 & 77.35 $\pm$ 0.64 & 81.95 $\pm$ 1.64 & 77.73 $\pm$ 0.61& \textbf{95.86  $\pm$  2.18} \\
& MOOC &  82.38 $\pm$ 1.75 & 80.71 $\pm$ 2.08 & 81.53 $\pm$ 0.79 & 86.59 $\pm$ 2.03 & 71.74 $\pm$ 0.88 & 73.22 $\pm$ 1.21 & 77.09 $\pm$ 0.83 & 87.26 $\pm$ 0.83 & 73.87 $\pm$ 2.77 & 87.91 $\pm$ 0.93 & \textbf{95.22  $\pm$  1.65} \\
& Reddit & 80.70 $\pm$ 0.20 & 79.96 $\pm$ 0.23 & 79.60 $\pm$ 0.09 & 81.04 $\pm$ 0.23 & 80.42 $\pm$ 0.20 & 76.83 $\pm$ 0.12 & 77.83 $\pm$ 0.33 & 80.61 $\pm$ 0.48 & 90.63 $\pm$ 2.28 & 81.71 $\pm$ 0.49 & \textbf{97.49  $\pm$  0.17} \\
& Wikipedia &  80.71 $\pm$ 0.64 & 77.49 $\pm$ 0.72 & 82.83 $\pm$ 0.27 & 83.28 $\pm$ 0.26 & 65.74 $\pm$ 3.46 & 85.55 $\pm$ 0.47 & 87.47 $\pm$ 0.20 & 72.78 $\pm$ 6.65 & 95.43 $\pm$ 0.07 & 78.99 $\pm$ 1.24 & \textbf{99.02 $\pm$ 0.17}\\
& UCI & 78.21 $\pm$ 3.18 & 58.65 $\pm$ 3.58 & 57.12 $\pm$ 0.98 & 78.48 $\pm$ 1.79 & 57.67 $\pm$ 1.11 & 65.42 $\pm$ 2.62 & 77.46 $\pm$ 1.63 & 75.71 $\pm$ 0.57 & 75.05 $\pm$ 0.13 & 75.43 $\pm$ 1.99 & \textbf{87.86 $\pm$ 0.59} \\
& Social Evo.  & 91.83 $\pm$ 1.52 & 92.81 $\pm$ 0.60 & 93.63 $\pm$ 0.48 & 94.27 $\pm$ 1.33 & 87.61 $\pm$ 0.06 & 95.03 $\pm$ 0.82 & 94.65 $\pm$ 0.28 & 97.16 $\pm$ 0.06 & Timeout & 97.27 $\pm$ 0.30 & \textbf{98.89 $\pm$ 0.56}  \\
\midrule
& \textbf{Avg. Rank} & 6.00 & 7.71 & 8.43 & 4.43 & 9.57 & 8.14 & 6.86 & 5.00 & 5.14 & 3.71 & \textbf{1.00} \\
\midrule
\multirow{7}{*}{INS} 
& LastFM & 61.59 $\pm$ 5.72 & 60.62 $\pm$ 2.20 & 63.96 $\pm$ 0.41 & 65.48 $\pm$ 4.13 & 67.90 $\pm$ 0.44 & 54.75 $\pm$ 1.31 & 59.98 $\pm$ 0.20 & 67.87 $\pm$ 0.53 & 78.70 $\pm$ 0.87 & 68.74 $\pm$ 0.55 & \textbf{94.17 $\pm$  0.22} \\
& Enron & 70.75 $\pm$ 0.69 & 67.37 $\pm$ 2.21 & 59.78 $\pm$ 1.12 & 73.22 $\pm$ 0.42 & 75.29 $\pm$ 0.66 &  69.74 $\pm$ 1.19 & 70.72 $\pm$ 1.08 & 74.67 $\pm$ 0.80 & 75.40 $\pm$ 1.92 & 75.47 $\pm$ 1.41 & \textbf{95.80  $\pm$  1.96} \\
& MOOC & 67.53 $\pm$ 1.76 & 62.60 $\pm$ 1.27 & 74.44 $\pm$ 0.81 & 76.89 $\pm$ 2.13 & 70.08 $\pm$ 0.33  & 71.80 $\pm$ 1.09 & 72.25 $\pm$ 0.57 & 80.78 $\pm$ 0.89 & 68.17 $\pm$ 3.73 & 81.08 $\pm$ 0.82 & \textbf{99.08  $\pm$  0.35} \\
& Reddit & 83.40 $\pm$ 0.33 & 82.75 $\pm$ 0.36 & 87.46 $\pm$ 0.10 & 84.57 $\pm$ 0.19 & 88.19 $\pm$ 0.20  & 84.41 $\pm$ 0.18 & 82.24 $\pm$ 0.24 & 86.25 $\pm$ 0.64 & 91.42 $\pm$ 2.18 & 86.35 $\pm$ 0.52& \textbf{99.51  $\pm$  0.03} \\
& Wikipedia  & 70.41 $\pm$ 0.39 & 67.57 $\pm$ 0.94 & 81.54 $\pm$ 0.31 & 81.21 $\pm$ 0.30 & 68.48 $\pm$ 3.64 & 73.51 $\pm$ 1.88 & 84.20 $\pm$ 0.36 & 64.09 $\pm$ 9.75 & 93.67 $\pm$ 0.11 & 75.64 $\pm$ 2.42 & \textbf{99.36  $\pm$  0.07}\\
& UCI & 64.14 $\pm$ 1.25 & 54.10 $\pm$ 2.74 & 59.60 $\pm$ 0.61 & 63.76 $\pm$ 0.99 & 57.85 $\pm$ 0.59  & 65.46 $\pm$ 2.07 & 74.25 $\pm$ 0.71 & 64.92 $\pm$ 0.83 & 66.51 $\pm$ 0.25 & 66.83 $\pm$ 2.83  & \textbf{89.75 $\pm$ 0.32}\\
& Social Evo. & 91.81 $\pm$ 1.69 & 92.77 $\pm$ 0.64 & 93.54 $\pm$ 0.48 & 94.86 $\pm$ 1.25 & 90.10 $\pm$ 0.11 & 95.13 $\pm$ 0.83 & 94.50 $\pm$ 0.26 & 95.01 $\pm$ 0.15 & Timeout & 97.37 $\pm$ 0.26 &  \textbf{99.24 $\pm$ 0.47}\\
\midrule
& \textbf{Avg. Rank} & 8.29 & 9.86 & 6.71 & 5.86 & 6.86 & 7.14 & 6.57 & 5.71 & 3.67 & 3.29 & \textbf{1.00} \\
\bottomrule
\end{tabular}}
\end{table*}
\begin{table*}
\centering
\caption{AUC-ROC of inductive dynamic link prediction. } \label{tab:auc_roc_inductive}
\resizebox{\textwidth}{!}{
\begin{tabular}{llcccccccccccc}
\toprule
\textbf{NSS} & \textbf{Datasets} & JODIE & DyRep & TGAT & TGN & CAWN & TCL & GraphMixer & DyGFormer & CTAN & DyGmamba & \texttt{CTDG-SSM} \\
\midrule
\multirow{7}{*}{RNS} 
& LastFM & 83.13 $\pm$ 1.19 & 83.47 $\pm$ 1.06 & 78.40 $\pm$ 0.30 & 81.18 $\pm$ 3.27 & 89.33 $\pm$ 0.06 & 81.38 $\pm$ 1.53 & 82.07 $\pm$ 0.31 & 94.17 $\pm$ 0.10 & 60.40 $\pm$ 3.01 & 94.42 $\pm$ 0.21 & \textbf{94.49 $\pm$  0.27} \\
& Enron &  78.97 $\pm$ 1.59 & 73.97 $\pm$ 3.00 & 66.67 $\pm$ 1.07 & 78.76 $\pm$ 1.69 & 86.30 $\pm$ 0.56 & 82.61 $\pm$ 0.61 & 75.55 $\pm$ 0.81 & 89.62 $\pm$ 0.27 & 74.61 $\pm$ 1.64 & 89.67 $\pm$ 0.27& \textbf{93.66  $\pm$  4.67}\\
& MOOC& 80.57 $\pm$ 0.52 & 80.50 $\pm$ 0.68 & 85.28 $\pm$ 0.30 & 88.01 $\pm$ 1.48 & 81.32 $\pm$ 0.42 & 82.28 $\pm$ 0.99 & 81.38 $\pm$ 0.17 & 87.05 $\pm$ 0.51 & 64.99 $\pm$ 2.24 & 88.64 $\pm$ 0.08 & \textbf{98.67  $\pm$  0.46} \\
& Reddit  & 96.43 $\pm$ 0.16 & 95.89 $\pm$ 0.26 & 97.13 $\pm$ 0.04 & 97.41 $\pm$ 0.12 & 98.62 $\pm$ 0.01 & 95.01 $\pm$ 0.10 & 95.24 $\pm$ 0.08 & 98.83 $\pm$ 0.02 & 80.07 $\pm$ 2.53 & 98.97 $\pm$ 0.01 & \textbf{99.13  $\pm$  0.03} \\
& Wikipedia & 94.91 $\pm$ 0.32 & 92.21 $\pm$ 0.29 & 96.26 $\pm$ 0.12 & 97.81 $\pm$ 0.18 & 98.27 $\pm$ 0.02 & 97.48 $\pm$ 0.06 & 96.61 $\pm$ 0.04 & 98.58 $\pm$ 0.01 & 93.58 $\pm$ 0.65 & 98.77 $\pm$ 0.03 & \textbf{99.06  $\pm$  0.10} \\
& UCI & 79.73 $\pm$ 1.48 & 58.39 $\pm$ 2.38 & 79.10 $\pm$ 0.49 & 87.81 $\pm$ 1.32 & 92.61 $\pm$ 0.35 & 84.19 $\pm$ 1.37 & 91.17 $\pm$ 0.29 & 94.45 $\pm$ 0.13 & 49.78 $\pm$ 5.02 & \textbf{94.76 $\pm$ 0.19} & 87.43  $\pm$  0.79 \\
& Social Evo. & 91.72 $\pm$ 0.66 & 89.10 $\pm$ 1.90 & 91.47 $\pm$ 0.10 & 90.74 $\pm$ 1.40 & 79.83 $\pm$ 0.14 & 92.51 $\pm$ 0.11 & 91.89 $\pm$ 0.05 & 93.05 $\pm$ 0.10 & Timeout & 93.13 $\pm$ 0.05 & \textbf{98.60 $\pm$ 0.14} \\
\midrule
\textbf{} & \textbf{Avg. Rank} 
& 7.29 & 9.00 & 8.00 & 6.00 & 5.29 & 6.57 & 6.71 & 3.00 & 10.57 & 1.86 & \textbf{1.71} \\ 
\midrule
\multirow{7}{*}{INS} 
& LastFM & 69.85 $\pm$ 1.70 & 68.14 $\pm$ 1.61 & 69.89 $\pm$ 0.41 & 67.01 $\pm$ 5.77 & 67.72 $\pm$ 0.20 & 63.15 $\pm$ 1.17 & 69.93 $\pm$ 0.17 & 69.86 $\pm$ 0.80 & 57.85 $\pm$ 3.67 & 70.59 $\pm$ 0.57 & \textbf{ 94.77 $\pm$  0.26} \\
& Enron & 65.95 $\pm$ 1.27 & 62.20 $\pm$ 2.15 & 56.52 $\pm$ 0.84 & 64.21 $\pm$ 0.94 & 62.07 $\pm$ 0.72 & 67.56 $\pm$ 1.34 & 67.39 $\pm$ 1.33 & 66.07 $\pm$ 0.65 & 68.70 $\pm$ 1.82 & 68.98 $\pm$ 1.00 & \textbf{94.59  $\pm$  3.37} \\
& MOOC  & 65.37 $\pm$ 0.96 & 62.97 $\pm$ 2.05 & 74.94 $\pm$ 0.80 & 76.36 $\pm$ 2.91 & 71.18 $\pm$ 0.54 & 71.30 $\pm$ 1.21 & 72.15 $\pm$ 0.65 & 80.42 $\pm$ 0.72 & 58.06 $\pm$ 0.89 & 81.12 $\pm$ 0.63 & \textbf{98.71  $\pm$  0.47} \\
& Reddit & 61.84 $\pm$ 0.44 & 60.35 $\pm$ 0.53 & 64.92 $\pm$ 0.08 & 65.24 $\pm$ 0.08 & 65.37 $\pm$ 0.12 & 61.85 $\pm$ 0.11 & 64.56 $\pm$ 0.26 & 64.80 $\pm$ 0.53 & 81.70 $\pm$ 4.71 & 64.93 $\pm$ 0.89 & \textbf{ 99.15  $\pm$  0.03} \\
& Wikipedia & 61.66 $\pm$ 0.30 & 56.34 $\pm$ 0.67 & 78.40 $\pm$ 0.77 & 75.86 $\pm$ 0.50 & 59.00 $\pm$ 4.33 & 71.45 $\pm$ 2.23 & 82.76 $\pm$ 0.11 & 58.21 $\pm$ 8.78 & 91.12 $\pm$ 0.13 & 67.92 $\pm$ 2.23 & \textbf{99.09  $\pm$  0.10} \\
& UCI & 60.66 $\pm$ 1.82 & 51.50 $\pm$ 2.08 & 61.27 $\pm$ 0.78 & 62.07 $\pm$ 0.67 & 55.60 $\pm$ 1.22 & 65.87 $\pm$ 1.90 & 75.72 $\pm$ 0.70 & 64.37 $\pm$ 0.98 & 51.68 $\pm$ 2.60 & 66.95 $\pm$ 2.22 & \textbf{87.86  $\pm$  0.73} \\
& Social Evo. & 88.98 $\pm$ 0.81 & 86.43 $\pm$ 1.48 & 92.37 $\pm$ 0.50 & 91.66 $\pm$ 2.14 & 83.84 $\pm$ 0.21 & 95.50 $\pm$ 0.31 & 93.88 $\pm$ 0.22 & 94.97 $\pm$ 0.36 & Timeout & 96.65 $\pm$ 0.29 & \textbf{98.90 $\pm$ 0.14} \\
\midrule
\textbf{} & \textbf{Avg. Rank} 
& 8.00 & 9.71 & 6.14 & 6.14 & 8.14 & 6.14 & 4.57 & 5.71 & 7.14 & 3.29 & \textbf{1.00} \\ 
\bottomrule
\end{tabular}}
\end{table*}

\begin{table} 
\centering 
\caption{\footnotesize Model Hyperparameters. N/A: Not Applicable, OHE: One-hot encoding, LR: Learnable, RN: Random.}
\label{tab:hyperparameter_details}
\resizebox{\textwidth}{!}{
\begin{tabular}{lccccc}
\hline
Dataset & Latent dimension & Time embedding dimension & $N_u$ & Batch size & Static embedding \\
\hline
Enron & 32  & 16 & 10 & 128 & OHE \\
UCI   & 32  & 16 & 10 & 128 & RN\\
MOOC  & 32  & 16 & 10 & 128 & N/A\\
Wikipedia & 128 & 16 & 10 & 128 & N/A \\
Reddit & 128 & 16 & 10 & 128 & N/A \\
Lastfm & 32 & 16 & 10 & 128 & N/A\\
Flights  & 32  & 16 & 10 & 128 & N/A\\
Can. Parl.   & 32  & 16 & 10 & 128 & N/A\\
US Legis.  & 32  & 16 & 10 & 128 & N/A\\
UN Trade  & 32  & 16 & 10 & 128 & N/A\\
UN Vote  & 32  & 16 & 10 & 128 & N/A\\
Contact  & 32  & 16 & 10 & 128 & N/A\\
tgbl-wiki  & 128  & 16 & 10 & 128 & N/A\\
tgbl-coin  & 32  & 16 & 10 & 128 & LR\\
Sequence Classification & 32 & N/A & 10 & 128 & N/A \\
\hline
\end{tabular}}
\end{table}

\subsection{Additional Results and Hyperparameter details} \label{sec:Additional results}

In this section, we provide additional results for the dynamic link prediction task. Specifically, we report performance using average precision (AP) as an evaluation metric. Furthermore, we present AUC-ROC results under both inductive and transductive settings, comparing different sampling strategies.   In Table~\ref{tab:trans_auc_roc}, ~\ref{tab:Auc_roc transductive} and Table~\ref{tab:ap_transductive},~\ref{tab:Ap_transductive},~\ref{tab:freedyg under transductive setting} we report AUC-ROC and AP scores under the transductive setting with different sampling techniques. The results clearly demonstrate that the proposed model outperforms state-of-the-art algorithms on LRT datasets, primarily due to its ability to jointly encode structural information via graph polynomials that capture multi-hop neighborhood interactions and temporal evolution through a state-space formulation. In Table~\ref{tab:auc_roc_inductive},~\ref{tab:Auc_roc inductive setting}, and Table~\ref{tab:ap_inductive},~\ref{tab:AP inductive setting},~\ref{tab:freedyg-only under inductive setting} we report results under the inductive setting, where the task is more challenging since the test set includes nodes unseen during training. Additionally, we report the mean reciprocal rank (MRR) in Table~\ref{tab:tgbl_results} using the evaluation mechanism proposed in~\citep{tgbl} (values close to 1 are better). The proposed model not only outperforms existing approaches but also exhibits only a minor performance drop compared to the transductive setting, highlighting its ability to effectively capture global structural and temporal patterns instead of learning local structural patterns.  \\ 


\textbf{Hyperparameter Details}: In Table~\ref{tab:hyperparameter_details}, we report the hyperparameters used in all experiments. The latent dimension corresponds to the size of the memory representations, the batch size denotes the number of events in each batch, and OHE refers to one-hot encoding.

\begin{table}[htbp]
\centering
\caption{AP of transductive dynamic link prediction.} \label{tab:ap_transductive}
\resizebox{\textwidth}{!}{
\begin{tabular}{llccccccccccc}
\toprule
\textbf{NSS} & \textbf{Datasets} & JODIE & DyRep & TGAT & TGN & CAWN & TCL & GraphMixer & DyGFormer & CTAN & DyGmamba & \texttt{CTDG-SSM} \\ \hline
\multirow{7}{*}{RNS} & LastFM & 70.95 $\pm$ 2.94 & 71.85 $\pm$ 2.44 & 73.30 $\pm$ 0.18 & 75.31 $\pm$ 5.62 & 86.60 $\pm$ 0.11  & 76.62 $\pm$ 1.83 & 75.56 $\pm$ 0.19 & 92.95 $\pm$ 0.14 & 86.44 $\pm$ 0.80 & 93.35 $\pm$ 0.20 & \textbf{93.40 ± 0.49}\\ 
 & Enron & 84.85 $\pm$ 3.13 & 79.80 $\pm$ 2.28 & 70.76 $\pm$ 1.05 & 86.98 $\pm$ 1.05 & 89.50 $\pm$ 0.10 &  85.41 $\pm$ 0.71 & 82.13 $\pm$ 0.30 & 92.42 $\pm$ 0.11 & 92.52 $\pm$ 1.20 & 92.65 $\pm$ 0.12 & \textbf{94.46 $\pm$ 4.73} \\ 
 & MOOC & 81.04 $\pm$ 0.83 & 81.50 $\pm$ 0.77 & 85.71 $\pm$ 0.20 & 89.15 $\pm$ 1.69 & 80.30 $\pm$ 0.43 &  83.89 $\pm$ 0.86 & 82.80 $\pm$ 0.15 & 87.66 $\pm$ 0.48 & 84.71 $\pm$ 2.85 & 89.21 $\pm$ 0.08 & \textbf{98.85 $\pm$ 0.35}\\ 
 & Reddit & 98.31 $\pm$ 0.06 & 98.18 $\pm$ 0.03 & 98.57 $\pm$ 0.01 & 98.65 $\pm$ 0.04 & 99.11 $\pm$ 0.01 & 97.78 $\pm$ 0.02 & 97.31 $\pm$ 0.01 & 99.22 $\pm$ 0.01 & 97.21 $\pm$ 0.84 & 99.32 $\pm$ 0.01 & \textbf{99.53 $\pm$ 0.02} \\ 
 & Wikipedia & 96.51 $\pm$ 0.22 & 94.88 $\pm$ 0.29 & 96.88 $\pm$ 0.06 & 98.45 $\pm$ 0.10 & 98.77 $\pm$ 0.01 & 97.75 $\pm$ 0.04 & 97.22 $\pm$ 0.02 & 99.03 $\pm$ 0.03 & 96.61 $\pm$ 0.79 & 99.15 $\pm$ 0.02 &  \textbf{99.40 $\pm$ 0.09}\\ 
 & UCI & 89.28 $\pm$ 1.02 & 66.11 $\pm$ 2.75 & 79.40 $\pm$ 0.61 & 92.33 $\pm$ 0.64 & 95.13 $\pm$ 0.23 & 86.63 $\pm$ 1.30 & 93.15 $\pm$ 0.41 & 95.74 $\pm$ 0.17 & 76.64 $\pm$ 4.11 & \textbf{95.91 $\pm$ 0.15} & 90.18 $\pm$ 0.98 \\ 
 & Social Evo. & 89.88 $\pm$ 0.40 & 88.39 $\pm$ 0.69 & 93.33 $\pm$ 0.06 & 93.45 $\pm$ 0.29 & 84.90 $\pm$ 0.11 & 93.82 $\pm$ 0.19 & 93.36 $\pm$ 0.06 & 94.63 $\pm$ 0.07 & Timeout & 94.77 $\pm$ 0.01 &  \textbf{98.65 $\pm$ 0.65}\\ \midrule
\textbf{} & \textbf{Avg. Rank} 
& 8.71 & 9.71 & 7.86 & 5.29 & 5.86 & 6.71 & 7.29 & 3.14 & 7.86 & 1.86 & \textbf{1.71} \\ \hline
\multirow{7}{*}{HNS} & LastFM  & 74.38 $\pm$ 6.27 & 71.85 $\pm$ 2.91 & 71.60 $\pm$ 0.36 & 75.03 $\pm$ 6.90 & 69.93 $\pm$ 0.33 & 71.02 $\pm$ 2.07 & 72.28 $\pm$ 0.37 & 81.51 $\pm$ 0.14 & 82.29 $\pm$ 0.94 & 83.02 $\pm$ 0.16 & \textbf{88.91 ± 0.93}  \\ 
 & Enron & 69.13 $\pm$ 1.66 & 72.58 $\pm$ 1.83 & 64.24 $\pm$ 1.24 & 74.31 $\pm$ 1.99 & 65.40 $\pm$ 0.36 & 72.39 $\pm$ 0.61 & 77.35 $\pm$ 1.22 & 76.93 $\pm$ 0.76 & 77.24 $\pm$ 1.53 & 77.77 $\pm$ 1.32 &\textbf{95.80 $\pm$ 3.33}\\ 
 & MOOC & 78.62 $\pm$ 2.43 & 75.14 $\pm$ 2.86 & 82.83 $\pm$ 0.71 & 85.65 $\pm$ 2.32 & 74.46 $\pm$ 0.53 & 78.51 $\pm$ 0.84 & 77.09 $\pm$ 0.83 & 86.43 $\pm$ 0.38 & 67.73 $\pm$ 2.08 & 85.89 $\pm$ 0.94 & \textbf{94.76 $\pm$ 1.76}\\ 
 & Reddit & 79.96 $\pm$ 0.30 & 79.40 $\pm$ 0.00 & 79.78 $\pm$ 0.25 & 81.05 $\pm$ 0.32 & 80.96 $\pm$ 0.28 & 77.38 $\pm$ 0.02 & 78.39 $\pm$ 0.36 & 83.81 $\pm$ 1.08 & 89.77 $\pm$ 2.28 & 88.81 $\pm$ 1.52  & \textbf{97.55 $\pm$ 0.22}\\ 
 & Wikipedia & 81.16 $\pm$ 0.73 & 79.44 $\pm$ 0.95 & 87.31 $\pm$ 0.36 & 87.31 $\pm$ 0.25 & 66.77 $\pm$ 6.62 & 86.12 $\pm$ 1.69 & 90.74 $\pm$ 0.06 & 70.13 $\pm$ 11.02 & 95.91 $\pm$ 0.10 & 81.77 $\pm$ 1.20 & \textbf{98.99 $\pm$ 0.32}\\\ 
 & UCI& 74.77 $\pm$ 5.35 & 55.89 $\pm$ 2.83 & 66.78 $\pm$ 0.77 & 81.32 $\pm$ 1.26 & 64.69 $\pm$ 1.78  & 74.62 $\pm$ 2.70 & 83.88 $\pm$ 1.06 & 80.44 $\pm$ 1.16 & 76.62 $\pm$ 0.33 & 81.03 $\pm$ 1.09 & \textbf{88.87 $\pm$ 1.28}\\ 
 & Social Evo.& 91.26 $\pm$ 2.47 & 92.86 $\pm$ 0.90 & 95.31 $\pm$ 0.30 & 93.84 $\pm$ 1.68 & 85.65 $\pm$ 0.11  & 95.93 $\pm$ 0.63 & 95.30 $\pm$ 0.34 & 97.05 $\pm$ 0.16 & Timeout & 97.35 $\pm$ 0.52  &  \textbf{98.20 $\pm$ 0.81} \\ \midrule
\textbf{} & \textbf{Avg. Rank} 
& 7.43 & 8.71 & 7.36 & 4.93 & 9.71 & 7.71 & 5.57 & 4.71 & 5.57 & 3.29 & \textbf{1.00} \\ \hline
\multirow{7}{*}{INS} & LastFM & 62.63 $\pm$ 6.89 & 62.49 $\pm$ 3.04 & 71.16 $\pm$ 0.33 & 65.09 $\pm$ 7.05 & 67.38 $\pm$ 0.57 & 62.76 $\pm$ 0.81 & 67.87 $\pm$ 0.37 & 72.60 $\pm$ 0.06 & 80.06 $\pm$ 0.85 & 73.63 $\pm$ 0.54 & \textbf{93.81 ± 0.44}  \\ 
 & Enron& 69.51 $\pm$ 1.06 & 66.78 $\pm$ 2.21 & 63.16 $\pm$ 0.59 & 73.27 $\pm$ 0.58 & 75.08 $\pm$ 0.81 & 70.98 $\pm$ 0.96 & 74.12 $\pm$ 0.65 & 78.22 $\pm$ 0.80 & 72.02 $\pm$ 2.64 & 80.86 $\pm$ 1.24 & \textbf{95.81 $\pm$ 2.99}\\ 
 & MOOC& 66.56 $\pm$ 1.49 & 61.48 $\pm$ 0.96 & 76.96 $\pm$ 0.89 & 77.59 $\pm$ 1.83 & 73.55 $\pm$ 0.36 & 76.35 $\pm$ 1.41 & 74.24 $\pm$ 0.75 & 80.99 $\pm$ 0.88 & 64.93 $\pm$ 3.31 & 81.11 $\pm$ 0.63& \textbf{99.03 $\pm$ 0.38}\\ 
 & Reddit & 86.93 $\pm$ 0.21 & 86.06 $\pm$ 0.36 & 89.93 $\pm$ 0.10 & 88.12 $\pm$ 0.13 & 91.89 $\pm$ 0.18 & 86.97 $\pm$ 0.26 & 85.37 $\pm$ 0.26 & 91.06 $\pm$ 0.60 & 90.99 $\pm$ 2.19 & 91.15 $\pm$ 0.54 & \textbf{99.58 $\pm$ 0.02}\\ 
 & Wikipedia & 74.78 $\pm$ 0.56 & 70.55 $\pm$ 1.22 & 86.77 $\pm$ 0.29 & 85.80 $\pm$ 0.15 & 69.27 $\pm$ 7.07 & 72.54 $\pm$ 4.69 & 88.54 $\pm$ 0.20 & 62.00 $\pm$ 14.00 & 94.15 $\pm$ 0.08 & 79.86 $\pm$ 2.18 & \textbf{99.45 $\pm$ 0.06} \\ 
 & UCI & 66.02 $\pm$ 1.28 & 54.64 $\pm$ 2.52 & 67.63 $\pm$ 0.51 & 70.34 $\pm$ 0.72 & 64.08 $\pm$ 1.06 & 73.49 $\pm$ 2.21 & 79.57 $\pm$ 0.61 & 70.51 $\pm$ 1.83 & 66.25 $\pm$ 0.51 & 71.95 $\pm$ 2.51 & \textbf{91.44 $\pm$ 0.50} \\ 
 & Social Evo. & 91.08 $\pm$ 3.29 & 92.84 $\pm$ 0.98 & 95.20 $\pm$ 0.30 & 94.58 $\pm$ 1.52 & 88.50 $\pm$ 0.13  & 96.14 $\pm$ 0.63 & 95.11 $\pm$ 0.32 & 97.62 $\pm$ 0.12 & Timeout & 97.68 $\pm$ 0.42 & \textbf{98.88 $\pm$ 0.63} \\ \hline
\textbf{} & \textbf{Avg. Rank} 
& 8.86 & 10.00 & 6.14 & 6.14 & 7.29 & 6.57 & 5.71 & 4.71 & 6.43 & 3.14 & \textbf{1.00} \\ \hline
\end{tabular}%
}
\end{table}
\begin{table*}[t] 
\centering
\caption{\footnotesize AP of inductive dynamic link prediction.}
\label{tab:ap_inductive}
\resizebox{\textwidth}{!}{
\begin{tabular}{llccccccccccc}
\toprule
NSS & Datasets & JODIE & DyRep & TGAT & TGN & CAWN & TCL & GraphMixer & DyGFormer & CTAN & DyGmamba& \texttt{CTDG-SSM} \\ \midrule

\multirow{7}{*}{RNS} 
& LastFM      & 83.13 $\pm$ 1.19 & 83.47 $\pm$ 1.06 & 78.40 $\pm$ 0.30 & 81.18 $\pm$ 3.27 & 89.33 $\pm$ 0.06 & 81.38 $\pm$ 1.53 & 82.07 $\pm$ 0.31 & 94.17 $\pm$ 0.10 & 60.40 $\pm$ 3.01 & \textbf{94.42 $\pm$ 0.21} & 93.65 $\pm$  0.62 \\
& Enron       & 78.97 $\pm$ 1.59 & 73.97 $\pm$ 3.00 & 66.67 $\pm$ 1.07 & 78.76 $\pm$ 1.69 & 86.30 $\pm$ 0.56 & 82.61 $\pm$ 0.61 & 75.55 $\pm$ 0.81 & 89.62 $\pm$ 0.27 & 74.61 $\pm$ 1.64 & 89.67 $\pm$ 0.27 & \textbf{93.02 $\pm$ 7.25} \\
& MOOC        & 80.57 $\pm$ 0.52 & 80.50 $\pm$ 0.68 & 85.28 $\pm$ 0.30 & 88.01 $\pm$ 1.48 & 81.32 $\pm$ 0.42 & 82.28 $\pm$ 0.99 & 81.38 $\pm$ 0.17 & 87.05 $\pm$ 0.51 & 64.99 $\pm$ 2.24 & 88.64 $\pm$ 0.08 & \textbf{98.49 $\pm$ 0.48} \\
& Reddit      & 96.43 $\pm$ 0.16 & 95.89 $\pm$ 0.26 & 97.13 $\pm$ 0.04 & 97.41 $\pm$ 0.12 & 98.62 $\pm$ 0.01 & 95.01 $\pm$ 0.18 & 95.24 $\pm$ 0.08 & 98.83 $\pm$ 0.02 & 80.07 $\pm$ 2.53 & 98.97 $\pm$ 0.01 &  \textbf{99.28 $\pm$ 0.03} \\
& Wikipedia   & 94.91 $\pm$ 0.32 & 92.21 $\pm$ 0.29 & 96.26 $\pm$ 0.12 & 97.81 $\pm$ 0.18 & 98.27 $\pm$ 0.02 & 97.48 $\pm$ 0.06 & 96.61 $\pm$ 0.04 & 98.58 $\pm$ 0.01 & 93.58 $\pm$ 0.65 & 98.77 $\pm$ 0.03 & \textbf{99.19 $\pm$ 0.09}\\
& UCI         & 79.73 $\pm$ 1.48 & 58.39 $\pm$ 2.38 & 79.10 $\pm$ 0.49 & 87.81 $\pm$ 1.32 & 92.61 $\pm$ 0.35 & 84.19 $\pm$ 1.67 & 91.17 $\pm$ 0.29 & 94.45 $\pm$ 0.13 & 49.78 $\pm$ 5.02 & \textbf{94.76 $\pm$ 0.19} & 89.12 $\pm$ 1.02\\
& Social Evo. & 91.72 $\pm$ 0.66 & 89.10 $\pm$ 1.90 & 91.47 $\pm$ 0.10 & 90.74 $\pm$ 1.40 & 79.83 $\pm$ 0.14 & 92.51 $\pm$ 0.11 & 91.89 $\pm$ 0.05 & \textbf{93.05 $\pm$ 0.10} & Timeout & 93.13 $\pm$ 0.05 & \textbf{97.56 $\pm$ 0.45} \\
\midrule
\textbf{} & \textbf{Avg. Rank} 
& 7.29 & 9.00 & 8.00 & 6.14 & 5.29 & 6.57 & 6.71 & 2.86 & 10.57 & \textbf{1.71} & 1.86 \\ \midrule

\multirow{7}{*}{INS} 
& LastFM      & 71.37 $\pm$ 3.45 & 69.75 $\pm$ 2.73 & 76.26 $\pm$ 0.34 & 68.47 $\pm$ 6.07 & 71.28 $\pm$ 0.43 & 68.79 $\pm$ 0.93 & 76.27 $\pm$ 0.37 & 75.07 $\pm$ 1.45 & 55.60 $\pm$ 3.91 & 76.76 $\pm$ 0.43 & \textbf{94.08 $\pm$  0.57} \\
& Enron      & 66.99 $\pm$ 1.15 & 62.64 $\pm$ 2.33 & 59.95 $\pm$ 1.00 & 64.51 $\pm$ 1.66 & 60.61 $\pm$ 0.63 & 68.93 $\pm$ 1.34 & 71.71 $\pm$ 1.33 & 67.21 $\pm$ 0.72 & 68.66 $\pm$ 2.31 & 68.77 $\pm$ 0.60 & \textbf{94.56 $\pm$ 5.02} \\
& MOOC        & 64.67 $\pm$ 1.18 & 62.05 $\pm$ 2.11 & 77.43 $\pm$ 0.81 & 76.81 $\pm$ 2.83 & 74.36 $\pm$ 0.78 & 75.95 $\pm$ 1.46 & 73.87 $\pm$ 0.99 & 80.66 $\pm$ 0.94 & 57.49 $\pm$ 1.34 & 80.75 $\pm$ 1.00 & \textbf{98.64 $\pm$ 0.51}\\
& Reddit     & 62.54 $\pm$ 0.52 & 61.07 $\pm$ 0.86 & 63.96 $\pm$ 0.25 & 65.27 $\pm$ 0.57 & 64.10 $\pm$ 0.22 & 61.45 $\pm$ 0.25 & 64.82 $\pm$ 0.30 & 65.03 $\pm$ 1.20 & 78.35 $\pm$ 5.03 & 65.30 $\pm$ 1.05 & \textbf{99.32 $\pm$ 0.03} \\
& Wikipedia  & 68.22 $\pm$ 0.36 & 61.07 $\pm$ 0.82 & 84.19 $\pm$ 0.96 & 81.96 $\pm$ 0.62 & 62.34 $\pm$ 6.79 & 71.46 $\pm$ 4.95 & 87.47 $\pm$ 0.25 & 57.90 $\pm$ 11.05 & 92.61 $\pm$ 0.90 & 71.14 $\pm$ 2.44 & \textbf{99.23 $\pm$ 0.09} \\
& UCI         & 63.57 $\pm$ 2.15 & 52.63 $\pm$ 1.87 & 69.77 $\pm$ 0.43 & 69.94 $\pm$ 0.50 & 63.44 $\pm$ 1.52 & 74.39 $\pm$ 1.81 & 81.40 $\pm$ 0.52 & 70.25 $\pm$ 2.02 & 52.31 $\pm$ 2.67 & 72.17 $\pm$ 2.20  & \textbf{90.34 $\pm$ 0.74}\\
& Social Evo.  & 89.06 $\pm$ 1.23 & 87.30 $\pm$ 1.55 & 94.24 $\pm$ 0.36 & 90.67 $\pm$ 2.41 & 80.30 $\pm$ 0.21 & 95.94 $\pm$ 0.37 & 94.56 $\pm$ 0.24 & 96.73 $\pm$ 0.11 & Timeout & 96.83 $\pm$ 0.56 & \textbf{98.15 $\pm$ 0.27} \\ 
\midrule
\textbf{} & \textbf{Avg. Rank} 
& 7.86 & 9.57 & 6.29 & 6.43 & 8.43 & 5.86 & 4.14 & 5.43 & 7.57 & 3.43 & \textbf{1.00} \\ 
\bottomrule
\end{tabular}
}
\end{table*}

\section{Model Efficiency}
\subsection{ Batch level Subgraph sampling}
{
The discrete update equation involves computing $p(\mL)^{-1}$, which incurs a cost of $\cO(N_\tau^3)$, where $N_\tau$ denotes the number of nodes in $\cG_\tau$. To implement this update efficiently, we operate on a subset of nodes from $\cG_\tau$ whose states are updated, while the remaining node states are kept unchanged. We refer to this subset as the \emph{active batch nodes}. This set includes:
\begin{itemize}
    \item Nodes appearing in interaction events of the form $(u, v, t)$ within the batch.
    \item Neighbors of the nodes selected from these interactions.
\end{itemize}

Neighbor selection depends on the chosen polynomial. For first-order polynomials, we select at most $K$ of the most recent 1-hop neighbors for each node $u$ and $v$ in interaction $(u,v,t)$. For a polynomial of order $m$, we extend this to an $m$-hop neighborhood. All nodes in this $m$-hop ego network are enumerated, and at most $K$ neighbors are chosen based on temporal proximity, using the most recent timestamp along the path. For example, if a node $w$ is connected to $u$ via $v$ through
\[
u \xrightarrow{t_1} v \xrightarrow{t_2} w,
\]
then the time associated with $w$ when sampling neighbors for interaction $(u, v, t)$ is computed as
\[
t_w = (t - t_1) + (t - t_2).
\]

The number of active batch nodes $N_B$ for a batch of length $B$ satisfies
\[
N_B \le 2 B K \ll N_\tau,
\]
resulting in a substantial reduction in update cost.
}

\subsection{Additional experiments on Runtime and Convergence} \label{sec:run_timeanalysis1}
In Table~\ref{tab:model-total-time-ctdgssm}  we present the total training time, obtained as the product of the per-epoch time and the number of training epochs. In , we analyze the convergence behavior of proposed algorithm where we show the training loss across epochs for multiple datasets. The plots clearly show that the proposed model converges within only a few epochs highlighting its computational efficiency.

\begin{table}
\centering
\caption{Number of epochs and total time (minutes) across datasets.}
\begin{tabular}{l|cc|cc|cc}
\hline
\textbf{Models} 
& \multicolumn{2}{c|}{\textbf{Enron}} 
& \multicolumn{2}{c|}{\textbf{UCI}} 
& \multicolumn{2}{c}{\textbf{Reddit}} \\
& \textbf{\#Epoch} & \textbf{T$_{tot}$} 
& \textbf{\#Epoch} & \textbf{T$_{tot}$} 
& \textbf{\#Epoch} & \textbf{T$_{tot}$} \\
\hline
CTAN        & 173.00 & 86.50  & 236.00 & 89.68  & 327.18 & 173.41 \\
DyGFormer   & 32.80  & 89.54  & 34.80  & 21.58  & 24.60  & 104.30 \\
DyGMamba    & 33.00  & 67.65  & 28.00  & 16.80  & 26.80  & 88.98 \\
CTDG-SSM    & 83.00  & \textbf{45.65}  & 38.00  & \textbf{6.46}   & 27.00  & \textbf{52.65} \\
\hline
\end{tabular}
\label{tab:model-total-time-ctdgssm}
\end{table} 

\begin{figure}
    \centering
    \includegraphics[width=0.5\columnwidth]{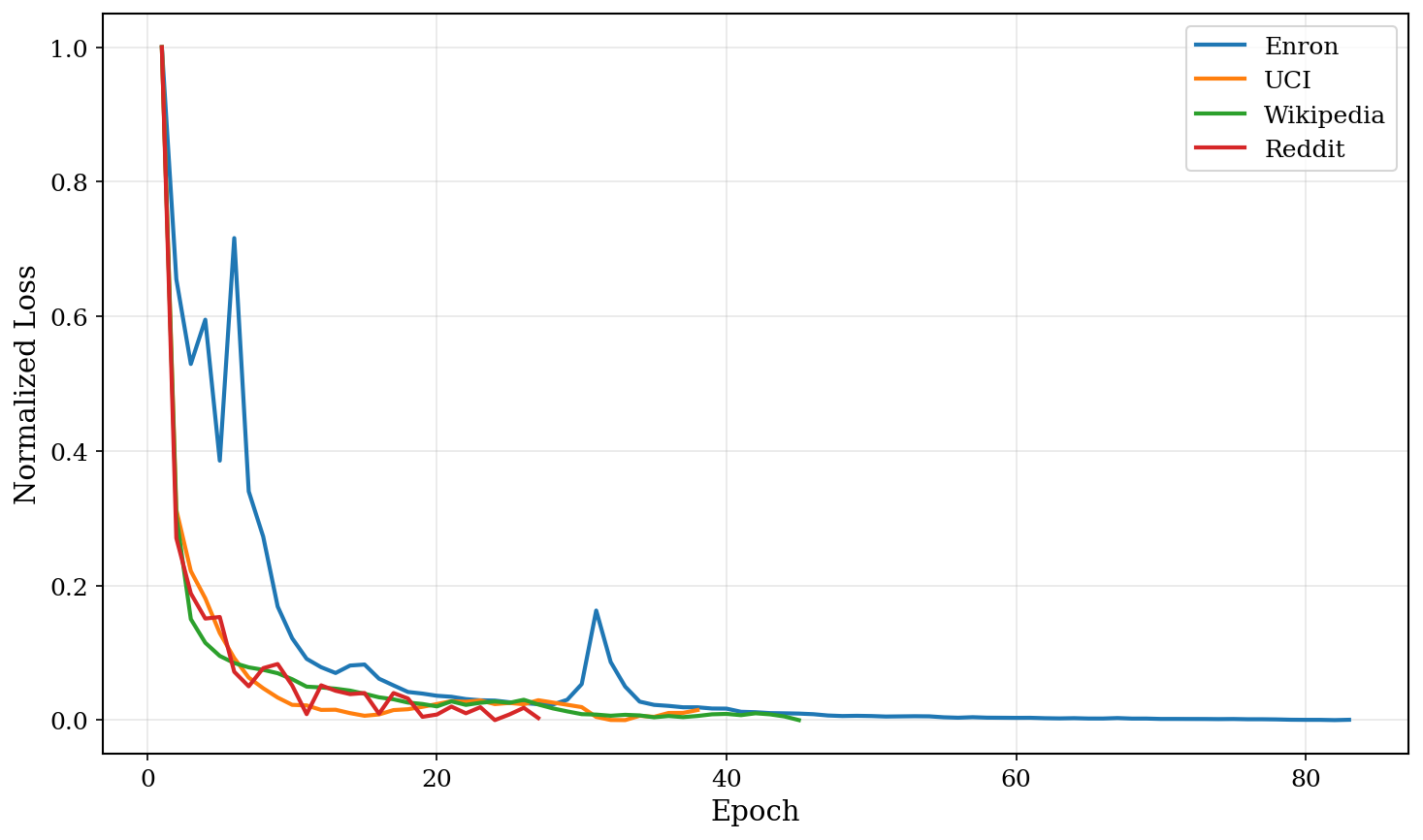}
    \caption{Convergence behavior of \texttt{CTDG-SSM} across multiple datasets}
    \label{fig:Loss-conver-enron}
\end{figure}

\begin{table}
\centering
\caption{Ablation study with respect to the order of the graph filter.}
\small
\begin{tabular}{lccc}
\toprule
\textbf{Prediction Node} 
& ${p(\mL)=\mI}$ 
& ${p(\mL)= \alpha_0\mI + \alpha_1\mL}$ 
& ${p(\mL)= \alpha_0 \mI + \alpha_1 \mL + \alpha_2 \mL^2}$ \\
\midrule
First       & 1.00  $\pm$ 0.00                  & 1.00 $\pm$ 0.00                   & 1.00 $\pm$ 0.00 \\
Second      & 0.51 $\pm$ 0.06         & 0.97 $\pm$ 0.02         & 1.00 $\pm$ 0.00 \\
Third       & 0.47 $\pm$ 0.02         & 0.96 $\pm$ 0.01         & 1.00 $\pm$ 0.00 \\
Second-Last & 0.46 $\pm$ 0.01         & 0.90 $\pm$ 0.01         & 0.92 $\pm$ 0.07 \\
Last        & 0.45 $\pm$ 0.02         & 0.88 $\pm$ 0.18         & 0.90 $\pm$ 0.06 \\
\bottomrule
\end{tabular}
\label{tab:poly_P_L_results}
\end{table}
{
\subsection{ Robustness to Structural Perturbations}

We evaluate the robustness of the proposed method to structural perturbations on the Enron dataset, using link prediction as the downstream task. In particular, we introduce the perturbations to the true graph as  
$
\bar{\mL}_{B}[k] = \mL_{B}[k] + \epsilon \Delta \mL_{B}[k],
$
where $\Delta \mL_{B}[k]$ is a perturbation matrix whose entries are sampled from a normal distribution, i.e., $[\Delta \mL_{B}]_{ij} \sim \mathcal{N}(0,1)$ and $\epsilon$ controls the noise level.

We then evaluate the proposed algorithm by replacing $\mL_{B}[k]$ with $\bar{\mL}_{B}[k]$ under different values of $\epsilon$, thereby varying the severity of the structural perturbation. In Fig.~\ref{fig:mooc_delta_t}(a), we report the accuracy across these noise levels. As expected, accuracy decreases as the noise variance increases; however, for small values of $\epsilon$, the model performs very close to the noise-free setting. This demonstrates that the proposed approach is stable and robust under mild structural perturbations.

{
\subsection{Ablation study: Importance of graph filters and CTDG-SSM module}

We present an ablation study to understand which components of the model capture long-range information. Specifically, we analyze the role of graph filters and the proposed \texttt{CTDG-SSM} module on the long-range spatial ($\texttt{LRS}$) task and the long-range temporal ($\texttt{LRT}$) task.

To evaluate the $\texttt{LRT}$ capabilities of CTDG-SSM, we evaluate model with filter order as $1$ i.e $p(\mL) =\mathbf{I}$. For an event of the form $(u,v,t,x_u,x_v)$ in sequence classification, instead of restricting the model to update only a small subset of nodes (i.e., those in the batch subgraph), we update the state vectors of all nodes. Formally, we define the input signal at time $t$ as 
$
\mX(t) \in \mathbb{R}^{N _\tau \times 1} \quad \text{such that} \quad \mX [u] (t) = x _u, \mX [v] (t) = x _v, \text{and 0 otherwise}.
$
This eliminates the step where the previous states of inactive nodes are carried forward through memory. This carry-forward mechanism could aid the model in LRT, so by removing it, we can evaluate the LRT capability of CTDG-SSM in isolation.
In this setup, the model is tasked with predicting the initial feature $\mX [0]$ observed at the first node using the final state vectors of different nodes. A successful LRT would yield strong performance as long as the state vector of node 1 preserves the information of the feature $\mX [0]$. Notably, this model completely lacks LRS capability, as it does not account for the underlying graph structure and updates states solely based on the input at the corresponding nodes. Next, we evaluate the effect of aggregating multi-hop information by applying graph filters of different orders. In Table~\ref{tab:poly_P_L_results}, we report the mean accuracies obtained from representations at different nodes using filter orders 1and 
2. We observe a substantial improvement in prediction accuracy by leveraging the representations from the node 
$2,\ldots,31 \text{(the last node)}$, demonstrating the model’s enhanced ability to preserve spatial information over longer ranges. In particular, using deeper aggregation-i.e., a filter of order 
2-yields a notable gain in accuracy, indicating that incorporating information from larger hop neighborhoods significantly strengthens the model’s capacity to capture long-range spatial dependencies. 
\subsection{Ablation study: Impact of default time}
In CTDG-SSM, during the link prediction task, the final link probability is computed using the node state vectors together with \(\Delta t\) the time elapsed since the last occurrence of the queried link. When a link that has never occurred before is queried, the model assigns a large default value to \(\Delta t\), following standard practice. In Fig.~\ref{fig:mooc_delta_t}(b), we evaluate the model’s performance across different choices of this default time value. The results show that performance remains consistent over a wide range of default \(\Delta t\) values, indicating that the model does not rely solely on large \(\Delta t\) values to generate final prediction probabilities.

\begin{figure}[htbp]
     \centering
     \begin{subfigure}[b]{0.45\textwidth}
         \centering
         \includegraphics[width=\textwidth]{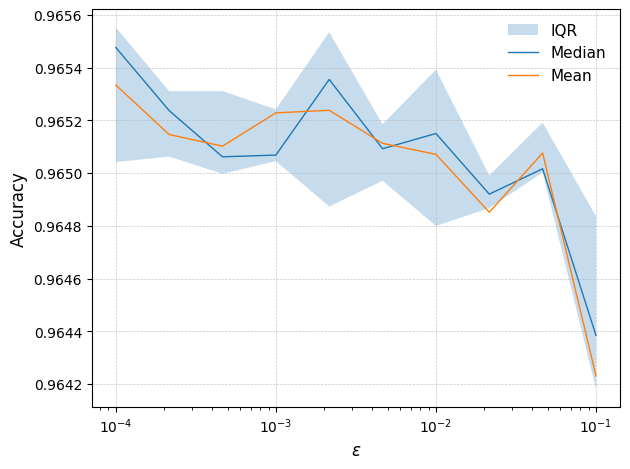}
         \caption{Robustness to $\epsilon$-noise (Enron).}
     \end{subfigure}
     \hfill
     \begin{subfigure}[b]{0.45\textwidth}
         \centering
         \includegraphics[width=\textwidth]{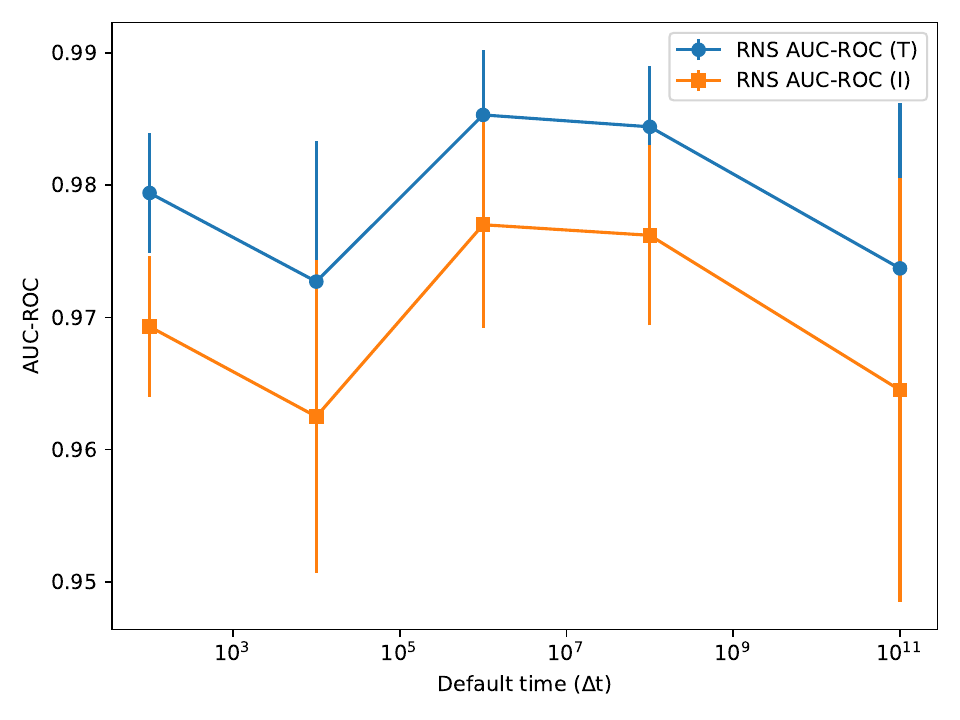}
         \caption{Sensitivity to $\Delta t$ (MOOC).}
     \end{subfigure}
        \caption{\textbf{Extended Robustness and Ablation Studies.} Left: Accuracy of CTDG-SSM on the Enron dataset under noise injection of the form $\mathbf{L}_B[k]+\epsilon\frac{\Delta \mathbf{L}}{\|\Delta\mathbf{L}\|_2}$, demonstrating model stability even under significant perturbation. Right: AUC-ROC on the MOOC dataset across transductive and inductive settings, showing that the model is relatively insensitive to the choice of the default time interval $\Delta t$ for unseen links.}
        \label{fig:mooc_delta_t}
\end{figure}
\subsection{Ablation Study: Impact of Static Embeddings, Normalization, and Residual Connections}
\label{app:ablation_study}

In this section, we present a systematic ablation study analyzing the impact of removing static embeddings (SE), Root Mean Square (RMS) normalization, and residual connections. Specifically, we evaluate three model variants against our proposed architecture:
\begin{itemize}
    \item \texttt{CTDG-SSM-No-SE}: Disables static embeddings.
    \item \texttt{CTDG-SSM-No-RMS}: Excludes all normalization layers while retaining other modules.
    \item \texttt{CTDG-SSM-No-Res}: Omits residual connections.
    \item \texttt{CTDG-SSM}: The proposed full architecture.
\end{itemize}

Notably, while the choices of residual connections and normalization are adapted from existing non-graph Mamba architectures to ensure training stability, their removal results in relatively marginal performance shifts across most datasets. 

The most significant performance degradation occurs upon the removal of static embeddings. This drop is particularly evident on the Enron dataset ($0.9401 \to 0.9167$) and the UCI dataset ($0.8814 \to 0.8701$), highlighting the critical importance of static embeddings for datasets that lack inherent node or edge features. Conversely, as detailed in Table~\ref{tab:hyperparameter_details}, static embeddings are omitted for datasets where rich node or edge features are natively available (\textit{e.g.}, Wiki, Reddit). As shown in Table~\ref{tab:ablation_study_se}, the core state-space evolution mechanism remains the primary driver of the model's predictive power and structural robustness, maintaining strong performance even when auxiliary modules are omitted.

\subsection{Ablation Study: Impact of Batch size}
Since the batch size $B$ is a critical training hyperparameter, we conduct a sensitivity analysis evaluating the proposed model across a range of batch choices, specifically $B \in \{64, 128, 200\}$, across four benchmark datasets. The choice of batch size heavily influences temporal granularity and computational complexity; specifically, excessively large batch sizes can lose fine-grained chronological patterns in event streams by grouping too many interactions together.

As presented in Table~\ref{tab:batch_size_sensitivity}, our model exhibits strong empirical robustness and stability across different batch configurations, with overall predictive performance remaining high and consistently competitive. Across all four benchmark datasets, changing the batch size $B$ from 64 to 200 results in only marginal variations in AUC. Crucially, the performance metrics at a standard batch size of $B=128$ are nearly identical to those achieved at $B=200$. This sensitivity analysis demonstrates that the model's predictive power is highly stable with respect to this hyperparameter, confirming that our core architectural benefits and state-of-the-art comparative claims remain robust and virtually unchanged regardless of the specific batch size selected for deployment.

\begin{table}[t]
    \centering
    \caption{Ablation and sensitivity analysis for the proposed CTDG-SSM model.}
    \label{tab:combined_experiments}
    \vspace{0.5em}
    
    \begin{subtable}[b]{0.48\textwidth}
        \centering
                \caption{Ablation study of CTDG-SSM components across three benchmark datasets. We evaluate the impact of Static Embeddings (SE), RMS Normalization (RMS), and Residual Connections (Res) and present the mean AUC-ROC of 5 runs with \textit{inductive} random negative sampling.}
        \begin{tabular}{lccc}
            \toprule
            \textbf{Model Variant} & \textbf{Enron} & \textbf{UCI} & \textbf{Wiki} \\
            \midrule
            CTDG-SSM-No-SE  & 0.9167 & 0.8701 & --- \\
            CTDG-SSM-No-RMS & 0.9396 & 0.8838 & 0.9883 \\
            CTDG-SSM-No-Res & 0.9337 & 0.8743 & \textbf{0.9930} \\
            \textbf{CTDG-SSM} & \textbf{0.9401} & \textbf{0.8814} & 0.9900 \\
            \bottomrule
        \end{tabular}
        \label{tab:ablation_study_se}
    \end{subtable}
    \hfill 
    \begin{subtable}[b]{0.48\textwidth}
        \centering
                        \caption{Sensitivity analysis of Batch Size on Model Performance (AUC). Results demonstrate the robustness of the CTDG-SSM architecture across varying computational constraints.}
        \begin{tabular}{lccc}
            \toprule
            \textbf{Dataset} & \textbf{64} & \textbf{128} & \textbf{200} \\
            \midrule
            Wiki   & 0.9905 & \textbf{0.9933} & 0.9931 \\
            UCI    & 0.8416 & 0.8724 & \textbf{0.8796} \\
            Enron  & 0.9261 & 0.9369 & \textbf{0.9385} \\
            Reddit & 0.9906 & 0.9915 & \textbf{0.9919} \\
            \bottomrule
        \end{tabular}
        \label{tab:batch_size_sensitivity}
    \end{subtable}
\end{table}

{
\section{Additional experiments}

In this section, we present results on additional temporal datasets-Flights, Contacts, UN Trade, UN Vote, and CanParl~\citep{DYGFORMER}-using link prediction as the downstream task. We further compare the proposed method with several state-of-the-art approaches, including \texttt{Edgebank}, \texttt{DyG-Mamba}~\citep{dyg-mamba}, and \texttt{FreeDyG}~\citep{freedyg}.

In Tables~\ref{tab:Ap_transductive}, \ref{tab:Auc_roc transductive}, \ref{tab:AP inductive setting}, and \ref{tab:Auc_roc inductive setting} we compare  the performance of proposed model against the state of the art methods with across these datasets. It is clear that the proposed model consistently outperforms competing methods on most datasets, which we attribute to its ability to jointly model structural and temporal evolution through graph filters and state-space dynamics.

Additionally, in Tables~\ref{tab:freedyg under transductive setting} and \ref{tab:freedyg-only under inductive setting}, we provide direct comparisons against \texttt{Edgebank}, \texttt{DyG-Mamba}, and \texttt{FreeDyG}. The results clearly demonstrate that our model consistently achieves superior performance across both transductive and inductive settings.
}

\begin{table}
\centering
\caption{Performance comparison with AP on dynamic link prediction under transductive setting. }
\resizebox{0.6\columnwidth}{!}{
\begin{tabular}{llcccc}
\toprule
\textbf{NSS} & \textbf{Datasets} &\texttt{Edgebank}& \texttt{DyG-Mamba}& \texttt{FreeDyG} &\texttt{CTDG-SSM} \\
\midrule
\multirow{7}{*}{rnd}
& Wiki     & 90.37$\pm$0.00 &99.08 $\pm$ 0.09 & 99.26 $\pm$ 0.01 & \textbf{99.40 $\pm$ 0.00} \\
& Reddit   &94.86$\pm$0.00  &99.27 $\pm$ 0.00 & 99.48 $\pm$ 0.01 & \textbf{99.53 $\pm$ 0.00}\\
& MOOC      &57.97$\pm$0.00  &90.25 $\pm$ 0.01 & 89.61 $\pm$ 0.10 & \textbf{98.85 $\pm$ 0.00}\\
& LastFM   &79.29$\pm$0.00 &  \textbf{94.23 $\pm$ 0.01} & 92.15 $\pm$ 0.16 &93.40$\pm$0.49\\
& Enron      &83.53 $\pm$ 0.00 &93.14 $\pm$ 0.08 & 92.51 $\pm$ 0.05 & \textbf{94.46 $\pm$ 4.73} \\
& Social Evo. &74.95 $\pm$ 0.00 &94.77 $\pm$ 0.01 & 94.91 $\pm$ 0.01 & \textbf{98.65 $\pm$ 0.65} \\
& UCI     &76.20 $\pm$ 0.00 &96.14 $\pm$ 0.14   & \textbf{96.28 $\pm$ 0.11} & 90.18 $\pm$0.98 
\\
\midrule
\multirow{7}{*}{hist}
& Wiki       &73.35 $\pm$ 0.00 &82.35 $\pm$1.25 & 91.59 $\pm$ 0.57 &\textbf{98.99 $\pm$ 0.32} \\
& Reddit      & 73.59 $\pm$ 0.00 &81.02 $\pm$ 0.19 & 85.67 $\pm$ 1.01 & \textbf{97.55 $\pm$ 0.22} \\
& MOOC        &60.71 $\pm$ 0.00 &87.42 $\pm$ 1.57 & 86.71 $\pm$ 0.81 & \textbf{94.76 $\pm$ 1.76}\\
& LastFM      &73.03 $\pm$ 0.00 &84.08 $\pm$ 0.45 & 79.71 $\pm$ 0.51 & \textbf{88.91 $\pm$ 0.93} \\
& Enron      &76.53 $\pm$ 0.00 &77.85 $\pm$ 1.20& 78.87 $\pm$ 0.82  & \textbf{95.80 $\pm$ 3.33} \\
& Social Evo. &80.57 $\pm$ 0.00 &97.35 $\pm$ 0.18 & 77.79 $\pm$ 0.23 & \textbf{98.20 $\pm$ 0.81} \\
& UCI        &65.50 $\pm$ 0.00 &81.36 $\pm$ 0.14 & 86.10 $\pm$ 1.19 & \textbf{88.87 $\pm$ 1.28} \\
\midrule
\multirow{7}{*}{ind}
& Wiki    &80.63 $\pm$ 0.00    &87.06 $\pm$ 0.86 & 90.05 $\pm$ 0.79 & \textbf{99.45 $\pm$ 0.06}\\
& Reddit   &85.48 $\pm$0.00 &91.77 $\pm$0.46  & 90.74 $\pm$ 0.17 & \textbf{99.58 $\pm$ 0.02}\\
& MOOC     &49.43 $\pm$0.00 &81.19 $\pm$ 2.02  & 83.01 $\pm$ 0.87 & \textbf{99.03 $\pm$ 0.38}\\
& LastFM    &75.49 $\pm$ 0.00 &75.05 $\pm$ 0.40 & 72.19 $\pm$ 0.24 &\textbf{93.81 $\pm$ 0.44} \\
& Enron    &73.89 $\pm$ 0.00  &77.46 $\pm$ 0.90 & 77.81 $\pm$ 0.65 & \textbf{95.81 $\pm$ 2.99} \\
& Social Evo. & 83.69$\pm$0.00 &97.78 $\pm$ 0.15 & 97.50 $\pm$ 0.15 & \textbf{98.88 $\pm$ 0.63} \\
& UCI      &57.43 $\pm$0.00  &77.75 $\pm$ 1.56 & 82.35 $\pm$ 0.73 &\textbf{91.44 $\pm$0.50} \\
\bottomrule
\end{tabular}}
\label{tab:freedyg under transductive setting}
\end{table}

\begin{table}
\centering
\caption{ AP for  dynamic link prediction under transductive setting}
\small
\resizebox{\columnwidth}{!}{
\begin{tabular}{l l c c c c c c c c c c c}

\toprule
\textbf{NSS} & \textbf{Dataset} & \texttt{JODIE} & \texttt{DyRep} & \texttt{TGAT} &
\texttt{TGN} & \texttt{CAWN} & \texttt{Edgebank} & \texttt{TCL} &
\texttt{GraphMixer} & \texttt{DyGFormer} & \texttt{DyGMamba} & \texttt{CTDG-SSM} \\
\midrule

\multirow{6}{*}{rnd}
& Flights  & 95.60$\pm$1.73 & 95.29±0.72 & 94.03±0.18 & 97.95±0.14 & 98.51±0.01 & 89.35±0.00 & 91.23±0.02 & 90.99±0.05 & 98.91±0.01  & \textbf{98.95±0.05} & 98.70±0.05 \\
& Can. Parl.   & 69.26 $\pm$ o.31 & 66.54 $\pm$2.76 & 70.73 $\pm$0.72 & 70.88 $\pm$ 2.34& 69.82 $\pm$2.34 & 64.55 $\pm$0.00 & 68.67 $\pm$2.67 & 77.04 $\pm$0.46 & 97.36±0.45 & \textbf{99.57±0.08} & 98.20 $\pm$1.73 \\
& US Legis.    & 75.05 $\pm$ 1.52 & 75.34 $\pm$ 0.39 & 68.52$\pm$3.16 & 75.99 $\pm$ 0.58 & 70.58 $\pm$ 0.48& 58.39 $\pm$ 0.00 & 69.59 $\pm$ 0.48 &  70.74 $\pm$ 1.02 & 71.11 $\pm$ 0.59 & 71.75 $\pm$ 0.26 & \textbf{82.51 $\pm$ 0.00} \\
& UN Trade     & 64.94 $\pm$ 0.31 & 63.21 $\pm$ 0.93 & 61.47 $\pm$ 0.18 & 65.03 $\pm$ 1.37 & 65.39 $\pm$ 0.12 & 60.41 $\pm$ 0.00 & 62.21 $\pm$ 0.03 & 62.21 $\pm$ 0.27 & 66.46 $\pm$ 1.29 & 67.50$\pm$0.14 & \textbf{69.10$\pm$0.20} \\
& UN Vote      & 63.91 $\pm$ 0.81 & 62.81$\pm$0.80 & 52.21 $\pm$ 0.98 & 65.72 $\pm$ 2.17& 52.84 $\pm$ 0.10 & 58.49 $\pm$ 0.00 & 51.90 $\pm$ 0.30 & 52.11 $\pm$ 0.16 & 55.55$\pm$0.42 & 56.39$\pm$0.18 & \textbf{95.31 $\pm$ 0.01} \\
& Contact      & 95.31$\pm$1.33  & 95.98$\pm0.15$ & 96.28$\pm0.09$ & 96.89$\pm$0.56 & 90.26$\pm$0.28 & 92.58$\pm$0.00 & 92.44$\pm$0.12 & 91.92$\pm$0.03 & 98.29$\pm$0.01 & 98.43$\pm$0.12 & \textbf{98.90 $\pm$ 0.05} \\
\midrule

\multirow{6}{*}{hist}
& Flights      & 66.48$\pm$2.59 & 67.61$\pm$0.99 & 72.38$\pm$ 0.18& 66.70$\pm$1.64 & 64.72$\pm$0.97 & 70.53$\pm$0.00 & 70.68$\pm$ 0.24& 71.47$\pm$0.26 & 66.59±0.49 & 67.80±2.17 &  \textbf{87.2 $\pm$ 1.50}\\
& Can. Parl.   & 51.79$\pm$0.63 & 63.31$\pm$1.23 & 67.13$\pm$0.84 & 68.42$\pm$3.07 & 66.53$\pm$2.77 & 63.84$\pm$0.00&65.93 $\pm$3.00 &74.34 $\pm$0.87 & 97.00$\pm$0.31 & \textbf{99.77$\pm$1.00} & 97.8 $\pm$ 1.24 \\
& US Legis.    & 51.71 $\pm$5.76 & 86.88 $\pm$  2.25 & 62.14 $\pm$ 6.60 & 74.00 $\pm$7.57 & 68.82 $\pm$ 8.23 & 63.22$\pm$0.00 & 80.53 $\pm$ 3.95 & 81.65 $\pm$ 1.02 & 85.30$\pm$3.88 & \textbf{86.12$\pm$0.26} & 80.02 $\pm$ 0.00 \\
& UN Trade     & 61.39 $\pm$ 1.83 & 59.19 $\pm$ 0.17 & 55.74 $\pm$ 0.91 & 58.44 $\pm$ 5.51 & 55.71 $\pm$ 0.38 & 81.32 $\pm$ 0.00 & 55.90 $\pm$ 1.17 & 57.05$\pm$ 1.22 & 64.41$\pm$1.40 & 66.10$\pm$1.02 & \textbf{68.4 $\pm$ 0.04} \\
& UN Vote      & 70.02 $\pm$ 0.81& 69.30$\pm$ 1.12 & 52.96 $\pm$ 2.14 & 69.37 $\pm$ 3.93 & 51.26 $\pm$ 0.04 & 84.89 $\pm$ 0.00 & 52.30 $\pm$ 2.35 & 51.20 $\pm$ 1.60 & 60.84$\pm$1.58 & 61.07$\pm$1.39 &  \textbf{95.29$\pm$0.01}\\
& Contact      & 95.31 $\pm$ 2.13 & 96.39 $\pm$ 0.20 & 96.05 $\pm$ 0.52 & 93.05 $\pm$ 2.35 & 84.16 $\pm$ 0.49 & 88.81 $\pm$ 0.00 & 93.86 $\pm$ 0.01 & 93.36$\pm$ 0.41 & 97.57 $\pm$ 0.06 & 97.61$\pm$0.04 & \textbf{98.2 $\pm$ 0.05}\\
\midrule

\multirow{6}{*}{ind}
& Flights      & 69.07 $\pm$ 4.02 & 70.57 $\pm$ 1.82 & 75.48 $\pm$ 0.26 & 71.09 $\pm$ 2.72 & 69.18$\pm$ 1.52 & 81.08 $\pm$ 0.00 & 74.62 $\pm$ 0.18 & 74.87 $\pm$ 0.21 & 70.92$\pm$1.78 & 73.79$\pm$5.69 & \textbf{86.50$\pm$1.34} \\
& Can. Parl.   & 48.42 $\pm$ 0.66 & 58.61$\pm$ 0.86 & 68.82 $\pm$ 1.21 & 65.34$\pm$2.87 & 67.75$\pm$1.00 & 62.16$\pm$0.00 & 65.85 $\pm$ 1.75& 69.48 $\pm$ 0.63 & \textbf{95.44 $\pm$ 0.57} & 94.87±0.67 & 94.2 $\pm$ 0.50 \\
& US Legis.    & 50.27$\pm$5.13 & 83.44$\pm$1.16 & 61.91 $\pm$ 5.82 & 67.57$\pm$6.47 & 65.81$\pm$ 8.52 & 65.74 $\pm$ 0.00& 78.15 $\pm$ 3.34 & 79.63 $\pm$ 0.84 & 81.25$\pm$3.62 & 81.22±1.34 & \textbf{81.32 $\pm$ 0.00} \\
& UN Trade     & 60.42 $\pm$ 1.48 & 60.19 $\pm$ 1.24 & 60.61 $\pm$ 1.24 & 61.04 $\pm$ 6.01 & 62.54 $\pm$ 0.67 & 72.97$\pm$0.00 & 61.06$\pm$1.74 & 60.15 $\pm$1.29 & 55.79$\pm$1.02 & 58.89$\pm$0.59 & \textbf{67.92 $\pm$ 0.5} \\
& UN Vote      & 67.79 $\pm$ 1.46 & 67.53 $\pm$ 1.98& 52.89 $\pm$ 1.61 & 67.63 $\pm$ 2.67 & 52.19 $\pm$ 0.34 & 66.30 $\pm$ 0.00 & 50.62 $\pm$ 0.82 & 51.60 $\pm$ 0.73 & 51.91$\pm$0.84 & 52.24$\pm$0.95 & \textbf{95.37 $\pm$ 0.01} \\
& Contact      & 93.43 $\pm$ 1.78 & 94.18 $\pm$ 0.10 & 94.35$\pm$ 0.48 & 90.18 $\pm$ 3.28 & 89.31 $\pm$ 0.27 & 85.20 $\pm$ 0.00 & 91.35 $\pm$ 0.21 & 90.87 $\pm$ 0.35 & 94.75$\pm$0.28 & 95.43$\pm$0.17 & \textbf{97.60 $\pm$ 0.32} \\
\bottomrule
\end{tabular}
}
\label{tab:Ap_transductive}
\end{table}

\begin{table}
\centering
\caption{AUC-ROC for dynamic link prediction under transductive setting}

\resizebox{\columnwidth}{!}{
\begin{tabular}{llccccccccccc}
\toprule
\textbf{NSS} & \textbf{Dataset} & \texttt{JODIE} & \texttt{DyRep} & \texttt{TGAT} & \texttt{TGN} & \texttt{CAWN} & \texttt{EdgeBank} & \texttt{TCL} & \texttt{GraphMixer} & \texttt{DyGFormer} & \texttt{DyGMamba} & \texttt{CTDG-SSM} \\
\midrule

\multirow{6}{*}{rnd}
& Flights      & 96.21 $\pm$ 1.42 & 95.95 $\pm$ 0.62 & 94.13 $\pm$ 0.17 & 98.22 $\pm$ 0.13 & 98.45 $\pm$ 0.01 & 90.23 $\pm$ 0.00 & 91.21 $\pm$ 0.02 & 91.13 $\pm$ 0.01 & 98.93 $\pm$ 0.01 & \textbf{98.98 $\pm$ 0.05} & 98.53 $\pm$ 0.02 \\
& Can. Parl.   & 78.21 $\pm$ 0.23 & 73.35 $\pm$ 3.67 & 75.69 $\pm$ 0.78 & 76.99 $\pm$ 1.80 & 75.70 $\pm$ 3.27 & 64.14 $\pm$ 0.00 & 72.46 $\pm$ 3.23 & 83.17 $\pm$ 0.53 & 97.76 $\pm$ 0.41 &\textbf{99.69 $\pm$ 0.06} & 98.1 $\pm$ 0.80\\
& US Legis.    & 82.85 $\pm$ 1.07 & 82.28 $\pm$ 0.32 & 75.84 $\pm$ 1.99 & 83.34 $\pm$ 0.43 & 77.16 $\pm$ 0.39 & 62.57 $\pm$ 0.00 & 76.27 $\pm$ 0.63 & 76.96 $\pm$ 0.79 & 77.90 $\pm$ 0.58 &79.03 $\pm$ 0.26 & \textbf{85.92 $\pm$ 0.00} \\
& UN Trade     & 69.62 $\pm$ 0.44 & 67.44 $\pm$ 0.83 & 64.01 $\pm$ 0.12 & 69.10 $\pm$ 1.67 & 68.54 $\pm$ 0.18 & 66.75 $\pm$ 0.00 & 64.72 $\pm$ 0.05 & 65.52 $\pm$ 0.51 & 70.20 $\pm$ 1.44 & 71.41 $\pm$  0.21 & \textbf{73.76 $\pm$ 0.30} \\
& UN Vote      & 68.53 $\pm$ 0.95 & 67.18 $\pm$ 1.04 & 52.83 $\pm$ 1.12 & 68.71 $\pm$ 2.65 & 53.09 $\pm$ 0.22 & 62.97 $\pm$ 0.00 & 51.88 $\pm$ 0.36 & 52.46 $\pm$ 0.27 & 57.12 $\pm$ 0.62 &58.48 $\pm$ 0.12  & \textbf{97.41 $\pm$ 0.00}\\
& Contact      & 96.66 $\pm$ 0.89 & 96.48 $\pm$ 0.14 & 96.95 $\pm$ 0.08 & 97.54 $\pm$ 0.35 & 89.99 $\pm$ 0.34 & 94.34 $\pm$ 0.00 & 94.15 $\pm$ 0.09 & 93.94 $\pm$ 0.02 & 98.53 $\pm$ 0.01 & 98.68 $\pm$ 0.02 &  \textbf{98.70 $\pm$ 0.01}\\
\midrule

\multirow{6}{*}{hist}
& Flights      & 68.97 $\pm$ 1.87 & 69.43 $\pm$ 0.90 & 72.20 $\pm$ 0.16 & 68.39 $\pm$ 0.95 & 66.11 $\pm$ 0.71 & 74.64 $\pm$ 0.00 & 70.57 $\pm$ 3.01 & 70.37 $\pm$ 0.23 & 68.09 $\pm$ 0.43 & 68.98 $\pm$ 1.81 & 90.1 $\pm$ 1.50\\
& Can. Parl.   & 62.44 $\pm$ 1.11 & 70.16 $\pm$ 1.70 & 70.86 $\pm$ 0.94 & 73.23 $\pm$ 3.08 & 72.06 $\pm$ 3.94 & 63.04 $\pm$ 0.00 & 69.95 $\pm$ 3.70 & 79.03 $\pm$ 1.01 & 97.61 $\pm$ 0.40 & \textbf{99.82 $\pm$ 0.10} & 97.5 $\pm$ 0.05 \\
& US Legis.    & 67.47 $\pm$ 6.40 & \textbf{91.44 $\pm$ 1.18} & 73.47 $\pm$ 5.25 & 83.53 $\pm$ 4.53 & 78.62 $\pm$ 7.46 & 67.41 $\pm$ 0.00 & 83.97 $\pm$ 3.71 & 85.17 $\pm$ 0.70 & {90.77 $\pm$ 1.96} & 88.36 $\pm$ 1.78 & 85.55 $\pm$ 0.00\\
& UN Trade     & 68.92 $\pm$ 1.40 & 64.36 $\pm$ 1.40 & 60.37 $\pm$ 0.68 & 63.93 $\pm$ 5.41 & 63.09 $\pm$ 0.74 & \textbf{86.61 $\pm$ 0.00} & 61.43 $\pm$ 1.04 & 63.20 $\pm$ 1.54 & 73.86 $\pm$ 1.13 & 74.10 $\pm$ 2.02 & 73.3 $\pm$ 0.7\\
& UN Vote      & 76.84 $\pm$ 1.01 & 74.72 $\pm$ 1.43 & 53.95 $\pm$ 3.15 & 73.40 $\pm$ 5.20 & 51.27 $\pm$ 0.33 & 89.62 $\pm$ 0.00 & 52.29 $\pm$ 2.39 & 52.61 $\pm$ 1.44 & 64.27 $\pm$ 1.78 &65.17 $\pm$ 1.24 & \textbf{97.22 $\pm$ 0.00} \\
& Contact      & 96.35 $\pm$ 0.92 & 96.00 $\pm$ 0.23 & 95.39 $\pm$ 0.43 & 93.76 $\pm$ 1.29 & 83.06 $\pm$ 0.32 & 92.17 $\pm$ 0.00 & 93.34 $\pm$ 0.19 & 94.14 $\pm$ 0.34 & 97.17 $\pm$ 0.05 & 97.27 $\pm$ 0.06 & \textbf{97.78 $\pm$ 0.04} \\
\midrule

\multirow{6}{*}{ind}
& Flights      & 69.99 $\pm$ 3.10 & 71.13 $\pm$ 1.55 & 73.47 $\pm$ 0.18 & 71.63 $\pm$ 1.72 & 69.70 $\pm$ 0.75 & 81.10 $\pm$ 0.00 & 72.54 $\pm$ 0.19 & 72.21 $\pm$ 0.21 & 69.53 $\pm$ 1.17 & 71.16 $\pm$ 3.24 & \textbf{89.25 $\pm$ 1.24}  \\
& Can. Parl.   & 52.88 $\pm$ 0.80 & 63.53 $\pm$ 0.65 & 72.47 $\pm$ 1.18 & 69.57 $\pm$ 2.81 & 72.93 $\pm$ 1.78 & 61.41 $\pm$ 0.00 & 69.47 $\pm$ 2.12 & 70.52 $\pm$ 0.94 & 96.70 $\pm$ 0.59 & \textbf{99.82 $\pm$ 0.10} & 96.87 $\pm$ 0.05 \\
& US Legis.    & 59.05 $\pm$ 5.52 & 89.44 $\pm$ 0.71 & 71.62 $\pm$ 5.42 & 78.12 $\pm$ 4.46 & 76.45 $\pm$ 7.02 & 68.66 $\pm$ 0.00 & 82.54 $\pm$ 3.91 & 84.22 $\pm$ 0.91 & \textbf{87.96 $\pm$ 1.80} & 86.08 $\pm$ 2.27 & 86.06 $\pm$ 0.00 \\
& UN Trade     & 66.82 $\pm$ 1.27 & 65.60 $\pm$ 1.28 & 66.13 $\pm$ 0.78 & 66.37 $\pm$ 5.39 & 71.73 $\pm$ 0.74 & 74.20 $\pm$ 0.00 & 67.80 $\pm$ 1.21 & 66.53 $\pm$ 1.22 & 62.56 $\pm$ 1.51 &67.60 $\pm$ 0.64  & 72.65 $\pm$0.45\\
& UN Vote      & 73.73 $\pm$ 1.61 & 72.80 $\pm$ 2.16 & 53.04 $\pm$ 2.58 & 72.69 $\pm$ 3.72 & 52.75 $\pm$ 0.90 & 72.85 $\pm$ 0.00 & 52.02 $\pm$ 1.22 & 62.56 $\pm$ 1.51 & 51.89 $\pm$ 0.74 & 53.37 $\pm$ 1.26 & \textbf{97.45 $\pm$ 0.01}\\
& Contact      & 94.47 $\pm$ 1.08 & 94.23 $\pm$ 0.18 & 94.10 $\pm$ 0.41 & 91.64 $\pm$ 1.72 & 87.68 $\pm$ 0.24 & 85.87 $\pm$ 0.00 & 91.23 $\pm$ 0.19 & 90.96 $\pm$ 0.27 & 95.01 $\pm$ 0.15& 95.68 $\pm$ 0.20 & \textbf{97.50 $\pm$ 0.45}\\
\bottomrule
\end{tabular}}
\label{tab:Auc_roc transductive}
\end{table}

\begin{table*}
\centering
\caption{AP for dynamic link prediction under inductive setting}
\resizebox{\textwidth}{!}{
\begin{tabular}{llcccccccccc}
\toprule
\textbf{NSS} & \textbf{Dataset} & \texttt{JODIE} & \texttt{DyRep} & \texttt{TGAT} & \texttt{TGN} & \texttt{CAWN} & \texttt{TCL} & \texttt{GraphMixer} & \texttt{DyGFormer} & \texttt{DyGMamba} &\texttt{CTDG-SSM} \\
\midrule

\multirow{6}{*}{rnd}
& Flights      & 94.74 $\pm$ 0.37 & 92.88 $\pm$ 0.73 & 88.73 $\pm$ 0.33 & 95.03 $\pm$ 0.60 & 97.06 $\pm$ 0.02 & 83.41 $\pm$ 0.07 & 83.03 $\pm$ 0.05 & 97.79 $\pm$ 0.02 & \textbf{97.85 $\pm$ 0.22} & 97.15$\pm$0.04 \\
& Can. Parl.   & 53.92 $\pm$ 0.94 & 54.02 $\pm$ 0.76 & 55.18 $\pm$ 0.79 & 54.10 $\pm$ 0.93 & 55.80 $\pm$ 0.69 & 54.30 $\pm$ 0.66 &   55.91 $\pm$ 0.82 & 87.74 $\pm$ 0.71 &\textbf{93.46 $\pm$ 2.62} &  88.65 $\pm$ 0.70  \\
& US Legis.    & 54.93 $\pm$ 2.29 & 57.28  $\pm$ 0.71 & 51.00 $\pm$ 3.11 & 58.63 $\pm$ 0.37 & 53.17 $\pm$ 1.20 & 52.59 $\pm$ 0.97 & 50.71 $\pm$ 0.76 & 54.28 $\pm$ 2.87 & 55.95 $\pm$ 1.16 & \textbf{76.94 $\pm$ 0.01}\\
& UN Trade     & 59.65 $\pm$ 0.77 & 57.02 $\pm$ 0.69 & 61.03 $\pm$ 0.18 & 58.31 $\pm$ 3.15 & 65.24 $\pm$ 0.21 & 62.21 $\pm$ 0.12 & 62.17 $\pm$ 0.31 & 64.55 $\pm$ 0.62 & 70.55 $\pm$ 0.04 &\textbf{72.42 $\pm$ 0.02}\\
& UN Vote      & 56.64 $\pm$ 0.96 & 54.62 $\pm$ 2.22 & 52.24 $\pm$ 1.46 & 58.85 $\pm$ 2.51 & 49.94 $\pm$ 0.45 & 51.60 $\pm$ 0.97 & 50.68 $\pm$ 0.44 & 55.93 $\pm$ 0.39 &56.61 $\pm$ 0.13 & \textbf{95.79$\pm$0.01}\\
& Contact      & 94.34 $\pm$ 1.45 & 92.18 $\pm$ 0.41 & 95.87 $\pm$ 0.11 & 93.82 $\pm$ 0.99 & 89.55 $\pm$ 0.30 & 91.11 $\pm$ 0.12 & 90.59 $\pm$ 0.05 & 98.03 $\pm$ 0.02 & 98.16 $\pm$ 0.03 &\textbf{98.42 $\pm$  0.01}\\
\midrule

\midrule

\multirow{6}{*}{ind}
& Flights      & 61.01 $\pm$ 1.66 & 62.83 $\pm$ 1.31 & 64.72 $\pm$ 0.37 & 59.32 $\pm$ 1.45 & 56.82 $\pm$ 0.56 & 64.50 $\pm$ 0.25 & 65.29 $\pm$ 0.24 & 57.11 $\pm$ 0.20 & 57.76 $\pm$2.06 & \textbf{92.24 $\pm$ 1.05} \\
& Can. Parl.   & 52.58 $\pm$ 0.86 & 52.24 $\pm$ 0.28 & 56.46 $\pm$ 0.50 & 54.18 $\pm$ 0.73 & 57.06 $\pm$ 0.08 & 55.46 $\pm$ 0.69 & 55.76 $\pm$ 0.65 & 87.22 $\pm$ 0.82 &\textbf{92.68 $\pm$ 0.97} & 88.42 $\pm$ 0.65\\
& US Legis.    & 52.94 $\pm$ 2.11 & 62.10 $\pm$ 1.41 & 51.83 $\pm$ 3.95 & 61.18 $\pm$ 1.10 & 55.56 $\pm$ 1.71 & 53.87 $\pm$ 1.41 & 52.03 $\pm$ 1.02 & 56.31 $\pm$ 3.46 &57.85 $\pm$ 0.23 & \textbf{75.64 $\pm$ 0.01}\\
& UN Trade     & 55.43 $\pm$ 1.20 & 55.42 $\pm$ 0.87 & 55.58 $\pm$ 0.68 & 52.80 $\pm$ 3.24 & 54.97 $\pm$ 0.38 & 55.66 $\pm$ 0.98 & 54.88 $\pm$ 1.01 & 52.56 $\pm$ 1.70 & 52.81 $\pm$ 0.18 & \textbf{69.24 $\pm$ 1.02} \\
& UN Vote      & 61.17 $\pm$ 1.33 & 60.29 $\pm$ 1.79 & 53.08 $\pm$ 3.10 & 63.71 $\pm$ 2.97 & 48.01 $\pm$ 0.82 & 54.13 $\pm$ 2.16 & 48.10 $\pm$ 0.40 & 52.61 $\pm$ 1.25 &53.70 $\pm$ 2.40 & \textbf{95.77 $\pm$ 0.00} \\
& Contact      & 90.43 $\pm$ 2.33 & 89.22 $\pm$ 0.65 & 94.14 $\pm$ 0.45 & 88.12 $\pm$ 1.50 & 74.19 $\pm$ 0.81 & 90.43 $\pm$ 0.17 & 89.91 $\pm$ 0.36 & 93.55 $\pm$ 0.52 & 94.05 $\pm$ 0.32 &\textbf{96.78$\pm$ 0.72} \\
\bottomrule
\end{tabular}}
\label{tab:AP inductive setting}
\end{table*}

\begin{table}
\centering
\caption{AUC-ROC for dynamic link under inductive setting.}
\resizebox{\columnwidth}{!}{
\begin{tabular}{llcccccccccc}
\toprule
\textbf{NSS} & \texttt{Dataset} & \textbf{JODIE} & \texttt{DyRep} & \texttt{TGAT} & \texttt{TGN} & 
\texttt{CAWN} & \texttt{TCL} & \texttt{GraphMixer} & \texttt{DyGFormer} & \texttt{DyGMamba} &\texttt{CTDG-SSM} \\
\midrule

\multirow{6}{*}{\textbf{rnd}}
& Flights    & 95.21 $\pm$ 0.32 & 93.56 $\pm$ 0.70 & 88.64 $\pm$ 0.35 & 95.92 $\pm$ 0.43 & 96.86 $\pm$ 0.02 & 82.48 $\pm$ 0.01 & 82.27 $\pm$ 0.06 & 97.80 $\pm$ 0.02 & \textbf{97.98 $\pm$ 0.25} & 97.36$\pm$0.04 \\
& Can. Parl. & 53.81 $\pm$ 1.14 & 55.27 $\pm$ 0.49 & 56.51 $\pm$ 0.75 & 55.86 $\pm$ 0.75 & 58.83 $\pm$ 1.13 & 55.83 $\pm$ 1.07 & 58.32 $\pm$ 1.08 & 89.33 $\pm$ 0.48 & \textbf{94.02 $\pm$ 3.42} & 89.78$\pm$0.78\\
& US Legis.  & 58.12 $\pm$ 2.35 & 61.07 $\pm$ 0.56 & 48.27 $\pm$ 3.50 & 62.38 $\pm$ 0.48 & 51.49 $\pm$ 1.13 & 50.43 $\pm$ 1.48 & 47.20 $\pm$ 0.89 & 53.21 $\pm$ 3.04 &57.17 $\pm$ 0.20 &\textbf{81.17 $\pm$ 0.00}\\
& UN Trade   & 62.28 $\pm$ 0.50 & 58.82 $\pm$ 0.98 & 62.72 $\pm$ 0.12 & 59.99 $\pm$ 3.50 & 67.05 $\pm$ 0.21 & 63.76 $\pm$ 0.07 & 63.48 $\pm$ 0.37 & 67.25 $\pm$ 1.05 &68.26 $\pm$ 0.26 &\textbf{73.76$\pm$0.45}\\
& UN Vote    & 58.13 $\pm$ 1.43 & 55.13 $\pm$ 3.46 & 51.83 $\pm$ 1.35 & 61.23 $\pm$ 2.71 & 48.34 $\pm$ 0.76 & 50.51 $\pm$ 1.05 & 50.04 $\pm$ 0.86 & 56.73 $\pm$ 0.69 & 56.91 $\pm$ 0.12 &\textbf{97.72 $\pm$0.01}\\
& Contact    & 95.37 $\pm$ 0.92 & 91.89 $\pm$ 0.38 & 96.53 $\pm$ 0.10 & 94.84 $\pm$ 0.75 & 89.07 $\pm$ 0.34 & 93.05 $\pm$ 0.09 & 92.83 $\pm$ 0.05 & 98.30 $\pm$ 0.02&98.44$\pm$0.05 &\textbf{98.70$\pm$0.65} \\
\midrule

\midrule

\multirow{6}{*}{\textbf{ind}}
& Flights    & 60.72 $\pm$ 1.29 & 61.99 $\pm$ 1.39 & 63.40 $\pm$ 0.26 & 59.66 $\pm$ 1.05 & 56.58 $\pm$ 0.44 & 63.49 $\pm$ 0.23 & 63.32 $\pm$ 0.19 & 56.05 $\pm$ 0.22&56.58$\pm$2.12 & \textbf{91.36 $\pm$ 1.87} \\
& Can. Parl. & 51.61 $\pm$ 0.98 & 52.35 $\pm$ 0.52 & 58.15 $\pm$ 0.62 & 55.43 $\pm$ 0.42 & 60.01 $\pm$ 0.47 & 56.88 $\pm$ 0.93 & 56.63 $\pm$ 1.09 & 88.51 $\pm$ 0.73 & \textbf{92.37 $\pm$ 0.18} & 89.56 $\pm$ 0.69  \\
& US Legis.  & 58.12 $\pm$ 2.94 & 67.94 $\pm$ 0.98 & 49.99 $\pm$ 4.88 & 64.87 $\pm$ 1.65 & 54.41 $\pm$ 1.31 & 52.12 $\pm$ 2.13 & 49.28 $\pm$ 0.86 & 56.57 $\pm$ 3.22 & 57.91 $\pm$ 3.41 & \textbf{81.46 $\pm$ 0.01}\\
& UN Trade   & 58.71 $\pm$ 1.20 & 57.87 $\pm$ 1.36 & 59.98 $\pm$ 0.59 & 55.62 $\pm$ 3.59 & 60.88 $\pm$ 0.79 & 61.01 $\pm$ 0.93 & 59.71 $\pm$ 1.17 & 57.28 $\pm$ 3.06 & 57.58 $\pm$ 0.20 & \textbf{71.43 $\pm$ 0.04} \\
& UN Vote    & 65.29 $\pm$ 1.30 & 64.10 $\pm$ 2.10 & 51.78 $\pm$ 4.14 & 68.58 $\pm$ 3.08 & 48.04 $\pm$ 1.76 & 54.65 $\pm$ 2.20 & 45.57 $\pm$ 0.41 & 53.87 $\pm$ 2.01 &54.83 $\pm$ 2.17 & 97.73 $\pm$ 0.01\\
& Contact    & 90.80 $\pm$ 1.18 & 88.87 $\pm$ 0.67 & 93.76 $\pm$ 0.40 & 88.85 $\pm$ 1.39 & 74.79 $\pm$ 0.38 & 90.37 $\pm$ 0.16 & 90.04 $\pm$ 0.29 & 94.14 $\pm$ 0.26 & 94.35 $\pm$ 0.29 & \textbf{96.98 $\pm$ 0.56} \\
\bottomrule
\end{tabular}
}
\label{tab:Auc_roc inductive setting}
\end{table}

\begin{table}
\centering
\caption{Performance comparison with AP on dynamic link prediction under inductive setting.}
\resizebox{0.7\columnwidth}{!}{
\begin{tabular}{llcccc}
\toprule
\textbf{NSS} & \textbf{Datasets} &\texttt{Edgebank}& \texttt{DyG-Mamba}& \texttt{FreeDyG} &\texttt{CTDG-SSM} \\
\midrule
\multirow{7}{*}{rnd}
& Wiki     & N/A &98.65 $\pm$ 0.03 & 98.97 $\pm$ 0.01 & \textbf{99.19 $\pm$ 0.09} \\
& Reddit   & N/A &98.88 $\pm$ 0.00 & 98.91 $\pm$ 0.01 & \textbf{99.28 $\pm$ 0.00}\\
& MOOC     & N/A &90.20 $\pm$ 0.06 & 87.75 $\pm$ 0.62 & \textbf{98.49 $\pm$ 0.48}\\
& LastFM   & N/A &95.13 $\pm$ 0.08 & \textbf{94.89 $\pm$ 0.01} & 93.65 $\pm$ 0.62\\
& Enron    & N/A &91.14 $\pm$ 0.07 & 89.69 $\pm$ 0.17 & \textbf{93.02 $\pm$ 7.25} \\
& Social Evo. &N/A &93.23 $\pm$ 0.01 & 94.76 $\pm$ 0.05 & \textbf{97.56 $\pm$ 0.45} \\
& UCI      & N/A &94.15 $\pm$ 0.04 & \textbf{94.85 $\pm$ 0.10} & 89.12 $\pm$ 1.02 \\
\midrule
\multirow{7}{*}{ind}
& Wiki     & N/A &79.44 $\pm$ 2.78 & 87.54 $\pm$ 0.26 & \textbf{99.23 $\pm$ 0.09}\\
& Reddit   & N/A &65.61 $\pm$ 0.01 & 64.98 $\pm$ 0.20 & \textbf{99.32 $\pm$ 0.03}\\
& MOOC     & N/A &81.67 $\pm$ 1.08 & 81.41 $\pm$ 0.31 & \textbf{98.64 $\pm$ 0.51}\\
& LastFM   & N/A &79.60 $\pm$ 0.28 & 77.01 $\pm$ 0.43 & \textbf{94.08 $\pm$ 0.57} \\
& Enron    & N/A &68.44 $\pm$ 1.85 & 72.85 $\pm$ 0.81 & \textbf{94.56 $\pm$ 5.02} \\
& Social Evo. & N/A &96.93 $\pm$ 0.21 & 96.91 $\pm$ 0.12 & \textbf{98.15 $\pm$ 0.27} \\
& UCI      & N/A &79.27 $\pm$ 1.03 & 82.06 $\pm$ 0.58 & \textbf{90.34 $\pm$ 0.74} \\
\bottomrule
\end{tabular}}
\label{tab:freedyg-only under inductive setting}
\end{table}

\begin{table}
\centering
\caption{MRR on the \texttt{tgbl-wiki} and \texttt{tgbl-coin} datasets.}
\begin{tabular}{lccccc}
\toprule
\textbf{Dataset} 
& \textbf{CTDG-SSM} 
& \textbf{DyGMamba} 
& \textbf{DyGFormer} 
& \textbf{CTAN} 
& \textbf{TGN} \\
\midrule
\texttt{tgbl-wiki} 
& \textbf{0.817} $\pm$ \textbf{0.027}
& 0.739 $\pm$ 0.009 
& 0.798 $\pm$ 0.004 
& 0.668 $\pm$ 0.007 
& 0.396 $\pm$ 0.060 \\

\texttt{tgbl-coin} 
&\textbf{0.862} $\pm$ \textbf{0.003} 
& --- 
& 0.752 $\pm$ 0.004 
& 0.748 $\pm$ 0.004 
& 0.586 $\pm$ 0.037 \\
\bottomrule
\end{tabular}
\label{tab:tgbl_results}
\end{table}

\section{CTDG-SSM Beyond Node/Edge Addition.}
 The CTDG-SSM state update equation depends on the change in the graph Laplacian, and therefore naturally accommodates both the addition and removal of edges.

To handle edge deletions within the subgraph, one can simply invert the construction process described in the main paper. Specifically, if an edge is removed in batch $B$, we construct batch Laplacian $\mL_B[k]$ without this edge, while $\mL_B[k-1]$ includes it. The resulting difference $\mL_B[k] - \mL_B[k-1]$ correctly captures the effect of edge removal.

Node deletion can be treated analogously by removing all edges incident to that node. In this case, $\mL_B[k]$ contains no edges between the removed node and its neighbors, while $\mL_B[k-1]$ retains these edges. This ensures that the update mechanism captures the effective removal of the node.

\section{Limitations and  Future Research Directions.}

In the current model, we implement a polynomial of the Laplacian using simple graph filters, which provide an efficient linear approximation to the underlying differential operator. While effective, this design restricts the expressiveness of the operator. An important direction for future work is to explore learning these operators and their inverses directly through graph neural networks, potentially enabling more adaptive and data-dependent approximations. Also, the current framework is primarily evaluated on CTDG datasets, where all node and edge features are fully observed. Extending the framework to handle scenarios with missing features in sampled events, or to accommodate interleaved and partially observed dynamic graphs, presents a promising direction for future research.

\section{ CTDG-SSM Pseudo code and Time complexity}\label{sec:time_complexity}
The CTDG-SSM model consists of two primary components: the online update and the inference. In this section, we will provide the algorithm for both of these parts.
\subsection{Inference}
The query provided to the model for a downstream task may take the form \((u,v,t)\), where the model must determine whether this constitutes a valid link or classify node \(u\) based on the interaction and its historical context. Alternatively, the query may be of the form \((u)\), in which case the model retrieves the stored state of node \(u\) and processes it according to the downstream task. In this section, Algorithm~\ref{alg:infer_lp} specifies the procedure for link prediction queries, and Algorithm~\ref{alg:infer_nc} details the procedure for node classification queries. 
The inference time complexity depends on the task:
\begin{itemize}
    \item for link prediction and node classification, it is $\mathcal{O}(\mathrm{deg}(u))$ due to the computation of $\Delta t$;
    \item for sequence classification, it is $\mathcal{O}(1)$ per query.
\end{itemize}

\begin{algorithm}[ptb]
\caption{CTDG-SSM Inference (link-prediction) }
\label{alg:infer_lp}
\begin{algorithmic}[1]

\REQUIRE Link prediction queries $\{(u_i,v_i,t_i)\}_i$
\STATE $\Delta t_i = \ln(1+ t_i - t_{\text{last},u_iv_i})$ \hfill \COMMENT{$t_{\text{last},u_iv_i}$ is the last interaction time of node $u_i$ and $v_i$.}
\STATE $\hat{\vy}_{\text(link)}(u_i,v_i,t_i) = [\Tilde{\mX}^{(2)}[u_i] \mid \Tilde{\mX}^{(2)}[v_i] \mid \psi(\Delta t_i) ]$ \hfill \COMMENT{$\psi$ is temporal encoding function} 
\STATE $p_i = \vw^\top\hat{\vy}_{u_i}$ \hfill\COMMENT{logit for link prediction}
\STATE \textbf{return}
\end{algorithmic}
\end{algorithm}

\begin{algorithm}[ptb]
\caption{CTDG-SSM Inference (node-classification) }
\label{alg:infer_nc}
\begin{algorithmic}[1]

\REQUIRE Interaction $\{(u_i,v_i,t_i)\}_i$
\STATE $\Delta t_i = \ln(1+ t_i - t_{\text{last},u_iv_i})$ \hfill \COMMENT{$t_{\text{last},u_iv_i}$ is the last interaction time of node $u_i$ and $v_i$.}
\STATE $\hat{\vy}_{u_i} = [\Tilde{\mX}^{(2)}[u_i] \mid \psi(\Delta t_i) ]$ \hfill \COMMENT{$\psi$ is temporal encoding function} 
\STATE $\vp_i = \mW \hat{\vy}_{\text(link)}(u_i,v_i,t_i)$ \hfill\COMMENT{Multiclass logits}
\STATE \textbf{return}
\end{algorithmic}
\end{algorithm}

\subsection{Online Update}
From a stream of events, we form a batch of \(B\) concurrent events. Using subgraph sampling, we sample at-most $N_u$ neighbors per node and  construct the batch Laplacian \(\mL_B[k] \in \mathbb{R}^{N_B \times B_B}\), where $N_B$ is the number of active batch nodes bounded by $N_B < 2BN_u$. By removing the edges associated with the events in the current batch from \(\mL_B[k]\), we obtain the previous-step Laplacian \(\mL_B[k-1]\). Algorithm~\ref{alg:zoh_update} summarizes this procedure. For node state vector of dimension $d$,a polynomial of highest order $m$, and input feature of dimension $f$ the state update has a compute complexity of   $\mathcal{O}(m N _B ^{3} + d N _B ^{2} + N_Bdf)$, it is to be noted that $N_B << N_\tau$.

\begin{algorithm}[p]
\caption{CTDG-SSM ZOH Update}
\label{alg:zoh_update}
\begin{algorithmic}[1]

\REQUIRE Batch Laplacian $\mL_B[k]$, Batch Laplacian $\mL_B[{k-1}]$, events $\{u_i,v_i,t_i,\vx_{u}[t_i], \vx_{v}[t_i], \vx_{u,v}[t_i] \}_{i=1}^B$  and batch active node indices $\{\hat{u}_i\}_{i=1}^{N_B}$. 
learnable polynomial  $p_\alpha$, Gaussian quadrature node $\vq_{\text{nodes}}\in \mathbb R^{8 \times 1} $, $\vq_{\text{weights}} \in \mathbb R^{8 \times 1}$. State matrices parameters $\mA_{\log}^{(0)} \in \mathbb{R}^{d}, \mA_{\log }^{(1)} \in \mathbb{R}^{d} $,
hidden states $\mH^{(0)}[k]$ and $\mH^{(1)}[k]$.

\STATE $\mI \gets \mI_{\dim(\mL[k])}$\hfill\COMMENT{Identity matrix}
\STATE $\Delta p_\alpha(\mL_B[k]) \gets p_\alpha(\mL_B[k])-p_\alpha(\mL_B[k-1])$

\STATE $\bar{\mA}_{\mL_B[k]} \gets \exp \left(- p_\alpha(\mL_B[k])^{-1} \Delta p_\alpha (\mL_B[k])\right) $


\STATE LHS $\gets \mathbf{0}_{\text{dim}(N_B \times N_B \times 8)}$
\FOR{$i = 0$ to $7$} 
    \STATE LHS[:,:,i] $\gets \exp \left(- p_\alpha(\mL_B[k])^{-1} \Delta p_\alpha  (\mL_B[k])\vq_{\text{node}}[i]\right)$
\ENDFOR
\STATE Construct SSM Input $\mX$:
\STATE $\mN_{\text{St}} \gets \mathbf{0}_{dim (N_\tau \times 2d_s) }$\hfill\COMMENT{Zero matrix, $d_s$: Static embedding dimensions}
\FOR{$i = 0$ to $B-1$}
    \STATE $\mN_{\text{St}}[i, 1{:}d] \gets \mathrm{Static{\text{-}}Embedding}(u_i)$
    \STATE $\mN_{\text{St}}[i, d{+}1{:}2d] \gets \mathrm{Static{\text{-}}Embedding}(v_i)$
    \STATE $\mN_{\text{St}}[i+B, 1{:}d] \gets \mathrm{Static{\text{-}}Embedding}(v_i)$
    \STATE $\mN_{\text{St}}[i+B, d{+}1{:}2d] \gets \mathrm{Static{\text{-}}Embedding}(u_i)$
\ENDFOR

\FOR{$i = 1$ to $B$}                         
    \STATE $\mX[i,:] \gets [\vx_{u}[t_i] \mid \vx_{v}[t_i] \mid \vx_{u,v}[t_i] \mid \phi(\Delta[t_i]) \mid \mN_{st}[i,:]]$ \hfill \COMMENT{$\mV$ has $N_\tau$ rows}
    \STATE $\mX[i+B, :]  \gets [\vx_{u}[t_i] \mid \vx_{v}[t_i] \mid \vx_{u,v}[t_i] \mid \phi(\Delta[t_i]) \mid \mN_{st}[i+B,:]]$ \hfill \COMMENT{rest filled with 0.}
\ENDFOR   

\vspace{10pt}
\STATE \textbf{CTDG-SSM $1^{\text{st}}$ Layer}:
\STATE $\Tilde{\mX}^{(0)}[k] \gets h_\theta(\mX)$\hfill\COMMENT{ encoder $h_\theta$}
\STATE ${\mX}_n^{(0)}[k] \gets \mathrm{RMS}_0(\Tilde{\mX}^{(0)}[k])$
\STATE $\mB_{x,0} \gets \mB_0({\mX}_n^{(0)}[k])$

\STATE $\Delta_0 \gets \tau^{(0)}_\Delta({\mX}_n^{(0)}[k])$\hfill$(N_\tau \times 1)$
\STATE $\mA_{c}^{(0)} \gets -\exp(\mA_{\log}^{(0)})$ 
%
\STATE $\bar{\mA}^{(0)} \gets \exp\left({\Delta_0\odot\mA_{c}^{(0)}[\text{None},:]}\right)$

\STATE $\mC_0 \gets (p_\alpha(\mL_B[k])^{-1}( \Delta_0\odot\mB_{x,0}))[:,:,\text{None}] \odot \vq_{\text{weights}}[\text{None},\text{None},:]$
\STATE $RHS_0 \gets \exp((\Delta_0 \odot \mA_{c}^{(0)} [\text{None}, :])[:, :, \text{None}] \odot \vq_{\text{nodes}}[\text{None}, \text{None}, :])$\hfill$(  N_\tau \times d \times 8)$ 
\STATE $\bar{\mB}(\mL_B[k],{\mX}_n^{(0)}[k]) \gets \sum_{q=0}^7 LHS[:,:,q] \, C_0[:,:,q] \, RHS_{0}[:,:,q]$
\STATE $\hat{\mH}^{(0)}[k+1] \gets \bar{\mA}_{\mL_B[k]}( \mH^{(0)}[k][\{\hat{u}_i\}_{i=1}^{N_B}] \odot \bar{\mA}^{(0)} ) + \bar{\mB}(\mL_B[k],{\mX}_n^{(0)}[k])$ 

\STATE $\Tilde{\mX}^{(1)}[k] \gets \Tilde{\mX}^{(0)}[k] + \mathrm{GeLU}(\hat{\mH}^{(0)}[k+1])$

\vspace{10pt}
\STATE\textbf{CTDG-SSM $2^{\text{nd}}$ Layer}:
\STATE ${\mX}_n^{(1)}[k] \gets \mathrm{RMS}_1(\Tilde{\mX}^{(1)})$
\STATE $\mB_{x,1} \gets \mB_1({\mX}_n^{(1)}[k])$

\STATE $\Delta_1 \gets \tau^{(1)}_{\Delta}({\mX}_n^{(1)}[k])$
\STATE $\mA_{c}^{(1)} \gets -\exp(\mA_{\log}^{(1)})$
\STATE $\bar{\mA}^{(1)} \gets \exp(\Delta_1 \odot \mA_{c}^{(1)}[:,\text{None}])$

\STATE $\mC_1 \gets (p_\alpha(\mL_B[k])^{-1}(\mB_{x,2} \odot \Delta_1))[:,:,\text{None}] \odot \vq_{\text{weights}}[\text{None}, \text{None}, :]$
\STATE $RHS_1 \gets \exp((\Delta_1 \odot \mA_{c}^{(1)}[\text{None},:])[:,:,\text{None}] \odot \vq_{\text{nodes}}[\text{None}, \text{None}, :])$
\STATE $\bar{\mB}(\mL[k],{\mX}_n^{(1)}[k]) \gets \sum_{q=0}^7 LHS[:,:,q] \, C_1[:,:,q] \, RHS_{1}[:,:,q]$

\STATE $\hat{\mH}^{(1)}[k+1] \gets \bar{\mA}_{\mL_B[k]}( \mH^{(1)}[k][\{\hat{u}_i\}_{i=1}^{N_B}] \odot \bar{\mA}^{(1)} ) + \bar{\mB}(\mL_B[k],{\mX}_n^{(1)}[k])$

\STATE $\hat{\mX}^{(2)}[k] \gets\Tilde{\mX}^{(1)}[k] + \mathrm{GeLU}(\hat{\mH}^{2}[k+1])$


\STATE $\mH^{(0)}[k+1] \gets \mathrm{MeanAgg}(\hat{\mH}^{(0)}[k+1], \{\hat{u}_i\}_{i=1}^{N_B})$ \hfill \{Only update for batch active node
\STATE $\mH^{(1)}[k+1] \gets \mathrm{MeanAgg}(\hat{\mH}^{(1)}[k+1], \{\hat{u}_i\}_{i=1}^{N_B})$ \hfill rest of the node retain old values.\}
\STATE $\Tilde{\mX}^{(2)}[k] \gets \mathrm{MeanAgg}(\hat{\mX}^{(2)}[k], \{\hat{u}_i\}_{i=1}^{N_B})$  

\STATE \textbf{return}
\end{algorithmic}
\end{algorithm}

\end{document}